\documentclass{lmcs}
\usepackage{graphicx}

\usepackage{url}
\usepackage{tikzit, tikz-cd, circuitikz}

\tikzstyle{Medium box}=[fill=white, draw=black, shape=rectangle, tikzit shape=rectangle, minimum width=1.5cm, minimum height=1.5cm]
\tikzstyle{medium_box}=[fill=white, draw=black, shape=rectangle, minimum height=0.8cm, minimum width=0.5cm]
\tikzstyle{Large box}=[fill=white, draw=black, shape=rectangle, minimum width=5cm, minimum height=5cm]
\tikzstyle{small}=[fill=white, draw=black, shape=rectangle, minimum width=1cm, minimum height=0.5]
\tikzstyle{loss}=[fill=white, draw=black, shape=rectangle, minimum width=1cm, minimum height=3.5cm]
\tikzstyle{Medium}=[fill=white, draw=green, shape=rectangle, minimum width=1.5cm, minimum height=1.5cm, line width=0.8]
\tikzstyle{blue}=[fill=white, draw=blue, shape=rectangle, minimum width=1.5cm, minimum height=1.5cm]
\tikzstyle{red}=[fill=white, draw=red, shape=rectangle, minimum width=1.5cm, minimum height=1.5cm]
\tikzstyle{Long}=[fill=white, draw=black, shape=rectangle, minimum height=4.5cm, minimum width=1.5cm]
\tikzstyle{blackcircle}=[fill=black, draw=black, shape=circle, minimum width=0.2cm, inner sep=0pt]
\tikzstyle{Small box}=[fill=white, draw=black, shape=rectangle, minimum height=0.5cm, minimum width=0.5cm]
\tikzstyle{elongated}=[fill=white, draw=black, shape=rectangle, minimum height=1cm, minimum width=3cm]
\tikzstyle{circ}=[fill=white, draw=black, shape=circle]
\tikzstyle{blank}=[fill=white, draw=white, shape=circle]
\tikzstyle{none}=[fill=none, draw=none]
\tikzstyle{copy}=[fill=white, draw=black, shape=circle, minimum height=0.2cm, inner sep=0]
\tikzstyle{varCopy}=[fill=black, draw=black, shape=circle, minimum height=0.2cm, inner sep=0]
\tikzstyle{copy2}=[fill=black, draw=black, shape=circle, minimum height=0.2cm, inner sep=0]
\tikzstyle{1morph1}=[fill=white, draw=black, shape=rectangle, minimum width=1cm, minimum height=1cm]
\tikzstyle{1morph}=[fill=white, draw=black, shape=rectangle, minimum width=0.75cm, minimum height=0.75cm, inner sep=0.1cm]
\tikzstyle{2morph2}=[fill=white, draw=black, shape=rectangle, minimum width=1cm, minimum height=2cm]
\tikzstyle{2morph}=[fill=white, draw=black, shape=rectangle, minimum width=1cm, minimum height=1.25cm, inner sep=0.1cm]
\tikzstyle{nmorph}=[fill=white, draw=black, shape=rectangle, minimum height=6cm, minimum width=1cm, inner sep=0.1cm]
\tikzstyle{1state}=[fill=white, draw=black, regular polygon, regular polygon sides=3, minimum height=0.5cm, regular polygon rotate=-30]
\tikzstyle{dbox}=[fill=white, draw=black, dashed, shape=rectangle, minimum width=2cm, minimum height=1cm, inner sep=0.1cm]
\tikzstyle{vdbox}=[fill=white, draw=black, dashed, shape=rectangle, minimum width=2cm, minimum height=1.5cm, inner sep=0.1cm]
\tikzstyle{bigbox}=[fill=white, draw=black, dashed, shape=rectangle, minimum width=2cm, minimum height=4cm, inner sep=0.1cm]
\tikzstyle{2state}=[inner sep=0.05cm, fill=white, draw=black, isosceles triangle, minimum width=1.25cm, isosceles triangle apex angle=90, shape border rotate=180]
\tikzstyle{var2state}=[inner sep=0.05cm, fill=white, draw=black, isosceles triangle, minimum width=1.25cm, isosceles triangle apex angle=60, shape border rotate=180]
\tikzstyle{g2state}=[inner sep=0.05cm, fill=white, draw=black, isosceles triangle, minimum width=6cm, isosceles triangle apex angle=110, shape border rotate=180]
\tikzstyle{bigstate}=[inner sep=0.05cm, fill=white, draw=black, isosceles triangle, minimum width=3cm, isosceles triangle apex angle=110, shape border rotate=180]
\tikzstyle{bigeffect}=[inner sep=0.05cm, fill=white, draw=black, isosceles triangle, minimum width=3cm, isosceles triangle apex angle=110]
\tikzstyle{g2effect}=[inner sep=0.05cm, fill=white, draw=black, isosceles triangle, minimum width=6cm, isosceles triangle apex angle=110]
\tikzstyle{2effect}=[inner sep=0.05cm, fill=white, draw=black, isosceles triangle, minimum width=1.25cm, isosceles triangle apex angle=90]
\tikzstyle{b2effect}=[inner sep=0.05cm, fill=white, draw=black, isosceles triangle, minimum width=2cm, isosceles triangle apex angle=90]
\tikzstyle{node}=[fill=black, draw=black, shape=circle, scale=0.5]
\tikzstyle{GroundLeft}=[fill=white, draw=black, shape=tlground, rotate=-90]
\tikzstyle{GroundRight}=[fill=white, draw=black, shape=tlground, rotate=90]
\tikzstyle{GroundDown}=[fill=white, draw=black, shape=tlground]
\tikzstyle{GroundUp}=[fill=white, draw=black, shape=tlground, rotate=180]

\tikzstyle{Black arrow}=[->]
\tikzstyle{Red line}=[-, draw=red, line width=0.7]
\tikzstyle{Red arrow}=[draw=red, ->]
\tikzstyle{Gray line}=[-, draw={rgb,255: red,191; green,191; blue,191}, line width=0.8]
\tikzstyle{Gray arrow}=[->, draw={rgb,255: red,191; green,191; blue,191}]
\tikzstyle{Blue line}=[-, draw=blue]
\tikzstyle{Blue arrow}=[->, draw=blue]
\tikzstyle{bluearrow}=[->, fill=none, draw={rgb,255: red,29; green,206; blue,255}, thick]
\tikzstyle{midArrow}=[-, decoration={{markings,mark=at position .5 with {\arrow{>}}}}, postaction=decorate]
\tikzstyle{arrow}=[->]
\tikzstyle{pointy}=[->]
\tikzstyle{lightnone}=[-, draw={rgb,255: red,191; green,191; blue,191}]
\tikzstyle{Filled edge shape}=[-, fill=white, draw=black]
\tikzstyle{Thin Grey Arrow}=[->, draw={rgb,255: red,191; green,191; blue,191}, line width=0.3]
\tikzstyle{Thin Grey Line}=[-, draw={rgb,255: red,191; green,191; blue,191}, line width=0.3]
\tikzstyle{Black arrow crossing over}=[->, decoration={{crossing over}}]

\usepackage{xcolor}
\usepackage{amsmath,amsfonts}
\usepackage{thmtools} 
\usepackage{xparse} 
\usepackage{graphicx}
\usepackage[numbers]{natbib}
\usepackage{booktabs}
\usepackage{subcaption} 
\usepackage{mathtools}
\usepackage{scalefnt} 
\usepackage{float}
\usepackage{caption}

\usepackage{stmaryrd} 
\usepackage{todonotes}
\usepackage{scalefnt}

\newcommand{\doi}[1]{\textsc{doi}: \href{http://dx.doi.org/#1}{\nolinkurl{#1}}}
\DeclareMathOperator{\Ca}{\mathcal{C}}
\DeclareMathOperator{\Da}{\mathcal{D}}

\DeclareMathOperator{\Nb}{\mathbb{N}}

\DeclareMathOperator{\Rb}{\mathbb{R}}

\DeclareMathOperator{\Zb}{\mathbb{Z}}

\newcommand{\R}{\ensuremath{\mathbb R}}
\newcommand{\Z}{\ensuremath{\mathbb Z}}

\newcommand{\diset}[2]{\left({{#1} \atop {#2}}\right)} 
\newcommand{\pr}[2]{\begin{matrix}{#1} \\ {#2} \end{matrix}} 
\newcommand{\LensHom}[4]{\Bigg(\pr{#1}{#2} \, , \pr{#3}{#4}\Bigg)}

\newcommand{\NamedCat}[1]{\mathbf{#1}}

\newcommand{\Smooth}{\NamedCat{Smooth}}
\newcommand{\POLY}{\mathrm{POLY}}
\newcommand{\PolyZ}{\mathrm{POLY}_{\Z_2}}

\newcommand{\Para}{\NamedCat{Para}}
\newcommand{\Learn}{\NamedCat{Learn}}

\newcommand{\Lens}{\NamedCat{Lens}}
\newcommand{\CLACat}{\NamedCat{CLACat}}
\newcommand{\CoKl}{\NamedCat{CoKl}}

\newcommand{\groundtruth}{b_t}
\newcommand{\RC}{\mathbf{R}_{\Ca}} 

\newcommand{\gda}{\mathsf{gda}}
\newcommand{\GAN}{\mathsf{GAN}}

\newcommand{\Softmax}{\mathrm{Softmax}}

\newcommand{\<}{\langle}
\renewcommand{\>}{\rangle}

\newcommand{\loss}{\mathsf{loss}}
\newcommand{\id}{\text{id}}
\newcommand{\XOR}{\mathsf{XOR}}

\newcommand{\gett}{\ensuremath{\mathsf{get}}}
\newcommand{\putt}{\ensuremath{\mathsf{put}}}

\NewDocumentCommand{\scaletikzfig}{ O{1} O{1} m }{
  \begin{center}
     \scalefont{#2}\scalebox{#1}{\input{#3.tikz}}
   \end{center}
}

\newcommand{\cp}[0]{\ensuremath{\fatsemi}} 

\usepackage{listings}
\makeatletter
\def\lst@makecaption{%
  \def\@captype{table}%
  \@makecaption
}
\makeatother

\begin{document}

\title{Deep Learning with Parametric Lenses}

\author{Geoffrey S. H. Cruttwell}
\address{Mount Allison University, Canada}
\email{}

\author{Bruno Gavranovi\'c}
\address{University of Strathclyde, United Kingdom
}
\email{}

\author{Neil Ghani}
\address{University of Strathclyde, United Kingdom
}
\email{}

\author{Paul Wilson}
\address{Independent, United Kingdom
}
\email{}

\author{ Fabio Zanasi}
\address{University College London, United Kingdom, and University of Bologna, Italy
}
\email{}

\begin{abstract}
  \label{section:abstract}

We propose a categorical semantics for machine learning algorithms in
terms of lenses, parametric maps, and reverse derivative categories.  This foundation provides a powerful explanatory and unifying framework: it encompasses a variety of gradient
descent algorithms such as ADAM, AdaGrad, and Nesterov momentum,
as well as a variety of loss functions such as MSE and Softmax cross-entropy, and different architectures, shedding new light on their similarities and differences.  Furthermore, our approach to learning has examples generalising beyond the familiar continuous domains (modelled in categories of smooth maps) and can be realised in the discrete setting of Boolean and polynomial circuits.
We demonstrate the practical significance of our framework with an implementation in Python.

\end{abstract}


\maketitle

\section{Introduction}
\label{section:introduction}

The last decade has witnessed a surge of interest in machine learning, fuelled by the numerous successes and applications that these methodologies have found in many fields of science and technology. 
As machine learning techniques become increasingly pervasive, algorithms and models become more sophisticated, posing a significant challenge both to the software developers and the users that need to interface, execute and maintain these systems.
In spite of this rapidly evolving picture, the formal analysis of many
learning algorithms mostly takes place at a heuristic level~\cite{SeshiaS16}, or
using definitions that fail to provide a general and scalable framework for
describing machine learning.
Indeed, it is commonly acknowledged through academia, industry, policy makers
and funding agencies that there is a pressing need for a unifying perspective,
which can make this growing body of work more systematic, rigorous, transparent
and accessible both for users and developers \cite{DeepLearningAdHoc,ExplainableAI}.


Consider, for example, one of the most common machine
learning scenarios: supervised learning with a neural network. This technique trains the model towards a certain task, e.g. the recognition of patterns in a data set
(\emph{cf.} Figure~\ref{fig:informalGD}). There are several different ways of
implementing this scenario. Typically, at their core, there is a
\emph{gradient update} algorithm (often called the ``optimiser''), depending on a given \emph{loss function}, which updates in steps the parameters of the network, based on some \textit{learning rate} controlling the ``scaling'' of the update.
All of these components can vary independently in a supervised learning
algorithm and a number of choices is available for loss maps (quadratic error,
Softmax cross entropy, dot product, etc.) and optimisers (Adagrad \cite{Adagrad}, Momentum
\cite{Momentum}, and Adam \cite{Adam}, etc.).

\begin{figure*}[ht]
  \centering
	\includegraphics[width=7.6cm]{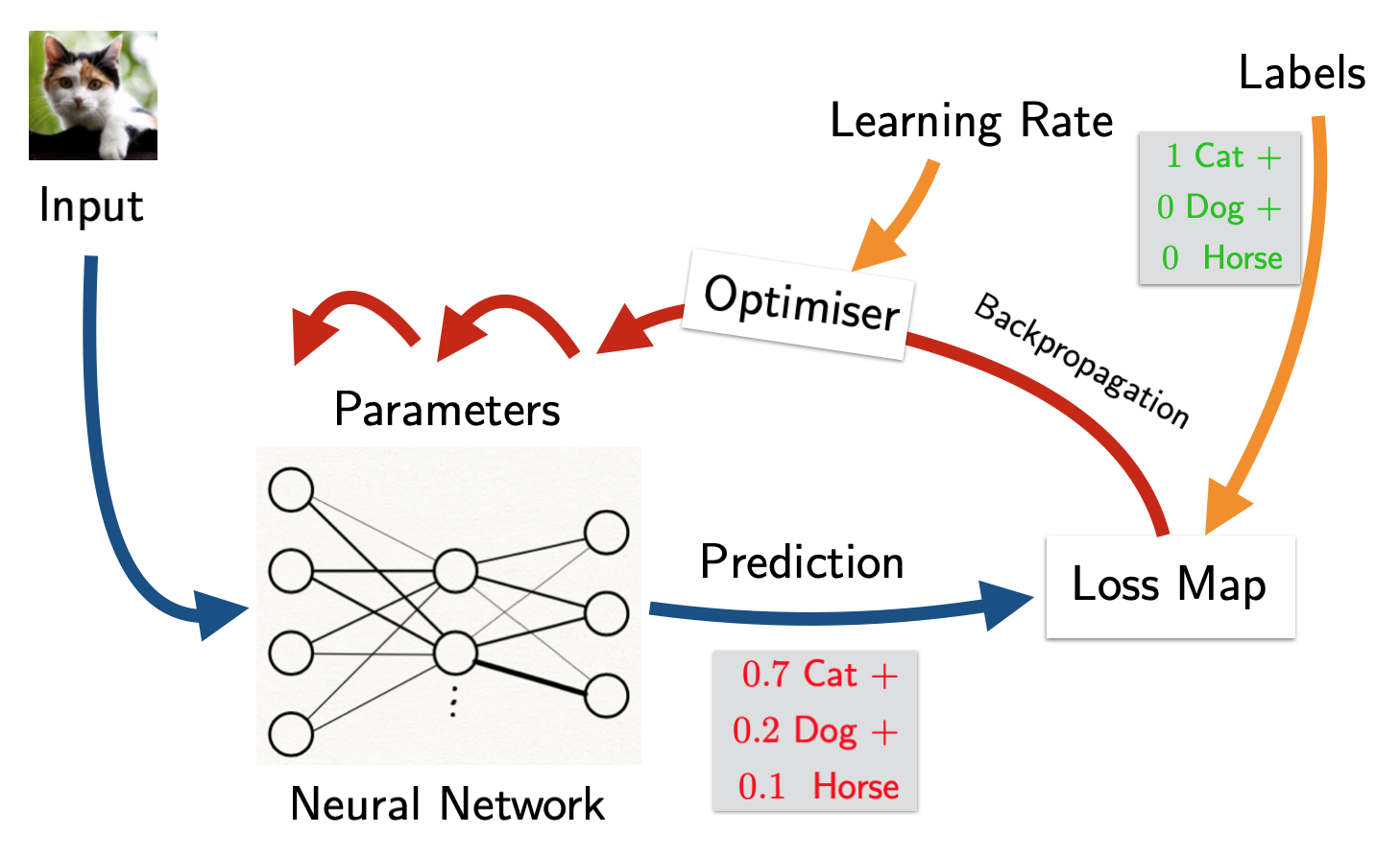} \caption{An
    informal illustration of gradient-based learning. This neural network is
    trained to distinguish different kinds of animals in the input image. Given
    an input $X$, the network predicts an output $Y$, which is compared by a
    `loss map' with what would be the correct answer (`label'). The loss map
    returns a real value expressing the error of the prediction; this
    information, together with the \emph{learning rate} (a weight controlling
    how much the model should be changed in response to error) is used by an
    \emph{optimiser}, which computes by gradient-descent the update of the
    parameters of the network, with the aim of improving its accuracy. The
    neural network, the loss map, the optimiser and the learning rate are all
    components of a supervised learning system, and can vary independently of
    one another.}\label{fig:informalGD}
\end{figure*}

This scenario highlights several questions: is there a uniform mathematical language capturing the different components of the learning process?
 Can we develop a unifying
picture of the various optimisation techniques, allowing for their comparative analysis? Moreover, it should be noted that supervised
learning is not limited to neural networks. For example, supervised learning is surprisingly applicable to the
discrete setting of boolean circuits \cite{rda} where continuous functions are replaced by boolean-valued functions. Can we identify an abstract perspective encompassing both the real-valued and the boolean case? In a nutshell, this paper seeks to answer the question:
 \begin{center}
 \emph{what are the fundamental mathematical structures} \\ \emph{underpinning gradient-based learning?}
 \end{center}

Our approach to this question stems from the identification of three fundamental aspects of the gradient-descent learning process: 

\begin{enumerate}
	\item[(i)] computation is \textbf{parametric}, e.g. in the simplest case we
    are given a function $f : P \times X \to Y$ and learning consists of finding
    a parameter $p : P$ such that $f(p, -)$ is the best function according to
    some criteria. Specifically, the weights on the internal nodes of a neural
    network are a parameter which the learning is seeking to optimize.
    Parameters also arise elsewhere, e.g. in the loss function (see later).
	\item[(ii)] information flows \textbf{bidirectionally}: in the forward direction, the
    computation turns inputs via a sequence of \textit{layers} into predicted
    outputs, and then into a loss value; in the reverse direction,
    backpropagation is used to propagate the changes \textit{backwards} through the layers,
    and then turn them into parameter updates.
	\item[(iii)] the basis of parameter update via gradient descent is \textbf{differentiation} e.g. in the simple case we differentiate the function mapping a parameter to its associated loss to reduce that loss.
  \end{enumerate}

We model bidirectionality via lenses \cite{ProfunctorOptics, Lens,
  BimorphicLenses} and based upon the above three insights, we propose the
notion of \textbf{parametric lens} as the fundamental semantic structure of
learning. In a nutshell, a parametric lens is a process with three kinds of
interfaces: inputs, outputs, and parameters. On each interface, information
flows both ways, i.e. computations are bidirectional. These data are best
explained with our graphical representation of parametric lenses, with inputs
$A$, $A'$, outputs $B$, $B'$, parameters $P$, $P'$, and arrows indicating
information flow (below left). The graphical notation also makes evident that
parametric lenses are \emph{open systems}, which may be composed along their
interfaces (below center and right).
\begin{equation} \label{eq:paralens}
  \begin{gathered}
    \scalebox{0.7}{ \begin{tikzpicture}
	\begin{pgfonlayer}{nodelayer}
		\node [style=none] (20) at (-1, 1) {};
		\node [style=none] (21) at (-1, -1) {};
		\node [style=none] (22) at (1, 1) {};
		\node [style=none] (23) at (1, -1) {};
		\node [style=none] (26) at (-2.5, 0.5) {$A$};
		\node [style=none] (27) at (-2.5, -0.5) {$A'$};
		\node [style=none] (28) at (2.5, 0.5) {$B$};
		\node [style=none] (29) at (2.5, -0.5) {$B'$};
		\node [style=none] (39) at (-2, 0.5) {};
		\node [style=none] (40) at (-1, 0.5) {};
		\node [style=none] (43) at (-1, -0.5) {};
		\node [style=none] (44) at (-2, -0.5) {};
		\node [style=none] (45) at (1, 0.5) {};
		\node [style=none] (46) at (2, 0.5) {};
		\node [style=none] (57) at (2, -0.5) {};
		\node [style=none] (58) at (1, -0.5) {};
		\node [style=none] (60) at (-0.5, 1) {};
		\node [style=none] (61) at (0.5, 1) {};
		\node [style=none] (62) at (-0.5, 2) {};
		\node [style=none] (63) at (0.5, 2) {};
		\node [style=none] (64) at (-0.5, 2.5) {$P$};
		\node [style=none] (65) at (0.5, 2.5) {$P'$};
	\end{pgfonlayer}
	\begin{pgfonlayer}{edgelayer}
		\draw (20.center) to (22.center);
		\draw (22.center) to (23.center);
		\draw (23.center) to (21.center);
		\draw (21.center) to (20.center);
		\draw [style=pointy] (39.center) to (40.center);
		\draw [style=pointy] (43.center) to (44.center);
		\draw [style=pointy] (45.center) to (46.center);
		\draw [style=pointy] (57.center) to (58.center);
		\draw [style=pointy] (62.center) to (60.center);
		\draw [style=pointy] (61.center) to (63.center);
	\end{pgfonlayer}
\end{tikzpicture}}
    \qquad
    \scalebox{0.7}{\begin{tikzpicture}
	\begin{pgfonlayer}{nodelayer}
		\node [style=none] (20) at (-3, 1) {};
		\node [style=none] (21) at (-3, -1) {};
		\node [style=none] (22) at (-1, 1) {};
		\node [style=none] (23) at (-1, -1) {};
		\node [style=none] (26) at (-4.5, 0.5) {$A$};
		\node [style=none] (27) at (-4.5, -0.5) {$A'$};
		\node [style=none] (28) at (0, 1) {$B$};
		\node [style=none] (29) at (0, -1) {$B'$};
		\node [style=none] (39) at (-4, 0.5) {};
		\node [style=none] (40) at (-3, 0.5) {};
		\node [style=none] (43) at (-3, -0.5) {};
		\node [style=none] (44) at (-4, -0.5) {};
		\node [style=none] (45) at (-1, 0.5) {};
		\node [style=none] (46) at (0, 0.5) {};
		\node [style=none] (57) at (0, -0.5) {};
		\node [style=none] (58) at (-1, -0.5) {};
		\node [style=none] (60) at (-2.5, 1) {};
		\node [style=none] (61) at (-1.5, 1) {};
		\node [style=none] (62) at (-2.5, 2) {};
		\node [style=none] (63) at (-1.5, 2) {};
		\node [style=none] (64) at (-2.5, 2.5) {$P$};
		\node [style=none] (65) at (-1.5, 2.5) {$P'$};
		\node [style=none] (66) at (1, 1) {};
		\node [style=none] (67) at (1, -1) {};
		\node [style=none] (68) at (3, 1) {};
		\node [style=none] (69) at (3, -1) {};
		\node [style=none] (72) at (4.5, 0.5) {$C$};
		\node [style=none] (73) at (4.5, -0.5) {$C'$};
		\node [style=none] (74) at (0, 0.5) {};
		\node [style=none] (75) at (1, 0.5) {};
		\node [style=none] (76) at (1, -0.5) {};
		\node [style=none] (77) at (0, -0.5) {};
		\node [style=none] (78) at (3, 0.5) {};
		\node [style=none] (79) at (4, 0.5) {};
		\node [style=none] (80) at (4, -0.5) {};
		\node [style=none] (81) at (3, -0.5) {};
		\node [style=none] (82) at (1.5, 1) {};
		\node [style=none] (83) at (2.5, 1) {};
		\node [style=none] (84) at (1.5, 2) {};
		\node [style=none] (85) at (2.5, 2) {};
		\node [style=none] (86) at (1.5, 2.5) {$Q$};
		\node [style=none] (87) at (2.5, 2.5) {$Q'$};
	\end{pgfonlayer}
	\begin{pgfonlayer}{edgelayer}
		\draw (20.center) to (22.center);
		\draw (22.center) to (23.center);
		\draw (23.center) to (21.center);
		\draw (21.center) to (20.center);
		\draw [style=pointy] (39.center) to (40.center);
		\draw [style=pointy] (43.center) to (44.center);
		\draw (45.center) to (46.center);
		\draw [style=pointy] (57.center) to (58.center);
		\draw [style=pointy] (62.center) to (60.center);
		\draw [style=pointy] (61.center) to (63.center);
		\draw (66.center) to (68.center);
		\draw (68.center) to (69.center);
		\draw (69.center) to (67.center);
		\draw (67.center) to (66.center);
		\draw [style=pointy] (74.center) to (75.center);
		\draw (76.center) to (77.center);
		\draw [style=pointy] (78.center) to (79.center);
		\draw [style=pointy] (80.center) to (81.center);
		\draw [style=pointy] (84.center) to (82.center);
		\draw [style=pointy] (83.center) to (85.center);
	\end{pgfonlayer}
\end{tikzpicture}}
    \qquad
    \vcenter{\hbox{\scalebox{0.7}{\begin{tikzpicture}
	\begin{pgfonlayer}{nodelayer}
		\node [style=none] (20) at (-1, 1) {};
		\node [style=none] (21) at (-1, -1) {};
		\node [style=none] (22) at (1, 1) {};
		\node [style=none] (23) at (1, -1) {};
		\node [style=none] (26) at (-2.5, 0.5) {$A$};
		\node [style=none] (27) at (-2.5, -0.5) {$A'$};
		\node [style=none] (28) at (2.5, 0.5) {$B$};
		\node [style=none] (29) at (2.5, -0.5) {$B'$};
		\node [style=none] (39) at (-2, 0.5) {};
		\node [style=none] (40) at (-1, 0.5) {};
		\node [style=none] (43) at (-1, -0.5) {};
		\node [style=none] (44) at (-2, -0.5) {};
		\node [style=none] (45) at (1, 0.5) {};
		\node [style=none] (46) at (2, 0.5) {};
		\node [style=none] (57) at (2, -0.5) {};
		\node [style=none] (58) at (1, -0.5) {};
		\node [style=none] (60) at (-0.5, 1) {};
		\node [style=none] (61) at (0.5, 1) {};
		\node [style=none] (62) at (-0.5, 3) {};
		\node [style=none] (63) at (0.5, 3) {};
		\node [style=none] (64) at (-1, 2) {$P$};
		\node [style=none] (65) at (1, 2) {$P'$};
		\node [style=none] (66) at (-1, 5) {};
		\node [style=none] (67) at (-1, 3) {};
		\node [style=none] (68) at (1, 5) {};
		\node [style=none] (69) at (1, 3) {};
		\node [style=none] (74) at (-0.5, 5) {};
		\node [style=none] (75) at (0.5, 5) {};
		\node [style=none] (76) at (-0.5, 6) {};
		\node [style=none] (77) at (0.5, 6) {};
		\node [style=none] (78) at (-0.5, 6.5) {$Q$};
		\node [style=none] (79) at (0.5, 6.5) {$Q'$};
	\end{pgfonlayer}
	\begin{pgfonlayer}{edgelayer}
		\draw (20.center) to (22.center);
		\draw (22.center) to (23.center);
		\draw (23.center) to (21.center);
		\draw (21.center) to (20.center);
		\draw [style=pointy] (39.center) to (40.center);
		\draw [style=pointy] (43.center) to (44.center);
		\draw [style=pointy] (45.center) to (46.center);
		\draw [style=pointy] (57.center) to (58.center);
		\draw [style=pointy] (62.center) to (60.center);
		\draw [style=pointy] (61.center) to (63.center);
		\draw (66.center) to (68.center);
		\draw (68.center) to (69.center);
		\draw (69.center) to (67.center);
		\draw (67.center) to (66.center);
		\draw [style=pointy] (76.center) to (74.center);
		\draw [style=pointy] (75.center) to (77.center);
	\end{pgfonlayer}
\end{tikzpicture}}}}
  \end{gathered}
\end{equation}
This pictorial formalism is not just an intuitive sketch: as we will show, it
can be understood as a completely formal (graphical) syntax using the formalism
of \emph{string diagrams}~\cite{PiedeleuZanasi23}, in a way similar to how other
computational phenomena have been recently analysed e.g. in quantum theory
~\cite{coecke_kissinger_2017}, control theory ~\cite{CategoriesInControl,
  BonchiSZ17}, and digital circuit theory~\cite{GhicaCircuits}. 

It is intuitively clear how parametric lenses express aspects (I) and (II) above, whereas
(III) will be achieved by studying them in a space of `differentiable objects'
(in a sense that will be made precise). The main technical contribution of our
paper is showing how the various ingredients involved in learning (the model,
the optimiser, the error map and the learning rate) can be uniformly understood as being built from parametric lenses.


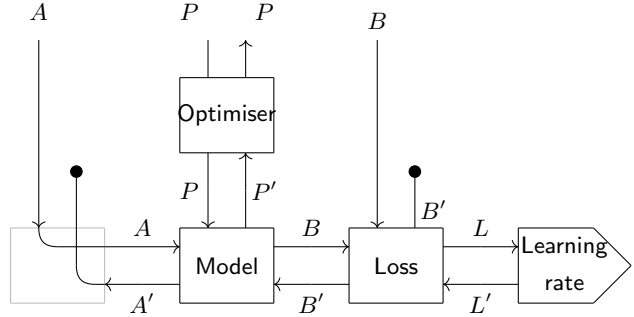
\begin{figure}[h]
  \centering
  {\scalefont{0.8}
    \begin{tikzpicture}
	\begin{pgfonlayer}{nodelayer}
		\node [style=none] (20) at (-3.5, 1) {};
		\node [style=none] (21) at (-3.5, -1) {};
		\node [style=none] (22) at (-1, 1) {};
		\node [style=none] (23) at (-1, -1) {};
		\node [style=none] (26) at (-4.5, 1) {$A$};
		\node [style=none] (27) at (-4.5, -1) {$A'$};
		\node [style=none] (28) at (0, 1) {$B$};
		\node [style=none] (29) at (0, -1) {$B'$};
		\node [style=none] (39) at (-4.5, 0.5) {};
		\node [style=none] (40) at (-3.5, 0.5) {};
		\node [style=none] (43) at (-3.5, -0.5) {};
		\node [style=none] (44) at (-4.5, -0.5) {};
		\node [style=none] (45) at (-1, 0.5) {};
		\node [style=none] (46) at (0, 0.5) {};
		\node [style=none] (57) at (0, -0.5) {};
		\node [style=none] (58) at (-1, -0.5) {};
		\node [style=none] (60) at (-2.75, 1) {};
		\node [style=none] (61) at (-1.75, 1) {};
		\node [style=none] (62) at (-2.75, 2) {};
		\node [style=none] (63) at (-1.75, 2) {};
		\node [style=none] (64) at (-3.25, 2) {$P$};
		\node [style=none] (65) at (-1.25, 2) {$P'$};
		\node [style=none] (66) at (1, 1) {};
		\node [style=none] (67) at (1, -1) {};
		\node [style=none] (68) at (3.5, 1) {};
		\node [style=none] (69) at (3.5, -1) {};
		\node [style=none] (72) at (4.5, 1) {$L$};
		\node [style=none] (73) at (4.5, -1) {$L'$};
		\node [style=none] (74) at (0, 0.5) {};
		\node [style=none] (75) at (1, 0.5) {};
		\node [style=none] (76) at (1, -0.5) {};
		\node [style=none] (77) at (0, -0.5) {};
		\node [style=none] (78) at (3.5, 0.5) {};
		\node [style=none] (79) at (4.5, 0.5) {};
		\node [style=none] (80) at (4.5, -0.5) {};
		\node [style=none] (81) at (3.5, -0.5) {};
		\node [style=none] (82) at (1.75, 1) {};
		\node [style=none] (83) at (2.75, 1) {};
		\node [style=none] (84) at (1.75, 6) {};
		\node [style=node] (85) at (2.75, 2.5) {};
		\node [style=none] (86) at (1.75, 6.5) {$B$};
		\node [style=none] (87) at (3.25, 1.5) {$B'$};
		\node [style=none] (89) at (2.25, 0) {\textsf{Loss}};
		\node [style=none] (92) at (5.5, 1) {};
		\node [style=none] (93) at (5.5, -1) {};
		\node [style=none] (94) at (7.5, -1) {};
		\node [style=none] (95) at (8.5, 0) {};
		\node [style=none] (96) at (7.5, 1) {};
		\node [style=none] (97) at (5.5, -0.5) {};
		\node [style=none] (98) at (5.5, 0.5) {};
		\node [style=none] (99) at (6.725, 0.5) {\textsf{Learning}};
		\node [style=none] (100) at (-3.5, 3) {};
		\node [style=none] (101) at (-1, 3) {};
		\node [style=none] (102) at (-1, 5) {};
		\node [style=none] (103) at (-3.5, 5) {};
		\node [style=none] (105) at (-2.75, 3) {};
		\node [style=none] (106) at (-1.75, 3) {};
		\node [style=none] (108) at (-2.75, 5) {};
		\node [style=none] (110) at (-1.75, 6) {};
		\node [style=none] (111) at (-2.75, 6) {};
		\node [style=none] (117) at (-1.75, 5) {};
		\node [style=none] (120) at (-8, 1) {};
		\node [style=none] (121) at (-8, -1) {};
		\node [style=none] (122) at (-5.5, 1) {};
		\node [style=none] (123) at (-5.5, -1) {};
		\node [style=none] (124) at (-8, 0.5) {};
		\node [style=none] (125) at (-8, -0.5) {};
		\node [style=none] (126) at (-5.5, 0.5) {};
		\node [style=none] (127) at (-5.5, -0.5) {};
		\node [style=none] (128) at (-7.25, 1) {};
		\node [style=none] (129) at (-6.25, 1) {};
		\node [style=none] (130) at (-2.25, 0) {\textsf{Model}};
		\node [style=none] (131) at (-6.75, 0.5) {};
		\node [style=none] (132) at (-6.25, 0) {};
		\node [style=none] (133) at (-5.75, -0.5) {};
		\node [style=none] (134) at (-7.25, 6) {};
		\node [style=node] (135) at (-6.25, 2.5) {};
		\node [style=none] (136) at (-3.25, 6.75) {$P$};
		\node [style=none] (137) at (-1.25, 6.75) {$P$};
		\node [style=none] (138) at (-7.25, 6.75) {$A$};
		\node [style=none] (139) at (-2.25, 4) {\textsf{Optimiser}};
		\node [style=none] (140) at (6.725, -0.475) {\textsf{rate}};
	\end{pgfonlayer}
	\begin{pgfonlayer}{edgelayer}
		\draw (20.center) to (22.center);
		\draw (22.center) to (23.center);
		\draw (23.center) to (21.center);
		\draw (21.center) to (20.center);
		\draw [style=pointy] (39.center) to (40.center);
		\draw (43.center) to (44.center);
		\draw (45.center) to (46.center);
		\draw [style=pointy] (57.center) to (58.center);
		\draw [style=pointy] (62.center) to (60.center);
		\draw (61.center) to (63.center);
		\draw (66.center) to (68.center);
		\draw (68.center) to (69.center);
		\draw (69.center) to (67.center);
		\draw (67.center) to (66.center);
		\draw [style=pointy] (74.center) to (75.center);
		\draw (76.center) to (77.center);
		\draw (78.center) to (79.center);
		\draw [style=pointy] (80.center) to (81.center);
		\draw [style=pointy] (84.center) to (82.center);
		\draw (83.center) to (85);
		\draw (92.center) to (96.center);
		\draw (96.center) to (95.center);
		\draw (95.center) to (94.center);
		\draw (94.center) to (93.center);
		\draw (93.center) to (92.center);
		\draw [style=pointy] (79.center) to (98.center);
		\draw (97.center) to (80.center);
		\draw (103.center) to (102.center);
		\draw (102.center) to (101.center);
		\draw (101.center) to (100.center);
		\draw (100.center) to (103.center);
		\draw (62.center) to (105.center);
		\draw [style=pointy] (63.center) to (106.center);
		\draw (111.center) to (108.center);
		\draw [style=pointy] (117.center) to (110.center);
		\draw [style=lightnone] (120.center) to (122.center);
		\draw [style=lightnone] (122.center) to (123.center);
		\draw [style=lightnone] (123.center) to (121.center);
		\draw [style=lightnone] (121.center) to (120.center);
		\draw (126.center) to (39.center);
		\draw [style=pointy] (44.center) to (127.center);
		\draw (126.center) to (131.center);
		\draw [in=-90, out=-180, looseness=1.25] (131.center) to (128.center);
		\draw [bend right=315, looseness=1.25] (133.center) to (132.center);
		\draw (127.center) to (133.center);
		\draw (132.center) to (129.center);
		\draw [style=pointy] (134.center) to (128.center);
		\draw (129.center) to (135);
	\end{pgfonlayer}
\end{tikzpicture}
    }
  
  \caption{The parametric lens that captures the learning process informally sketched in Figure~\ref{fig:informalGD}. Note each component is a lens itself, whose composition yields the interactions described in Figure~\ref{fig:informalGD}. Defining this picture formally will be the subject of Sections~\ref{section:components-as-lenses}-\ref{section:learning-with-lenses}. 
  }
  \label{fig:roadmap}
\end{figure}

We will use \emph{category theory} as the formal language to develop our notion of parametric lenses, and make Figure ~\ref{fig:roadmap} mathematically precise.
The categorical perspective brings several advantages, which are well-known, established principles in programming language semantics~\cite{plotkin_semantics, abramsky_semantics, selinger2001control}. Three of them are particularly important to our contribution, as they constitute distinctive advantages of our semantic foundations:
\begin{description}
	\item[$\mathbf{Abstraction}$] Our approach studies which categorical structures are
    sufficient to perform gradient-based learning. This analysis abstracts away
    from the standard case of neural networks in several different ways: as we
    will see, it encompasses other models (namely Boolean circuits), different
    kinds of optimisers (including Adagrad, Adam, Nesterov momentum), and error
    maps (including quadratic and softmax cross entropy loss). These can be all understood as parametric lenses, and different forms of learning result from their interaction.
	\item[$\mathbf{Uniformity}$] As seen in Figure~\ref{fig:informalGD}, learning involves ingredients that are seemingly quite different: a model, an optimiser, a loss map, etc. We will show how all these notions may be seen as instances of the categorical definition of a parametric lens, thus yielding a remarkably uniform description of the learning process, and supporting our claim of parametric lenses being a fundamental semantic structure of learning.
	\item[$\mathbf{Compositionality}$] The use of categorical structures to describe
    computation naturally enables \emph{compositional reasoning} whereby complex
    systems are analysed in terms of smaller, and hence easier to understand,
    components. Compositionality is a fundamental tenet of programming language
    semantics; in the last few years, it has found application in the  study of
    diverse kinds of computational models, across different fields--- see e.g.
    \cite{FunctorialDataMigration, CompositionalGameTheory,
      coecke_kissinger_2017,BonchiSZ17}. As made evident by
    Figure~\ref{fig:roadmap}, our approach models a neural network as a
    parametric lens, resulting from the \emph{composition} of simpler parametric
    lenses, capturing the different ingredients involved in the learning
    process. Moreover, as all the simpler parametric lenses are themselves
    composable, one may engineer a different learning process by simply plugging
    a new lens on the left or right of existing ones. This means that one can glue together smaller and relatively simple networks to create larger and more sophisticated neural networks.
\end{description}

\noindent We now give a synopsis of our contributions:
\begin{itemize}[leftmargin=*]
	\item In Section~\ref{section:background}, we introduce the tools necessary to
    define our notion of \textbf{parametric lens}. First, in
    Section~\ref{sec:para}, we introduce a notion of parametric categories, which amounts to a functor $\Para(-)$ turning a category $\Ca$ into one $\Para(\Ca)$ of `parametric $\Ca$-maps'. Second, we recall \emph{lenses} (Section~\ref{sec:lenses}). In a nutshell, a lens is a categorical morphism equipped with operations to view and update values in a certain data structure. Lenses play a prominent role in functional programming~\cite{lenses_functional}, as well as in the foundations of database theory~\cite{lenses_database} and more recently game theory~\cite{CompositionalGameTheory}. Considering lenses in $\Ca$ simply amounts to the application of a functorial construction $\Lens(-)$, yielding $\Lens(\Ca)$. Finally, we recall the notion of a \emph{cartesian reverse differential category} (CRDC): a categorical structure axiomatising the notion of differentiation \cite{CRDC} (Section~\ref{sec:crdc}). We wrap up in Section~\ref{sec:paralenses}, by combining these ingredients into the notion of parametric lens, formally defined as a morphism in $\Para(\Lens(\Ca))$ for a CRDC $\Ca$. In terms of our desiderata (I)-(III) above, note that $\Para(-)$ accounts for (I), $\Lens(-)$ accounts for (II), and the CRDC structure accounts for (III).
	\item As seen in Figure~\ref{fig:informalGD}, in the learning process there
    are many components at work: the model, the optimiser, the loss map, the
    learning rate, etc.. In Section~\ref{section:components-as-lenses}, we show
    how the notion of parametric lens provides a uniform characterisation for such components. Moreover, for each of them, we show how different variations appearing in the literature become instances of our abstract characterisation. The plan is as follows: 
	 \begin{itemize}[label=$\circ$,leftmargin=*]
	 \item In Section~\ref{subsec:models}, we show how the combinatorial \textbf{model} subject of the training can be seen as a parametric lens. The conditions we provide are met by the `standard' case of neural networks, but also enables the study of learning for other classes of models. In particular,  another instance are Boolean circuits: learning of these structures is relevant to binarisation~\cite{BinaryConnect} and it has been explored recently using a categorical approach~\cite{rda}, which turns out to be a particular case of our framework. We continue by describing internals of a model, and translating several exampls of models in deep learning to their categorical form. This includes linear layers, biases, activations, convolutional layers, but also general techniques such as weight tying and batching.
	 \item In Section~\ref{subsec:lossmaps}, we show how the \textbf{loss maps}
     associated with training are also parametric lenses. Our
     approach covers the cases of quadratic error, Boolean error, Softmax cross
     entropy, but also the `dot product loss' associated with the phenomenon of deep
     dreaming~\cite{DeepDreaming1, DeepDreaming2, DeepDreaming3, DeepDreaming4}.
		 \item In Section~\ref{subsec:learningrate}, we model the \textbf{learning rate} as a parametric lens. This analysis also allows us to contrast how learning rate is handled in the `real-valued' case of neural networks with respect to the `Boolean-valued' case of Boolean circuits.
\item In Section~\ref{subsec:optimisers}, we show how \textbf{optimisers} can be
  modelled as `reparameterisations' of models as parametric lenses. As case
  studies, in addition to basic gradient ascent and descent, we consider the stateful
  variants: Momentum~\cite{Momentum}, Nesterov Momentum~\cite{NesterovMomentum}, Adagrad~\cite{Adagrad}, and Adam (Adaptive Moment Estimation)~\cite{Adam}, as well as optimiser composition (Subsection \ref{subsubsec:can_we_compose_optimisers}). Also, on Boolean circuits, we show how the reverse derivative ascent of~\cite{rda} can be also regarded in such way.
\end{itemize}
	\item In Section~\ref{section:learning-with-lenses}, we study how the composition of the lenses defined in Section~\ref{section:components-as-lenses} yields a description of different kinds of learning processes.
	\begin{itemize}[label=$\circ$,leftmargin=*]
	\item Section~\ref{subsec:learning-parameters} is dedicated to modelling
    supervised \textbf{learning of parameters}, in the way described in
    Figure~\ref{fig:informalGD}. This amounts essentially to study of the composite
    of lenses expressed in Figure~\ref{fig:roadmap}, for different choices of
    the various components. In particular we look at (i) quadratic loss with basic
    gradient descent, (ii) softmax cross entropy loss with basic gradient
    descent, (iii) quadratic loss with Nesterov momentum, and (iv) learning in
    Boolean circuits with XOR loss and basic gradient ascent.
    \item In Section~\ref{subsec:unsupervised-learning} we describe how a system traditionally considered as unsupervised can be recast to its supervised form: Generative Adversarial Networks (\cite{GAN,arjovsky_wasserstein_2017}). We define this model abstractly as a parametric lens, and describe how a particular instantiation thereof --- Wasserstein GAN (\cite{arjovsky_wasserstein_2017}) --- arises as a supervised learning system with the dot product loss and the gradient descent-ascent optimiser.
	\item In order to showcase the flexibility of our approach, in Section~\ref{subsec:deep_dreaming} we depart from our `core' case study of parameter learning, and turn attention to supervised \textbf{learning of inputs}, also called \textbf{deep dreaming} --- the idea behind this technique is that, instead of the network parameters, one updates the inputs, in order to elicit a particular interpretation~\cite{DeepDreaming1, DeepDreaming2,
  DeepDreaming3, DeepDreaming4}. Deep dreaming can be easily expressed within our approach, with a different rearrangement of the parametric lenses involved in the learning process, see \eqref{eq:deep_dreaming} below. The abstract viewpoint of categorical semantics provides a mathematically precise and visually captivating description of the differences between the usual parameter learning process and deep dreaming.
	\end{itemize} 
\item In Section~\ref{section:implementation} we describe a proof-of-concept Python \textbf{implementation}, available at \cite{python-library}, based on the theory developed in this paper. This code is intended to show more concretely the payoff of our approach. Model architectures, as well as the various components participating in the learning process, are now expressed in a uniform, principled mathematical language, in terms of lenses. As a result, computing network gradients is greatly simplified, as it amounts to lens composition. Moreover, the modularity of this approach allows one to more easily tune the various parameters of training. 
	
	We show our library via a number of experiments, and prove
  correctness by achieving accuracy on par with an equivalent model in Keras, a mainstream deep learning framework \cite{Keras}. In particular, we create a working
  non-trivial neural network model for the MNIST image-classification problem \cite{mnist}. \item Finally, in Sections~\ref{section:related-work} and \ref{section:conclusions}, we discuss related and future work.
\end{itemize} 

Note this paper extends a previous conference version~\cite{CruttwellGGWZ22}.
Section~\ref{subsec:models} has been extended with examples of different architectures and techniques in deep learning, while Section~\ref{subsec:unsupervised-learning}, which considers unsupervised learning, and Remark~\ref{rmk:coalgebra}, followed by a discussion on the axioms of CRDCs needed in our framework, are new. Also, we included missing proofs and complementary background material (see in particular Appendices~\ref{app:para}-\ref{app:CRDC}).

\section{Categorical Toolkit}
\label{section:background}

In this section we describe the three categorical components of our framework, each corresponding to an aspect of gradient-based learning: (I) the $\Para$ construction (Section~\ref{sec:para}), which builds a category of parametric maps, (II) the $\Lens$ construction, which builds a category of ``bidirectional'' maps (Section \ref{sec:lenses}), and (III) the combination of these two constructions into the notion of ``parametric lenses'' (Section \ref{sec:paralenses}). Finally (IV) we recall Cartesian reverse differential categories --- categories equipped with an abstract gradient operator. 

\paragraph*{Notation} We shall use $f;g$ for sequential composition of morphisms $f \colon A \to B$ and $g \colon B \to C$ in a category, $1_A$ for the identity morphism on $A$, and $I$ for the unit object of a symmetric monoidal category.



\subsection{Parametric Maps}\label{sec:para}

In supervised learning one is typically interested in approximating a function $g: \Rb^n \to \Rb^m$
  for some $n$ and $m$.
  To do this, one begins by building a neural network, which is a smooth map 
	$f: \Rb^p \times \Rb^n \to \Rb^m$
where $\Rb^p$ is the set of possible weights of that neural network.  Then one
looks for a value of $q \in \Rb^p$ such that the function $f(q,-): \Rb^n \to
\Rb^m$ closely approximates $g$.  We formalise these maps categorically via the $\Para$ construction
\cite{BackpropAsFunctor, CompDL, TowardsCatCyber, CatsOfPossibleWorlds}.  

\begin{defi}[Parametric category] \label{def:paracat}
Let $(\Ca,\otimes, I)$ be a strict\footnote{One can also define $\Para(\Ca)$ in the case when $\Ca$ is non-strict; however, the result would be not a category but a bicategory.} symmetric monoidal category. We define a category $\Para(\Ca)$ with 

\vspace*{-0.3cm}
\begin{itemize}
  \item objects those of $\Ca$;
  \item a map from $A$ to $B$ is a pair $(P,f)$, with $P$ an object of $\Ca$ and $f: P \otimes A \to B$;
  \item the identity on $A$ is the pair $(I,1_A)$ (since $\otimes$ is strict monoidal, $I \otimes A = A$);
  \item the composite of maps $(P,f): A \to B$ and $(P',f'): B \to C$ is the pair $(P' \otimes P, (1_{P'} \otimes f);f')$;
\end{itemize}
\end{defi}

\begin{exa}\label{ex:parasmooth}
Take the category $\Smooth$ whose objects are natural numbers and whose morphisms $f: n \to m$ are smooth maps from $\Rb^n$ to $\Rb^m$.
This is a strict symmetric monoidal category with product given by addition.
As described above, the category $\Para(\Smooth)$ can be thought of as a category of neural networks: a map in this category from $n$ to $m$ consists of a choice of $p$ and a map $f: \Rb^p \times \Rb^n \to \Rb^m$ with $\Rb^p$ representing the set of possible weights of the neural network.  
\end{exa}

As we will see in the next sections, the interplay of the various components at work in the learning process becomes much clearer once represented the morphisms of $\Para(\Ca)$ using the pictorial formalism of \emph{string diagrams}, which we now recall. In fact,  we will mildly massage the traditional notation for string diagrams (below left), by representing a morphism $f \colon A \to B$ in
$\Para(\Ca)$ as below right.
\[
  \scalebox{0.7}{\begin{tikzpicture}
	\begin{pgfonlayer}{nodelayer}
		\node [style={medium_box}] (0) at (0, 0) {$f$};
		\node [style=none] (1) at (-2, 1) {};
		\node [style=none] (2) at (-2, -1) {};
		\node [style=none] (3) at (-0.5, 0.475) {};
		\node [style=none] (4) at (-0.5, -0.475) {};
		\node [style=none] (5) at (0.25, 0) {};
		\node [style=none] (6) at (2, 0) {};
		\node [style=none] (7) at (-3, 1) {$P$};
		\node [style=none] (8) at (-3, -1) {$A$};
		\node [style=none] (9) at (3, 0) {$B$};
		\node [style=none] (10) at (0, 1.75) {};
	\end{pgfonlayer}
	\begin{pgfonlayer}{edgelayer}
		\draw [style=pointy, in=180, out=0, looseness=0.75] (1.center) to (3.center);
		\draw [style=pointy, in=-180, out=0, looseness=0.75] (2.center) to (4.center);
		\draw [style=pointy] (5.center) to (6.center);
	\end{pgfonlayer}
\end{tikzpicture}}
\qquad \qquad\qquad \scalebox{0.7}{\begin{tikzpicture}
	\begin{pgfonlayer}{nodelayer}
		\node [style={medium_box}] (0) at (0, 0) {$f$};
		\node [style=none] (1) at (0, 1.5) {};
		\node [style=none] (2) at (-2, 0) {};
		\node [style=none] (3) at (0, 0.75) {};
		\node [style=none] (4) at (-0.5, 0) {};
		\node [style=none] (5) at (0.25, 0) {};
		\node [style=none] (6) at (2, 0) {};
		\node [style=none] (7) at (0, 2) {$P$};
		\node [style=none] (8) at (-3, 0) {$A$};
		\node [style=none] (9) at (3, 0) {$B$};
	\end{pgfonlayer}
	\begin{pgfonlayer}{edgelayer}
		\draw [style=pointy, in=-180, out=0, looseness=0.75] (2.center) to (4.center);
		\draw [style=pointy] (5.center) to (6.center);
		\draw [style=pointy] (1.center) to (3.center);
	\end{pgfonlayer}
\end{tikzpicture}}
\]
This is to emphasise the special role played by $P$, reflecting the fact that in machine learning data and parameters have different semantics. String diagrammatic notations also allows to neatly represent composition of maps $(P,f): A \to B$ and $(P',f'): B \to C$ (below left), and ``reparameterisation'' of $(P,f): A \to B$ by a map $\alpha: Q \to P$ (below right), yielding a new map $(Q, (\alpha \otimes 1_A);f): A \to B$. 

\begin{equation}\label{eq:para-composition-doublecat}
  \begin{gathered}
    \scalebox{0.8}{\begin{tikzpicture}
	\begin{pgfonlayer}{nodelayer}
		\node [style={medium_box}] (0) at (0, 0) {$f$};
		\node [style=none] (1) at (0, 2) {};
		\node [style=none] (2) at (-2, 0) {};
		\node [style=none] (3) at (0, 0.25) {};
		\node [style=none] (4) at (-0.25, 0) {};
		\node [style=none] (5) at (0.25, 0) {};
		\node [style=none] (6) at (2, 0) {};
		\node [style=none] (7) at (0, 3) {$P$};
		\node [style=none] (8) at (-3, 0) {$A$};
		\node [style=none] (9) at (2, 0.75) {$B$};
		\node [style={medium_box}] (10) at (4, 0) {$f'$};
		\node [style=none] (11) at (4, 2) {};
		\node [style=none] (12) at (2, 0) {};
		\node [style=none] (13) at (4, 0.25) {};
		\node [style=none] (14) at (3.75, 0) {};
		\node [style=none] (15) at (4.25, 0) {};
		\node [style=none] (16) at (6, 0) {};
		\node [style=none] (17) at (4, 3) {$P'$};
		\node [style=none] (19) at (7, 0) {$C$};
	\end{pgfonlayer}
	\begin{pgfonlayer}{edgelayer}
		\draw [in=-180, out=0, looseness=0.75] (2.center) to (4.center);
		\draw (5.center) to (6.center);
		\draw (1.center) to (3.center);
		\draw [in=-180, out=0, looseness=0.75] (12.center) to (14.center);
		\draw (15.center) to (16.center);
		\draw (11.center) to (13.center);
	\end{pgfonlayer}
\end{tikzpicture}}
    \qquad \qquad \qquad
    \scalebox{0.8}{\begin{tikzpicture}
	\begin{pgfonlayer}{nodelayer}
		\node [style={medium_box}] (0) at (0, 0) {$f$};
		\node [style=none] (1) at (0, 2.75) {};
		\node [style=none] (2) at (-2, 0) {};
		\node [style=none] (3) at (0, 0.25) {};
		\node [style=none] (4) at (-0.25, 0) {};
		\node [style=none] (5) at (0.25, 0) {};
		\node [style=none] (6) at (2, 0) {};
		\node [style=none] (7) at (-0.5, 1.5) {$P$};
		\node [style=none] (8) at (-3, 0) {$A$};
		\node [style=none] (9) at (3, 0) {$B$};
		\node [style={medium_box}] (10) at (0, 3) {$\alpha$};
		\node [style=none] (11) at (0, 3.25) {};
		\node [style=none] (12) at (0, 4.5) {};
		\node [style=none] (13) at (0, 5) {$Q$};
	\end{pgfonlayer}
	\begin{pgfonlayer}{edgelayer}
		\draw [in=-180, out=0, looseness=0.75] (2.center) to (4.center);
		\draw (5.center) to (6.center);
		\draw (1.center) to (3.center);
		\draw (12.center) to (11.center);
	\end{pgfonlayer}
\end{tikzpicture}}
  \end{gathered}
\end{equation} 
Intuitively, reparameterisation changes the parameter space of $(P,f): A \to B$ to some other object $Q$, via some map
$\alpha: Q \to P$.  We shall see later that gradient descent and its many
variants can naturally be viewed as reparameterisations.

Note coherence rules in combining the two operations in \eqref{eq:para-composition-doublecat}  just work as expected, as these diagrams can be ultimately `compiled' down to string diagrams for monoidal categories.

\subsection{Lenses}\label{sec:lenses}

In machine learning (or even learning in general) it is fundamental that information flows
both forwards and backwards: the `forward' flow corresponds to a model's predictions,
and the `backwards' flow to \emph{corrections} to the model.
The category of lenses is the ideal setting to capture this type of structure,
as it is a category consisting of maps with both a ``forward'' and a
``backward'' part.

\begin{defi}
For any Cartesian category $\Ca$, the category of (bimorphic) lenses in $\Ca$, $\Lens(\Ca)$, is
the category with the following data: 
\vspace*{-0.3cm}
\begin{itemize}
  \item Objects are pairs $(A,A')$ of objects in $\Ca$, written as $\diset{A}{A'}$;
  \item A map from $\diset{A}{A'}$ to $\diset{B}{B'}$ consists of a pair $(f, f^*)$ (also written as $\diset{f}{f^*}$) where $f: A \to B$ (called the \textbf{get} or \textbf{forward} part of the lens) and $f^*: A \times B' \to A'$ (called the \textbf{put} or \textbf{backwards} part of the lens);
  \item The identity on $\diset{A}{A'}$ is the pair $\diset{1_A}{\pi_1}$;
  \item The composite of $\diset{f}{f^*} : \diset{A}{A'} \to \diset{B}{B'}$ and $\diset{g}{g^*}: \diset{B}{B'} \to \diset{C}{C'}$ is given by get $f ; g$ and put $\<\pi_0,\<\pi_0;f, \pi_1\>;g^*\>;f^*$. 
\end{itemize}

\end{defi}

The embedding of $\Lens(\Ca)$ into the category of Tambara modules over $\Ca$ (see \cite[Thm. 23]{graphical_optics}) provides a rich string diagrammatic language, in which lenses may be represented with forward/backward wires indicating the information flow.  In this language, a morphism $\diset{f}{f^*}: \diset{A}{A'} \to \diset{B}{B'}$ is written as below left, which can be `expanded' as below right.

\begin{equation*}
  \scalebox{0.8}{\begin{tikzpicture}
	\begin{pgfonlayer}{nodelayer}
		\node [style=none] (20) at (-1, 1) {};
		\node [style=none] (21) at (-1, -1) {};
		\node [style=none] (22) at (1, 1) {};
		\node [style=none] (23) at (1, -1) {};
		\node [style=none] (26) at (-3, 0.5) {$A$};
		\node [style=none] (27) at (-3, -0.5) {$A'$};
		\node [style=none] (28) at (3, 0.5) {$B$};
		\node [style=none] (29) at (3, -0.5) {$B'$};
		\node [style=none] (39) at (-2, 0.5) {};
		\node [style=none] (40) at (-1, 0.5) {};
		\node [style=none] (43) at (-1, -0.5) {};
		\node [style=none] (44) at (-2, -0.5) {};
		\node [style=none] (45) at (1, 0.5) {};
		\node [style=none] (46) at (2, 0.5) {};
		\node [style=none] (57) at (2, -0.5) {};
		\node [style=none] (58) at (1, -0.5) {};
		\node [style=none] (59) at (0, 0) {$(f, f^*)$};
	\end{pgfonlayer}
	\begin{pgfonlayer}{edgelayer}
		\draw (20.center) to (22.center);
		\draw (22.center) to (23.center);
		\draw (23.center) to (21.center);
		\draw (21.center) to (20.center);
		\draw [style=pointy] (39.center) to (40.center);
		\draw [style=pointy] (43.center) to (44.center);
		\draw [style=pointy] (45.center) to (46.center);
		\draw [style=pointy] (57.center) to (58.center);
	\end{pgfonlayer}
\end{tikzpicture}} \qquad \qquad \scalebox{0.6}{\begin{tikzpicture}
	\begin{pgfonlayer}{nodelayer}
		\node [style=none] (20) at (-5, 4) {};
		\node [style=none] (21) at (-5, -4) {};
		\node [style=none] (22) at (5, 4) {};
		\node [style=none] (23) at (5, -4) {};
		\node [style=none] (26) at (-7, 1) {$A$};
		\node [style=none] (27) at (-7, -2) {$A'$};
		\node [style=none] (28) at (7, 2) {$B$};
		\node [style=none] (29) at (7, -2.5) {$B'$};
		\node [style=none] (30) at (-1, 3) {};
		\node [style=none] (31) at (-1, 1) {};
		\node [style=none] (32) at (1, 1) {};
		\node [style=none] (33) at (1, 3) {};
		\node [style=none] (34) at (1, 2) {};
		\node [style=none] (35) at (-1, 2) {};
		\node [style=none] (39) at (-6, 1) {};
		\node [style=none] (40) at (-5, 1) {};
		\node [style=none] (41) at (-3, 2) {};
		\node [style=none] (42) at (-3, 0) {};
		\node [style=none] (43) at (-5, -2) {};
		\node [style=none] (44) at (-6, -2) {};
		\node [style=none] (45) at (5, 2) {};
		\node [style=none] (46) at (6, 2) {};
		\node [style=none] (47) at (3, 0) {};
		\node [style=none] (48) at (3, -1.5) {};
		\node [style=none] (49) at (-1, -1) {};
		\node [style=none] (50) at (-1, -3) {};
		\node [style=none] (51) at (1, -3) {};
		\node [style=none] (52) at (1, -1) {};
		\node [style=none] (53) at (1, -1.5) {};
		\node [style=none] (54) at (-1, -2) {};
		\node [style=none] (55) at (3, -2.5) {};
		\node [style=none] (56) at (1, -2.5) {};
		\node [style=none] (57) at (6, -2.5) {};
		\node [style=none] (58) at (5, -2.5) {};
		\node [style=none] (59) at (0, 2) {$f$};
		\node [style=none] (60) at (0, -2) {$f^*$};
		\node [style=none] (61) at (-5, 1) {};
		\node [style=none] (62) at (-3, 2) {};
		\node [style=none] (63) at (-3, 0) {};
		\node [style=copy2] (64) at (-4, 1) {};
		\node [style=none] (65) at (-5, 1) {};
		\node [style=none] (66) at (-3, 2) {};
		\node [style=none] (67) at (-3, 0) {};
	\end{pgfonlayer}
	\begin{pgfonlayer}{edgelayer}
		\draw (20.center) to (22.center);
		\draw (22.center) to (23.center);
		\draw (23.center) to (21.center);
		\draw (21.center) to (20.center);
		\draw (39.center) to (40.center);
		\draw [style=pointy] (43.center) to (44.center);
		\draw [style=pointy] (41.center) to (35.center);
		\draw (34.center) to (45.center);
		\draw [style=pointy] (45.center) to (46.center);
		\draw (30.center) to (33.center);
		\draw (33.center) to (34.center);
		\draw (34.center) to (32.center);
		\draw (32.center) to (31.center);
		\draw (31.center) to (35.center);
		\draw (35.center) to (30.center);
		\draw [style=pointy] (42.center) to (47.center);
		\draw [style=pointy, bend left=90, looseness=1.75] (47.center) to (48.center);
		\draw (49.center) to (52.center);
		\draw (52.center) to (53.center);
		\draw (53.center) to (51.center);
		\draw (51.center) to (50.center);
		\draw (50.center) to (54.center);
		\draw (54.center) to (49.center);
		\draw [style=pointy] (48.center) to (53.center);
		\draw [style=pointy] (55.center) to (56.center);
		\draw (54.center) to (43.center);
		\draw (57.center) to (58.center);
		\draw (58.center) to (55.center);
		\draw [style=pointy, bend left=45, looseness=1.25] (64) to (66.center);
		\draw [style=pointy, bend right=45, looseness=1.25] (64) to (67.center);
		\draw (65.center) to (64);
	\end{pgfonlayer}
\end{tikzpicture}}
\end{equation*}
It is clear in this language how to describe the composite of $\diset{f}{f^*} : \diset{A}{A'} \to \diset{B}{B'}$ and $\diset{g}{g^*}: \diset{B}{B'} \to \diset{C}{C'}$:
\begin{equation}\label{eq:lens-composition}
  \scalebox{0.55}{\begin{tikzpicture}
	\begin{pgfonlayer}{nodelayer}
		\node [style=none] (20) at (-5, 4) {};
		\node [style=none] (21) at (-5, -4) {};
		\node [style=none] (22) at (5, 4) {};
		\node [style=none] (23) at (5, -4) {};
		\node [style=none] (26) at (-7, 1) {$A$};
		\node [style=none] (27) at (-7, -2) {$A'$};
		\node [style=none] (30) at (-1, 3) {};
		\node [style=none] (31) at (-1, 1) {};
		\node [style=none] (32) at (1, 1) {};
		\node [style=none] (33) at (1, 3) {};
		\node [style=none] (34) at (1, 2) {};
		\node [style=none] (35) at (-1, 2) {};
		\node [style=copy2] (36) at (-4, 1) {};
		\node [style=none] (39) at (-6, 1) {};
		\node [style=none] (40) at (-5, 1) {};
		\node [style=none] (41) at (-3, 2) {};
		\node [style=none] (42) at (-3, 0) {};
		\node [style=none] (43) at (-5, -2) {};
		\node [style=none] (44) at (-6, -2) {};
		\node [style=none] (45) at (3, 2) {};
		\node [style=none] (46) at (6, 2) {};
		\node [style=none] (47) at (3, 0) {};
		\node [style=none] (48) at (3, -1.5) {};
		\node [style=none] (49) at (-1, -1) {};
		\node [style=none] (50) at (-1, -3) {};
		\node [style=none] (51) at (1, -3) {};
		\node [style=none] (52) at (1, -1) {};
		\node [style=none] (53) at (1, -1.5) {};
		\node [style=none] (54) at (-1, -2) {};
		\node [style=none] (55) at (3, -2.5) {};
		\node [style=none] (56) at (1, -2.5) {};
		\node [style=none] (57) at (6, -2.5) {};
		\node [style=none] (58) at (5, -2.5) {};
		\node [style=none] (59) at (0, 2) {$f$};
		\node [style=none] (60) at (0, -2) {$f^*$};
		\node [style=none] (61) at (9, 4) {};
		\node [style=none] (62) at (9, -4) {};
		\node [style=none] (63) at (19, 4) {};
		\node [style=none] (64) at (19, -4) {};
		\node [style=none] (67) at (21, 2) {$C$};
		\node [style=none] (68) at (21, -2.5) {$C'$};
		\node [style=none] (69) at (13, 3) {};
		\node [style=none] (70) at (13, 1) {};
		\node [style=none] (71) at (15, 1) {};
		\node [style=none] (72) at (15, 3) {};
		\node [style=none] (73) at (15, 2) {};
		\node [style=none] (74) at (13, 2) {};
		\node [style=none] (76) at (8, 1) {};
		\node [style=none] (77) at (9, 1) {};
		\node [style=none] (78) at (11, 2) {};
		\node [style=none] (79) at (11, 0) {};
		\node [style=none] (80) at (9, -2) {};
		\node [style=none] (81) at (8, -2) {};
		\node [style=none] (82) at (19, 2) {};
		\node [style=none] (83) at (20, 2) {};
		\node [style=none] (84) at (17, 0) {};
		\node [style=none] (85) at (17, -1.5) {};
		\node [style=none] (86) at (13, -1) {};
		\node [style=none] (87) at (13, -3) {};
		\node [style=none] (88) at (15, -3) {};
		\node [style=none] (89) at (15, -1) {};
		\node [style=none] (90) at (15, -1.5) {};
		\node [style=none] (91) at (13, -2) {};
		\node [style=none] (92) at (17, -2.5) {};
		\node [style=none] (93) at (15, -2.5) {};
		\node [style=none] (94) at (20, -2.5) {};
		\node [style=none] (95) at (19, -2.5) {};
		\node [style=none] (96) at (14, 2) {$g$};
		\node [style=none] (97) at (14, -2) {$g^*$};
		\node [style=none] (98) at (7, 2.5) {$B$};
		\node [style=none] (99) at (7, -3.5) {$B'$};
		\node [style=copy2] (100) at (10, 1) {};
		\node [style=none] (101) at (9, 1) {};
		\node [style=none] (102) at (11, 2) {};
		\node [style=none] (103) at (11, 0) {};
	\end{pgfonlayer}
	\begin{pgfonlayer}{edgelayer}
		\draw (20.center) to (22.center);
		\draw (22.center) to (23.center);
		\draw (23.center) to (21.center);
		\draw (21.center) to (20.center);
		\draw (39.center) to (40.center);
		\draw [style=pointy] (43.center) to (44.center);
		\draw [style=pointy, bend left=45, looseness=1.25] (36) to (41.center);
		\draw [style=pointy, bend right=45, looseness=1.25] (36) to (42.center);
		\draw (40.center) to (36);
		\draw [style=pointy] (41.center) to (35.center);
		\draw (34.center) to (45.center);
		\draw [style=pointy] (45.center) to (46.center);
		\draw (30.center) to (33.center);
		\draw (33.center) to (34.center);
		\draw (34.center) to (32.center);
		\draw (32.center) to (31.center);
		\draw (31.center) to (35.center);
		\draw (35.center) to (30.center);
		\draw [style=pointy] (42.center) to (47.center);
		\draw [style=pointy, bend left=90, looseness=1.75] (47.center) to (48.center);
		\draw (49.center) to (52.center);
		\draw (52.center) to (53.center);
		\draw (53.center) to (51.center);
		\draw (51.center) to (50.center);
		\draw (50.center) to (54.center);
		\draw (54.center) to (49.center);
		\draw [style=pointy] (48.center) to (53.center);
		\draw [style=pointy] (55.center) to (56.center);
		\draw (54.center) to (43.center);
		\draw (57.center) to (58.center);
		\draw (58.center) to (55.center);
		\draw (61.center) to (63.center);
		\draw (63.center) to (64.center);
		\draw (64.center) to (62.center);
		\draw (62.center) to (61.center);
		\draw (76.center) to (77.center);
		\draw [style=pointy] (80.center) to (81.center);
		\draw [style=pointy] (78.center) to (74.center);
		\draw (73.center) to (82.center);
		\draw [style=pointy] (82.center) to (83.center);
		\draw (69.center) to (72.center);
		\draw (72.center) to (73.center);
		\draw (73.center) to (71.center);
		\draw (71.center) to (70.center);
		\draw (70.center) to (74.center);
		\draw (74.center) to (69.center);
		\draw [style=pointy] (79.center) to (84.center);
		\draw [style=pointy, bend left=90, looseness=1.75] (84.center) to (85.center);
		\draw (86.center) to (89.center);
		\draw (89.center) to (90.center);
		\draw (90.center) to (88.center);
		\draw (88.center) to (87.center);
		\draw (87.center) to (91.center);
		\draw (91.center) to (86.center);
		\draw [style=pointy] (85.center) to (90.center);
		\draw [style=pointy] (92.center) to (93.center);
		\draw (91.center) to (80.center);
		\draw (94.center) to (95.center);
		\draw (95.center) to (92.center);
		\draw [in=180, out=0] (46.center) to (76.center);
		\draw [in=-180, out=0] (57.center) to (81.center);
		\draw [style=pointy, bend left=45, looseness=1.25] (100) to (102.center);
		\draw [style=pointy, bend right=45, looseness=1.25] (100) to (103.center);
		\draw (101.center) to (100);
	\end{pgfonlayer}
\end{tikzpicture}}
\end{equation}

\begin{rem}
  \label{remark:lens_monoidal}
Note $\Lens(\Ca)$ is a monoidal category, with $\diset{A}{A'} \otimes \diset{B}{B'}$ defined as $\diset{A \times B}{A' \times B'}$.
However, in general $\Lens(\Ca)$ is not itself Cartesian.
This is easy to see when looking at even a terminal object: if $T$ is a terminal object in $\Ca$, then in general $\diset{T}{T}$ will not be a terminal object in $\Lens(\Ca)$ --- it if was, there would be a unique lens $\diset{!_A}{!_A^*}: \diset{A}{A'} \to \diset{T}{T}$ whose put part would need to be a (unique) map $A \times T \to A'$, but in general there are many such maps.  
\end{rem}

\subsection{Parametric Lenses}\label{sec:paralenses}

The fundamental category where supervised learning takes place is the composite $\Para(\Lens(\Ca))$ 
of the two constructions in the previous sections:
\begin{defi} The category $\Para(\Lens(\Ca))$ of \textbf{parametric lenses} on $\Ca$ is defined as follows.
\vspace*{-0.3cm}
  \begin{itemize}
    \item Its objects are pairs of objects $\diset{A}{A'}$ of objects from $\Ca$;
    \item  A morphism from $\diset{A}{A'}$ to $\diset{B}{B'}$, called a parametric
    lens\footnote{In \cite{BackpropAsFunctor}, these are called
      \textit{learners}. However, in this paper we study them in a much broader
      light; see Section \ref{section:related-work}.}, is a choice of parameter pair $\diset{P}{P'}$ and a lens $\diset{f}{f^*}: \diset{P}{P'} \times \diset{A}{A'} \to \diset{B}{B'}$ where $f: P \times A \to B$ and $f^*: P \times A \times B' \to P' \times A'$
  \end{itemize}
\end{defi}
String diagrams for parametric lenses are built by simply composing the graphical languages of the previous two sections --- see \eqref{eq:paralens}, where respectively a morphism, a composition of morphisms, and a reparameterisation are depicted.

Given a generic morphism in $\Para(\Lens(\Ca))$ as depicted in \eqref{eq:paralens} on the left, one can see how it is possible to  ``learn'' new values from $f$: it takes as input an input $A$, a parameter $P$, and a change $B'$, and outputs a change in $A$, a value of $B$, and a change $P'$.  This last element is the key component for supervised learning: intuitively, it says how to change the parameter values to get the neural network closer to the true value of the desired function.

The question, then, is how one is to define such a parametric lens given nothing more than a neural network, ie., a parametric map $(P,f): A \to B$.  This is precisely what the gradient operation provides, and its generalization to categories is explored in the next subsection.

\subsection{Cartesian Reverse Differential Categories}\label{sec:crdc}

Fundamental to all types of gradient-based learning is, of course, the gradient
operation.
In most cases this gradient operation is performed in the category of smooth
maps between Euclidean spaces.
However, recent work \cite{rda} has shown that gradient-based learning can also
work well in other categories; for example, in a category of boolean circuits.
Thus, to encompass these examples in a single framework, we will work
in a category with an abstract gradient operation.



\begin{defi}\label{defn:CRDC} A \textbf{Cartesian left additive category}~\cite[Defn. 1]{CRDC} 
    consists of a category $\Ca$ with chosen finite products (including a
    terminal object), and an addition operation and zero morphism in each
    homset, satisfying various axioms. A \textbf{Cartesian reverse differential category} (CRDC)~\cite[Defn. 13]{CRDC} consists of a Cartesian left additive category $\Ca$, together with an operation which provides, for each map $f: A \to B \mbox{ in } \Ca$, a map $R[f]: A \times B \to A$ satisfying various axioms (see (Def.\ \ref{def:crdc})).
\end{defi}

For $f: A \to B$, the pair $\diset{f}{R[f]}$ forms a lens from $\diset{A}{A}$ to $\diset{B}{B}$. We will pursue the idea that $R[f]$ acts as backwards map, thus giving a means to ``learn''$f$.

Note that assigning type $A \times B \to A$ to $R[f]$ hides some relevant information: $B$-values in the domain and $A$-values in the codomain of $R[f]$ do not play the same role as values of the same types in $f \colon A \to B$: in $R[f]$, they really take in a tangent vector at $B$ and output a tangent vector at $A$ (\emph{cf.} the definition of $R[f]$ in $\Smooth$, Example~\ref{ex:smoothrdc} below).
To emphasise this, we will type $R[f]$ as a map $A \times B' \to A'$ (even though in reality $A = A'$ and $B=B'$), thus meaning that $\diset{f}{R[f]}$ is actually a lens from $\diset{A}{A'}$ to $\diset{B}{B'}$.
This typing distinction will be helpful later on, when we want to add additional components to our learning algorithms. 






The following two examples of CRDCs will serve as the basis for the learning scenarios of the upcoming sections.

\begin{exa}\label{ex:smoothrdc}
The category $\Smooth$ (Example \ref{ex:parasmooth}) is Cartesian with product given by addition, and it is also
 a Cartesian reverse differential category: given a smooth map $f: \R^n \to \R^m$, the map $R[f] \colon \R^n \times \R^m \to \R^n$ sends a pair $(x,v)$ to $J[f]^{T}(x)\cdot v$: the transpose of the Jacobian of
$f$ at $x$ in the direction $v$. For example, if $f: \R^2 \to \R^3$ is defined as $f(x_1,x_2) := (x_1^3 + 2x_1x_2,x_2,\sin(x_1))$, then $R[f]: \R^2 \times \R^3 \to \R^2$ is given by
  $$(x,v) \mapsto
  \begin{bmatrix}
  3x_1^2 + 2x_2 & 0 & \cos(x_1) \\
  2x_1 & 1 & 0
  \end{bmatrix}
  \cdot
  \begin{bmatrix}
  v_1 \\
  v_2 \\
  v_3
  \end{bmatrix}
  .$$
Using the reverse derivative (as opposed to the forward derivative) 
is well-known to be much more computationally efficient for functions $f: \R^n \to \R^m$ when $m \ll n$ (for example, see \cite{griewank_walther}), as is the case in most supervised learning situations (where often $m = 1$).
\end{exa}

\begin{exa}\label{ex:polysmooth}
  Another CRDC is the symmetric monoidal category $\PolyZ$ \cite[Example 14]{CRDC} with objects the natural numbers and morphisms $f \colon A \to B$
  the $B$-tuples of polynomials $\Z_2[x_1 \ldots x_A]$.
  When presented by generators and relations these morphisms can be viewed as a
  syntax for boolean circuits, with parametric lenses for such circuits (and their reverse derivative) described in
  \cite{rda}.  
\end{exa}

\begin{rem}\label{rmk:coalgebra}
  The definition of a CRDC (see Def.\ \ref{def:crdc}) satisfies 7 axioms describing the interaction of the reverse differential operator with the rest of the cartesian left-additive structure.
  As pointed out in \cite[Sec. 2.3]{cockett_differential_2014}, the last two axioms are independent of the others.
  In this paper we will additionally see that these two axioms are also not needed to compositionally model the update step of supervised learning.
\end{rem}

The remark above can be stated abstractly in two steps.
Firstly, we note that a CRDC $\Ca$ for each morphism $f$ defines a lens $\diset{f}{R[f]}$ whose backwards map is additive in the second component (see Def.\ \ref{def:additive_second_variable} for the definition of additivity in the second component).
This defines a subcategory $\Lens_A(\Ca)$ of $\Lens(\Ca)$, and the following functor.
\begin{restatable}{theorem}{LensFunctorCLAC}
  \label{thm:lens_functor_clac}
  Lenses with backward passes additive in the second component form a functor
  \[
    \Lens_A : \CLACat \to \CLACat
  \]
\end{restatable}

\begin{proof}
  See Appendix.
  Definition \ref{def:clacat} contains the definition of the category $\CLACat$, Definition \ref{def:lensa} the definition of $\Lens_A(\Ca)$,
  Prop.\ \ref{prop:lens_cart_left_additive} the proof that $\Lens_A(\Ca)$ is a cartesian left-additive category and  Prop.\ \ref{prop:lens_functor_cart_left_add} the action of $\Lens_A$ on morphisms. 
\end{proof}

The second step is the observation that a coalgebra of this functor gives us a choice of a cartesian left-additive category $\Ca$ and a cartesian left-additive functor $\RC : \Ca \to \Lens_A(\Ca)$ such that a number of axioms are satisfied: precisely the first five axioms of a CRDC.

\begin{restatable}[{(compare \cite[Prop. 31]{CRDC})}]{prop}{FiveOutOfSeven}
  \label{prop:gcrdc}
  A coalgebra of the copointed $\Lens_A$ endofunctor gives rise to a cartesian left-additive category $\Ca$ equipped with a reverse differential combinator $R$ which satisfies the first five axioms of a cartesian reverse derivative category.
\end{restatable}

\begin{proof}
  A coalgebra of $\Lens_A$ consists of a category $\Ca$ and a cartesian left-additive functor $\RC : \Ca \to \Lens_A(\Ca)$ such that the following diagram commutes.
  \begin{equation}
    \label{eq:lens_projection}
    \begin{tikzcd}
        {\Lens_A(\Ca)} & \Ca \\
        \Ca
        \arrow["\RC", from=2-1, to=1-1]
        \arrow["{\epsilon_{\Ca}}", from=1-1, to=1-2]
        \arrow[Rightarrow, no head, from=2-1, to=1-2]
      \end{tikzcd}
  \end{equation}
  The data of the functor $\RC$ is equivalent to the data of a reverse differential combinator $R$: every map $f : A \to B$ in $\Ca$ is mapped it to a lens whose forward part is by the Eq.\ \ref{eq:lens_projection} above restricted to be $f$ itself, leaving the only choice involved in this functor the one of the backward part.
  What remains to prove is that this backward part satisfies the first five axioms of a CRDC.
  We do this in Appendix~\ref{app:CRDC}.
\end{proof}

We will see in the next section how only the data of the functor $\RC : \Ca \to \Lens_A(\Ca)$ is used to model supervised learning, justifying our claim that only the first five axioms of a CRDC are used.


\section{Components of learning as Parametric Lenses}
\label{section:components-as-lenses}

As seen in the introduction, in the learning process there are many components at work: a model, an optimiser, a loss map, a learning rate, etc. In this section we show how each such component can be understood as a parametric lens.  Moreover, for each component, we show how our framework encompasses several variations of the gradient-descent algorithms, thus offering a unifying perspective on many different approaches that appear in the literature.

\subsection{Models as Parametric Lenses}\label{subsec:models}

We begin by characterising the models used for training as parametric lenses. In essence, our approach identifies a set of abstract requirements necessary to perform training by gradient descent, which covers the case studies that we will consider in the next sections. 

The leading intuition is that a suitable model is a parametric map, equipped with a reverse derivative operator. Using the formal developments of Section~\ref{section:background}, this amounts to assuming that a model is a morphism in $\Para(\Ca)$, for a CRDC $\Ca$. In order to visualise such morphism as a parametric lens, it then suffices to apply under $\Para(-)$ the canonical morphism $\RC \colon \Ca \to \Lens(\Ca)$ (which exists for any CRDC $\Ca$, see Prop.\ \ref{prop:gcrdc})\footnote{Here we are treating $\RC$ as postcomposed with the inclusion $\Lens_A(\Ca) \hookrightarrow \Lens(\Ca)$.}, mapping $f$ to $\diset{f}{R[f]}$. This yields a functor $ \Para(\RC): \Para(\Ca) \to \Para(\Lens(\Ca))$, pictorially defined as



\begin{equation}\label{eq:para_rdc}
\scalebox{0.7}{\begin{tikzpicture}
	\begin{pgfonlayer}{nodelayer}
		\node [style={medium_box}] (0) at (0, 0) {$f$};
		\node [style=none] (1) at (0, 1.5) {};
		\node [style=none] (2) at (-2, 0) {};
		\node [style=none] (3) at (0, 0.75) {};
		\node [style=none] (4) at (-0.5, 0) {};
		\node [style=none] (5) at (0.25, 0) {};
		\node [style=none] (6) at (2, 0) {};
		\node [style=none] (7) at (0, 2) {$P$};
		\node [style=none] (8) at (-3, 0) {$A$};
		\node [style=none] (9) at (3, 0) {$B$};
	\end{pgfonlayer}
	\begin{pgfonlayer}{edgelayer}
		\draw [style=pointy, in=-180, out=0, looseness=0.75] (2.center) to (4.center);
		\draw [style=pointy] (5.center) to (6.center);
		\draw [style=pointy] (1.center) to (3.center);
	\end{pgfonlayer}
\end{tikzpicture}}
 \qquad \mapsto \qquad \scalebox{0.6}{\begin{tikzpicture}
	\begin{pgfonlayer}{nodelayer}
		\node [style=none] (20) at (-8, 5) {};
		\node [style=none] (21) at (-8, -4) {};
		\node [style=none] (22) at (8, 5) {};
		\node [style=none] (23) at (8, -4) {};
		\node [style=none] (26) at (-10, 1) {$A$};
		\node [style=none] (27) at (-10, -2) {$A'$};
		\node [style=none] (28) at (10, 2.5) {$B$};
		\node [style=none] (29) at (10, -2.5) {$B'$};
		\node [style=none] (30) at (3, 3.5) {};
		\node [style=none] (31) at (3, 1.5) {};
		\node [style=none] (32) at (5, 1.5) {};
		\node [style=none] (33) at (5, 3.5) {};
		\node [style=none] (34) at (5, 2.5) {};
		\node [style=none] (35) at (3, 2) {};
		\node [style=none] (39) at (-9, 1) {};
		\node [style=none] (41) at (-3, 2) {};
		\node [style=none] (43) at (-8, -2.5) {};
		\node [style=none] (44) at (-9, -2.5) {};
		\node [style=none] (45) at (8, 2.5) {};
		\node [style=none] (46) at (9, 2.5) {};
		\node [style=none] (47) at (6.25, 0) {};
		\node [style=none] (48) at (6, -2) {};
		\node [style=none] (49) at (3, -1) {};
		\node [style=none] (50) at (3, -3) {};
		\node [style=none] (51) at (5, -3) {};
		\node [style=none] (52) at (5, -1) {};
		\node [style=none] (53) at (5, -2) {};
		\node [style=none] (54) at (3, -2.5) {};
		\node [style=none] (55) at (6, -2.5) {};
		\node [style=none] (56) at (5, -2.5) {};
		\node [style=none] (57) at (9, -2.5) {};
		\node [style=none] (58) at (8, -2.5) {};
		\node [style=none] (59) at (4, 2.5) {$f$};
		\node [style=none] (60) at (4, -2) {$R[f]$};
		\node [style=none] (61) at (-5, 2.75) {};
		\node [style=none] (66) at (3, 3) {};
		\node [style=none] (68) at (6.25, 0.75) {};
		\node [style=none] (69) at (6, -1.5) {};
		\node [style=none] (70) at (5, -1.5) {};
		\node [style=none] (71) at (-5, 5) {};
		\node [style=none] (72) at (3, -1.75) {};
		\node [style=none] (73) at (2, -1.75) {};
		\node [style=none] (74) at (1, -1) {};
		\node [style=none] (75) at (1, 5) {};
		\node [style=none] (76) at (-5, 6) {};
		\node [style=none] (77) at (1, 6) {};
		\node [style=none] (78) at (-5, 7) {$P$};
		\node [style=none] (79) at (1, 7) {$P'$};
		\node [style=copy2] (83) at (-4, 1.75) {};
		\node [style=none] (85) at (-3, 3) {};
		\node [style=none] (86) at (-3, 0.75) {};
		\node [style=copy2] (90) at (-4, 1) {};
		\node [style=none] (91) at (-8, 1) {};
		\node [style=none] (93) at (-3, 0) {};
	\end{pgfonlayer}
	\begin{pgfonlayer}{edgelayer}
		\draw (20.center) to (22.center);
		\draw (22.center) to (23.center);
		\draw (23.center) to (21.center);
		\draw (21.center) to (20.center);
		\draw [style=pointy] (43.center) to (44.center);
		\draw [style=pointy] (41.center) to (35.center);
		\draw (34.center) to (45.center);
		\draw [style=pointy] (45.center) to (46.center);
		\draw (30.center) to (33.center);
		\draw (33.center) to (34.center);
		\draw (34.center) to (32.center);
		\draw (32.center) to (31.center);
		\draw (31.center) to (35.center);
		\draw (35.center) to (30.center);
		\draw [style=pointy, bend left=90, looseness=1.75] (47.center) to (48.center);
		\draw (49.center) to (52.center);
		\draw (52.center) to (53.center);
		\draw (53.center) to (51.center);
		\draw (51.center) to (50.center);
		\draw (50.center) to (54.center);
		\draw (54.center) to (49.center);
		\draw [style=pointy] (48.center) to (53.center);
		\draw [style=pointy] (55.center) to (56.center);
		\draw (54.center) to (43.center);
		\draw (57.center) to (58.center);
		\draw (58.center) to (55.center);
		\draw [style=pointy, bend left=90, looseness=1.50] (68.center) to (69.center);
		\draw [style=pointy] (69.center) to (70.center);
		\draw (71.center) to (61.center);
		\draw [style=pointy] (72.center) to (73.center);
		\draw [in=-90, out=-180] (73.center) to (74.center);
		\draw (74.center) to (75.center);
		\draw [style=pointy] (75.center) to (77.center);
		\draw [style=pointy] (76.center) to (71.center);
		\draw [style=pointy, bend left] (83) to (85.center);
		\draw [style=pointy, bend right] (83) to (86.center);
		\draw [style=Black arrow, bend right=45] (61.center) to (83);
		\draw [style=pointy] (85.center) to (66.center);
		\draw [style=pointy, bend right, looseness=1.25] (90) to (93.center);
		\draw (91.center) to (90);
		\draw [style=pointy, bend left] (90) to (41.center);
		\draw [style=pointy] (86.center) to (68.center);
		\draw (39.center) to (91.center);
		\draw [style=pointy] (93.center) to (47.center);
	\end{pgfonlayer}
\end{tikzpicture}}
\end{equation}

\begin{exa}[Neural networks] \label{ex:smooth-lens}
As noted previously, to learn a function of type $\Rb^n \to \Rb^m$, one constructs a neural network, which can be seen as a function of type $\Rb^p \times \Rb^n \to \Rb^m$ where $\Rb^p$ is the space of parameters of the neural network.  As seen in Example~\ref{ex:parasmooth}, this is a map in the category $\Para(\Smooth)$ of type $\Rb^n \to \Rb^m$ with parameter space $\Rb^p$.  Then one can apply the functor in \eqref{eq:para_rdc} to present a neural network together with its reverse derivative operator as a parametric lens, i.e. a morphism in $\Para(\Lens(\Smooth))$.
\end{exa}

\begin{exa}[Boolean and Polynomial circuits]
For learning of Boolean circuits as described in \cite{rda}, the recipe is the same as in Example~\ref{ex:smooth-lens}, except that the base category is $\PolyZ$ (see Example~\ref{ex:polysmooth}). The important observation here is that $\PolyZ$ is a CRDC, see~\cite{CRDC,rda}, and thus we can apply the functor in \eqref{eq:para_rdc}. Note this setting can be generalised to circuits over any polynomial ring, see~\cite{WilsonZ22}. 
\end{exa}

Note a model/parametric lens $f$ takes as inputs an element of $A$, a parameter $P$, an element of $B'$ (a change in $B$) and outputs an element of $B$, a change in $A$, and a change in $P$.  This is not yet sufficient to do machine learning!  When we perform  learning, we want to input a parameter $P$ and a pair $A \times B$ and receive a new parameter $P$.  Instead, $f$ expects a change in $B$ (not an element of $B$) and outputs a change in $P$ (not an element of $P$).  Deep dreaming, on the other hand, wants to return an element of $A$ (not a change in $A$).  Thus, to do machine learning (or deep dreaming) we need to add additional components to $f$; we will consider these additional components in the next sections.

We now proceed to describe the internals of a model, and translate several examples of models in deep learning to their categorical form.
But before doing so, we clarify some terminology.
While `layers' and `models' are both parametric maps, the former typically
refers to \emph{components} of larger models, while the latter refers to the
final model to be learned in the manner described in Section
\ref{subsec:learning-parameters}).

\begin{rem}
  \label{remark:hidden-layer}
  An unfortunate ambiguity in deep learning terminology is the second meaning of `layer'.
  For example, the `hidden layer' of a model refers to internal \emph{values} of
  a neural network, corresponding to the `wires' of a string diagram.
  We will avoid using the term `layer' in this sense unless explicitly noted.
\end{rem}

In deep learning, one often speaks of `model architectures'.
This can mean either a specific model (e.g., ResNet50~\cite{he2015deep}) or a
family of models employing a particular technique.
For example, one says a model has a `convolutional architecture' when it
contains a convolutional layer (Example \ref{example:convolutional-layer}).
Examples of layers, models, and architectures are given in the following sections.

\subsubsection{Layers}

The \emph{dense} or \emph{fully-connected} layer is a component of many neural
network architectures.
In the categorical viewpoint, a dense layer is the composition of linear,
bias, and activation layers, which we describe first.
Unless explicitly noted, we will assume for simplicity that most layers are maps
in $\Para(\Smooth)$.
However, many of the maps defined here only require a \emph{multiplication} map
as additional structure, and so can be defined in any cartesian distributive
category~\cite{polycirc} such as $\POLY_S$.

\begin{exa}[Linear layer]
  A \textbf{linear} or \textbf{matrix-multiply} layer is the parametric map
  $(\Rb^{n m}, \mathsf{linear} : \Rb^{nm} \times \Rb^n \to \Rb^m)$,
  where $\mathsf{linear}(M, x) = M \cdot x$
  is the matrix-vector product of $M$ interpreted as matrix coefficients and the
  vector $x$.
\end{exa}

Note that layers are best thought of as \emph{families} of maps.
For example, there is a linear layer morphism for all choices of dimension $m$
and $n$.

\begin{exa}[Bias layer]
  A bias layer is a parametric map $(\Rb^n, +)$, where $+ : \Rb^n \times \Rb^n \to \Rb^n$ is
  the cartesian left-additive addition map.
  The reverse derivative of $+$ is the copy map, so we may depict the bias layer
  as a parametric lens as below.
  \[ \begin{tikzpicture}
	\begin{pgfonlayer}{nodelayer}
		\node [style=blackcircle] (0) at (-0.5, 0.5) {};
		\node [style=none] (1) at (-2, 0.5) {};
		\node [style=none] (2) at (-0.5, 2) {};
		\node [style=none] (3) at (-2.75, 0.5) {$A$};
		\node [style=none] (4) at (-0.5, 2.75) {$A$};
		\node [style=none] (5) at (2, 0.5) {};
		\node [style=none] (6) at (2.75, 0.5) {$A$};
		\node [style=none] (7) at (-1.25, 1.25) {};
		\node [style=none] (8) at (1.25, 1.25) {};
		\node [style=none] (9) at (1.25, -1.25) {};
		\node [style=none] (10) at (-1.25, -1.25) {};
		\node [style=blackcircle] (11) at (0.5, -0.5) {};
		\node [style=none] (12) at (-2, -0.5) {};
		\node [style=none] (13) at (0.5, 2) {};
		\node [style=none] (14) at (-2.75, -0.5) {$A$};
		\node [style=none] (15) at (0.5, 2.75) {$A$};
		\node [style=none] (16) at (2, -0.5) {};
		\node [style=none] (17) at (2.75, -0.5) {$A$};
	\end{pgfonlayer}
	\begin{pgfonlayer}{edgelayer}
		\draw [style=Black arrow] (1.center) to (0);
		\draw [style=Black arrow] (2.center) to (0);
		\draw [style=Black arrow] (0) to (5.center);
		\draw [style=Thin Grey Line] (7.center) to (8.center);
		\draw [style=Thin Grey Line] (8.center) to (9.center);
		\draw [style=Thin Grey Line] (9.center) to (10.center);
		\draw [style=Thin Grey Line] (10.center) to (7.center);
		\draw [style=Black arrow] (16.center) to (11);
		\draw [style=Black arrow] (11) to (12.center);
		\draw [style=Black arrow] (11) to (13.center);
	\end{pgfonlayer}
\end{tikzpicture} \]
\end{exa}

An \emph{activation layer} $(I, \alpha : A^n \to A^n)$ is a typically nonlinear,
trivially parametric map often applied to the output of another layer.
Many are simply the $n$-fold tensor product of a univariate function $A \to A$.

\begin{exa}[Sigmoid Activation]
  The \textbf{sigmoid} activation layer $(I, \mathsf{sigmoid} : \Rb^n \to \Rb^n)$
  is the $n$-fold tensor product
  $\sigma \times \overset{n}{\ldots} \times \sigma$
  of the \textbf{sigmoid function} $\sigma(x) = \frac{\exp(x)}{\exp(x) + 1}$.
\end{exa}

Note that unlike other layers considered so far, while \textbf{sigmoid} is a map
in $\Smooth$, it is not a map in $\POLY_S$.
An example of an activation layer which is \emph{not} in $\Smooth$ is the ReLU
map.

\begin{exa}[ReLU Activation]
  The `Rectified Linear Unit' activation function is the map
  $ \mathsf{ReLU}(x) = \delta_{>0}(x) \cdot x $
  where $\delta_{>0}$ is the positive indicator function.
  The $\mathsf{ReLU}$ activation \emph{layer} $\mathsf{ReLU} : A^n \to A^n$ is
  again the $n$-fold tensor product of this function.
  Although ReLU is not a smooth map, some presentations of RDCs can be extended
  via Theorem 3.1 of~\cite{polycirc} to include the positive indicator function
  $\delta_{>0} : A \to A$.
  The $\mathsf{ReLU}$ function can then be expressed in terms of this function,
  and its reverse derivative can be derived as
  $\mathsf{R}[\mathsf{ReLU}](x, \delta_y) = \delta_{>0}(x) \cdot \delta_y$.
\end{exa}

The combination of linear, bias, and choice of activation layer gives a \emph{dense} or
\emph{fully-connected} layer.

\begin{exa}[Dense Layer]
  \label{example:dense-layer}
  A \textbf{dense} or \textbf{fully-connected} layer is the following composition.
  \begin{equation}\label{eq:paralens-dense-layer}
     \scalebox{0.8}{\begin{tikzpicture}
	\begin{pgfonlayer}{nodelayer}
		\node [style=none] (20) at (-3, 1) {};
		\node [style=none] (21) at (-3, -1) {};
		\node [style=none] (22) at (-1, 1) {};
		\node [style=none] (23) at (-1, -1) {};
		\node [style=none] (26) at (-4.5, 0.5) {$\Rb^a$};
		\node [style=none] (27) at (-4.5, -0.5) {$\Rb^a$};
		\node [style=none] (28) at (0, 1) {$\Rb^b$};
		\node [style=none] (39) at (-4, 0.5) {};
		\node [style=none] (40) at (-3, 0.5) {};
		\node [style=none] (43) at (-3, -0.5) {};
		\node [style=none] (44) at (-4, -0.5) {};
		\node [style=none] (45) at (-1, 0.5) {};
		\node [style=none] (46) at (0, 0.5) {};
		\node [style=none] (57) at (0, -0.5) {};
		\node [style=none] (58) at (-1, -0.5) {};
		\node [style=none] (60) at (-2.5, 1) {};
		\node [style=none] (61) at (-1.5, 1) {};
		\node [style=none] (62) at (-2.5, 2) {};
		\node [style=none] (63) at (-1.5, 2) {};
		\node [style=none] (64) at (-2.75, 2.5) {$\Rb^{b \times a}$};
		\node [style=none] (65) at (-1, 2.5) {$\Rb^{b \times a}$};
		\node [style=none] (66) at (1, 1) {};
		\node [style=none] (67) at (1, -1) {};
		\node [style=none] (68) at (3, 1) {};
		\node [style=none] (69) at (3, -1) {};
		\node [style=none] (74) at (0, 0.5) {};
		\node [style=none] (75) at (1, 0.5) {};
		\node [style=none] (76) at (1, -0.5) {};
		\node [style=none] (77) at (0, -0.5) {};
		\node [style=none] (78) at (3, 0.5) {};
		\node [style=none] (79) at (4, 0.5) {};
		\node [style=none] (80) at (4, -0.5) {};
		\node [style=none] (81) at (3, -0.5) {};
		\node [style=none] (82) at (1.5, 1) {};
		\node [style=none] (83) at (2.5, 1) {};
		\node [style=none] (84) at (1.5, 2) {};
		\node [style=none] (85) at (2.5, 2) {};
		\node [style=none] (86) at (1.5, 2.5) {$\Rb^b$};
		\node [style=none] (87) at (2.75, 2.5) {$\Rb^b$};
		\node [style=none] (88) at (-2, 0) {\textsf{linear}};
		\node [style=none] (89) at (2, 0) {\textsf{bias}};
		\node [style=none] (92) at (4, 0.5) {};
		\node [style=none] (93) at (4, -0.5) {};
		\node [style=none] (94) at (5, 1) {};
		\node [style=none] (95) at (5, -1) {};
		\node [style=none] (96) at (8, 1) {};
		\node [style=none] (97) at (8, -1) {};
		\node [style=none] (100) at (4, 0.5) {};
		\node [style=none] (101) at (5, 0.5) {};
		\node [style=none] (102) at (5, -0.5) {};
		\node [style=none] (103) at (4, -0.5) {};
		\node [style=none] (104) at (8, 0.5) {};
		\node [style=none] (105) at (9, 0.5) {};
		\node [style=none] (106) at (9, -0.5) {};
		\node [style=none] (107) at (8, -0.5) {};
		\node [style=none] (108) at (5.5, 1) {};
		\node [style=none] (109) at (7.5, 1) {};
		\node [style=none] (114) at (6.5, 0) {\textsf{activation}};
		\node [style=none] (115) at (0, -1) {$\Rb^b$};
		\node [style=none] (116) at (4, 1) {$\Rb^b$};
		\node [style=none] (117) at (4, -1) {$\Rb^b$};
		\node [style=none] (118) at (9.75, 0.5) {$\Rb^b$};
		\node [style=none] (119) at (9.75, -0.5) {$\Rb^b$};
	\end{pgfonlayer}
	\begin{pgfonlayer}{edgelayer}
		\draw (20.center) to (22.center);
		\draw (22.center) to (23.center);
		\draw (23.center) to (21.center);
		\draw (21.center) to (20.center);
		\draw [style=pointy] (39.center) to (40.center);
		\draw [style=pointy] (43.center) to (44.center);
		\draw (45.center) to (46.center);
		\draw [style=pointy] (57.center) to (58.center);
		\draw [style=pointy] (62.center) to (60.center);
		\draw [style=pointy] (61.center) to (63.center);
		\draw (66.center) to (68.center);
		\draw (68.center) to (69.center);
		\draw (69.center) to (67.center);
		\draw (67.center) to (66.center);
		\draw [style=pointy] (74.center) to (75.center);
		\draw (76.center) to (77.center);
		\draw [style=pointy] (78.center) to (79.center);
		\draw [style=pointy] (80.center) to (81.center);
		\draw [style=pointy] (84.center) to (82.center);
		\draw [style=pointy] (83.center) to (85.center);
		\draw (94.center) to (96.center);
		\draw (96.center) to (97.center);
		\draw (97.center) to (95.center);
		\draw (95.center) to (94.center);
		\draw [style=pointy] (100.center) to (101.center);
		\draw (102.center) to (103.center);
		\draw [style=pointy] (104.center) to (105.center);
		\draw [style=pointy] (106.center) to (107.center);
	\end{pgfonlayer}
\end{tikzpicture}}
  \end{equation}
  Where some choice of \emph{input dimension} $m \in \Nb$, \emph{output dimension}
  $m \in \Nb$, and \emph{activation layer} $\alpha : \Rb^m \to \Rb^n$ is
  assumed.
  Note that the activation layer has no parameters.
\end{exa}

The final two examples of layers we cover here are \emph{convolutional} and
\emph{max-pooling} layers, which are common in models for image-processing
tasks.

\begin{exa}[Convolutional Layer]
  \label{example:convolutional-layer}
  A \textbf{convolutional} layer is a map
  $(\Rb^{k^2}, \mathsf{convolve2D} : \Rb^{k^2} \times \Rb^{m^2} \to \Rb^{n^2})$
  where $\mathsf{convolve2D}$ denotes the discrete 2D convolution of a
  $k \times k$ kernel and an $m \times m$ image.
  The output of a convolutional layer is an $n \times n$ image with
  $n = \max(m, k) - \min(m, k) + 1$.
\end{exa}

A number of further variations of the convolutional layer exist in the
literature, but the basic idea is to use 2D convolution to map a kernel (the
parameters) over the input.
Convolutional layers are frequently composed with max-pooling layers.

\begin{exa}[Max-Pooling Layer]
  \label{example:convolutional-layer}
  A \textbf{max-pooling} layer $(I, \mathsf{maxpool} : S^{(kn)^2} \to S^n)$ computes
  the maximum of each of the $n^2$ size-$k \times k$ subregions of the input image.
\end{exa}

However, as with the $\mathsf{ReLU}$ layer, max-pooling layers cannot be thought
of as maps in $\Smooth$.
Nevertheless, by again appealing to \cite[Theorem 3.1]{polycirc}, one can extend
a presentation of RDCs to include a function $\max : 2 \to 1$ from which the
max-pooling layer and its reverse derivative can be derived.

\subsubsection{Architectures}

We now consider some examples of neural network architectures defined in terms
of the layers above.
Since both layers and architectures are just parametric maps, we can consider
the layers by themselves as architectures already,
and in fact the $\emph{linear}$ and $\emph{dense}$ layers are sufficient to
solve some simple machine learning problems.
The first non-trivial architecture we consider is the `single hidden layer'
network.

\begin{exa}[Hidden Layer Network]
  A neural network with a single `hidden layer' (in the sense of Remark \ref{remark:hidden-layer})
  is the composition of two $\mathsf{dense}$ maps.
  \[ \scalebox{0.8}{\begin{tikzpicture}
	\begin{pgfonlayer}{nodelayer}
		\node [style=none] (63) at (-7.25, 1) {};
		\node [style=none] (64) at (-7.25, -1) {};
		\node [style=none] (65) at (-1.75, 1) {};
		\node [style=none] (66) at (-1.75, -1) {};
		\node [style=none] (67) at (-8.75, 0.5) {$\Rb^a$};
		\node [style=none] (68) at (-8.75, -0.5) {$\Rb^a$};
		\node [style=none] (69) at (0, 1) {$\Rb^b$};
		\node [style=none] (70) at (-8.25, 0.5) {};
		\node [style=none] (71) at (-7.25, 0.5) {};
		\node [style=none] (72) at (-7.25, -0.5) {};
		\node [style=none] (73) at (-8.25, -0.5) {};
		\node [style=none] (74) at (-1.75, 0.5) {};
		\node [style=none] (75) at (0, 0.5) {};
		\node [style=none] (76) at (0, -0.5) {};
		\node [style=none] (77) at (-1.75, -0.5) {};
		\node [style=none] (78) at (-6.25, 1) {};
		\node [style=none] (79) at (-2.75, 1) {};
		\node [style=none] (80) at (-6.25, 2) {};
		\node [style=none] (81) at (-2.75, 2) {};
		\node [style=none] (82) at (-6.5, 2.5) {$\Rb^{b \times a} \times \Rb^b$};
		\node [style=none] (83) at (-2.25, 2.5) {$\Rb^{b \times a} \times \Rb^b$};
		\node [style=none] (84) at (0, 0.5) {};
		\node [style=none] (85) at (0, -0.5) {};
		\node [style=none] (86) at (-4.5, 0) {\textsf{dense}};
		\node [style=none] (87) at (0, -1) {$\Rb^b$};
		\node [style=none] (88) at (1.75, 1) {};
		\node [style=none] (89) at (1.75, -1) {};
		\node [style=none] (90) at (7.25, 1) {};
		\node [style=none] (91) at (7.25, -1) {};
		\node [style=none] (94) at (9, 0.5) {$\Rb^c$};
		\node [style=none] (96) at (1.75, 0.5) {};
		\node [style=none] (97) at (1.75, -0.5) {};
		\node [style=none] (99) at (7.25, 0.5) {};
		\node [style=none] (100) at (8.25, 0.5) {};
		\node [style=none] (101) at (8.25, -0.5) {};
		\node [style=none] (102) at (7.25, -0.5) {};
		\node [style=none] (103) at (2.75, 1) {};
		\node [style=none] (104) at (6.25, 1) {};
		\node [style=none] (105) at (2.75, 2) {};
		\node [style=none] (106) at (6.25, 2) {};
		\node [style=none] (107) at (2.5, 2.5) {$\Rb^{c \times b} \times \Rb^c$};
		\node [style=none] (108) at (6.75, 2.5) {$\Rb^{c \times b} \times \Rb^c$};
		\node [style=none] (109) at (8.25, 0.5) {};
		\node [style=none] (110) at (8.25, -0.5) {};
		\node [style=none] (111) at (4.5, 0) {\textsf{dense}};
		\node [style=none] (112) at (9, -0.5) {$\Rb^c$};
	\end{pgfonlayer}
	\begin{pgfonlayer}{edgelayer}
		\draw (63.center) to (65.center);
		\draw (65.center) to (66.center);
		\draw (66.center) to (64.center);
		\draw (64.center) to (63.center);
		\draw [style=pointy] (70.center) to (71.center);
		\draw [style=pointy] (72.center) to (73.center);
		\draw [style=pointy] (80.center) to (78.center);
		\draw [style=pointy] (79.center) to (81.center);
		\draw (88.center) to (90.center);
		\draw (90.center) to (91.center);
		\draw (91.center) to (89.center);
		\draw (89.center) to (88.center);
		\draw [style=pointy] (101.center) to (102.center);
		\draw [style=pointy] (105.center) to (103.center);
		\draw [style=pointy] (104.center) to (106.center);
		\draw [style=Black arrow] (97.center) to (85.center);
		\draw (85.center) to (77.center);
		\draw [style=Black arrow] (74.center) to (84.center);
		\draw (84.center) to (96.center);
		\draw [style=Black arrow] (99.center) to (109.center);
	\end{pgfonlayer}
\end{tikzpicture}} \]
  We emphasize that the term `hidden layer' here ambiguously refers to the
  central \emph{wires} labeled $\Rb^b$ rather than the \textsf{dense} morphisms.
\end{exa}

This architecture is demonstrated in detail in the
experiments~\cite{python-library} accompanying this paper.
Also included in our experiments is a convolutional model for classifying images
of handwritten digits (the MNIST~\cite{mnist} dataset).
A simplified version is below.

\begin{exa}[Convolutional Architecture]
  First, define a $\mathsf{CPR}$ layer as the composition of a convolution, max-pooling, and $\mathsf{ReLU}$ layer.
  \[ CPR \: = \: \scalebox{0.8}{\begin{tikzpicture}
	\begin{pgfonlayer}{nodelayer}
		\node [style=none] (63) at (-2.75, 1) {};
		\node [style=none] (64) at (-2.75, -1) {};
		\node [style=none] (65) at (2.75, 1) {};
		\node [style=none] (66) at (2.75, -1) {};
		\node [style=none] (69) at (4.5, 1) {$\Rb^{\frac{n}{2} \times \frac{n}{2}} $};
		\node [style=none] (70) at (-4.5, 0.5) {};
		\node [style=none] (71) at (-2.75, 0.5) {};
		\node [style=none] (72) at (-2.75, -0.5) {};
		\node [style=none] (73) at (-4.5, -0.5) {};
		\node [style=none] (74) at (2.75, 0.5) {};
		\node [style=none] (75) at (4.5, 0.5) {};
		\node [style=none] (76) at (4.5, -0.5) {};
		\node [style=none] (77) at (2.75, -0.5) {};
		\node [style=none] (78) at (-1.75, 1) {};
		\node [style=none] (79) at (1.75, 1) {};
		\node [style=none] (84) at (4.5, 0.5) {};
		\node [style=none] (85) at (4.5, -0.5) {};
		\node [style=none] (86) at (0, 0) {\textsf{maxpool}};
		\node [style=none] (87) at (4.5, -1) {$\Rb^{\frac{n}{2} \times \frac{n}{2}} $};
		\node [style=none] (88) at (6.25, 1) {};
		\node [style=none] (89) at (6.25, -1) {};
		\node [style=none] (90) at (11.75, 1) {};
		\node [style=none] (91) at (11.75, -1) {};
		\node [style=none] (94) at (13.5, 0.5) {$\Rb^{\frac{n}{2} \times \frac{n}{2}} $};
		\node [style=none] (96) at (6.25, 0.5) {};
		\node [style=none] (97) at (6.25, -0.5) {};
		\node [style=none] (99) at (11.75, 0.5) {};
		\node [style=none] (100) at (12.75, 0.5) {};
		\node [style=none] (101) at (12.75, -0.5) {};
		\node [style=none] (102) at (11.75, -0.5) {};
		\node [style=none] (109) at (12.75, 0.5) {};
		\node [style=none] (110) at (12.75, -0.5) {};
		\node [style=none] (111) at (9, 0) {\textsf{ReLU}};
		\node [style=none] (112) at (13.5, -0.5) {$\Rb^{\frac{n}{2} \times \frac{n}{2}} $};
		\node [style=none] (113) at (-11.75, 1) {};
		\node [style=none] (114) at (-11.75, -1) {};
		\node [style=none] (115) at (-6.25, 1) {};
		\node [style=none] (116) at (-6.25, -1) {};
		\node [style=none] (117) at (-13.75, 0.5) {$\Rb^{m \times m}$};
		\node [style=none] (118) at (-13.75, -0.5) {$\Rb^{m \times m}$};
		\node [style=none] (119) at (-12.75, 0.5) {};
		\node [style=none] (120) at (-11.75, 0.5) {};
		\node [style=none] (121) at (-11.75, -0.5) {};
		\node [style=none] (122) at (-12.75, -0.5) {};
		\node [style=none] (123) at (-6.25, 0.5) {};
		\node [style=none] (124) at (-6.25, -0.5) {};
		\node [style=none] (125) at (-10.75, 1) {};
		\node [style=none] (126) at (-7.25, 1) {};
		\node [style=none] (127) at (-10.75, 2) {};
		\node [style=none] (128) at (-7.25, 2) {};
		\node [style=none] (129) at (-11, 2.5) {$\R^{k \times k}$};
		\node [style=none] (130) at (-6.75, 2.5) {$\R^{k \times k}$};
		\node [style=none] (131) at (-9, 0) {\textsf{convolve2D}};
		\node [style=none] (132) at (-4.5, 1) {$\Rb^{n \times n}$};
		\node [style=none] (133) at (-4.5, -1) {$\Rb^{n \times n}$};
	\end{pgfonlayer}
	\begin{pgfonlayer}{edgelayer}
		\draw (63.center) to (65.center);
		\draw (65.center) to (66.center);
		\draw (66.center) to (64.center);
		\draw (64.center) to (63.center);
		\draw (70.center) to (71.center);
		\draw [style=pointy] (72.center) to (73.center);
		\draw (88.center) to (90.center);
		\draw (90.center) to (91.center);
		\draw (91.center) to (89.center);
		\draw (89.center) to (88.center);
		\draw [style=pointy] (101.center) to (102.center);
		\draw [style=Black arrow] (97.center) to (85.center);
		\draw (85.center) to (77.center);
		\draw [style=Black arrow] (74.center) to (84.center);
		\draw (84.center) to (96.center);
		\draw [style=Black arrow] (99.center) to (109.center);
		\draw (113.center) to (115.center);
		\draw (115.center) to (116.center);
		\draw (116.center) to (114.center);
		\draw (114.center) to (113.center);
		\draw [style=pointy] (119.center) to (120.center);
		\draw [style=pointy] (121.center) to (122.center);
		\draw [style=pointy] (127.center) to (125.center);
		\draw [style=pointy] (126.center) to (128.center);
		\draw [style=Black arrow] (123.center) to (70.center);
		\draw (73.center) to (124.center);
	\end{pgfonlayer}
\end{tikzpicture}} \]
  where $n = \max(m, k) - \min(m, k) + 1$.
  One can then define a convolutional architecture for
  classifying $28 \times 28$-pixel images of the MNIST dataset into digits $0-9$
  as follows.
  \[ \scalebox{0.8}{\begin{tikzpicture}
	\begin{pgfonlayer}{nodelayer}
		\node [style=none] (63) at (-7.25, 1) {};
		\node [style=none] (64) at (-7.25, -1) {};
		\node [style=none] (65) at (-1.75, 1) {};
		\node [style=none] (66) at (-1.75, -1) {};
		\node [style=none] (67) at (-9, 1) {$\Rb^{13 \times 13}$};
		\node [style=none] (68) at (-9, -1) {$\Rb^{13 \times 13}$};
		\node [style=none] (69) at (0, 1) {$\Rb^{5 \times 5}$};
		\node [style=none] (70) at (-9, 0.5) {};
		\node [style=none] (71) at (-7.25, 0.5) {};
		\node [style=none] (72) at (-7.25, -0.5) {};
		\node [style=none] (73) at (-9, -0.5) {};
		\node [style=none] (74) at (-1.75, 0.5) {};
		\node [style=none] (75) at (0, 0.5) {};
		\node [style=none] (76) at (0, -0.5) {};
		\node [style=none] (77) at (-1.75, -0.5) {};
		\node [style=none] (78) at (-6.25, 1) {};
		\node [style=none] (79) at (-2.75, 1) {};
		\node [style=none] (80) at (-6.25, 2) {};
		\node [style=none] (81) at (-2.75, 2) {};
		\node [style=none] (82) at (-6.5, 2.5) {$\R^{4 \times 4}$};
		\node [style=none] (83) at (-2.25, 2.5) {$\R^{4 \times 4}$};
		\node [style=none] (84) at (0, 0.5) {};
		\node [style=none] (85) at (0, -0.5) {};
		\node [style=none] (86) at (-4.5, 0) {\textsf{CPR}};
		\node [style=none] (87) at (0, -1) {$\Rb^{5 \times 5}$};
		\node [style=none] (88) at (1.75, 1) {};
		\node [style=none] (89) at (1.75, -1) {};
		\node [style=none] (90) at (7.25, 1) {};
		\node [style=none] (91) at (7.25, -1) {};
		\node [style=none] (94) at (9, 0.5) {$\Rb^{10}$};
		\node [style=none] (96) at (1.75, 0.5) {};
		\node [style=none] (97) at (1.75, -0.5) {};
		\node [style=none] (99) at (7.25, 0.5) {};
		\node [style=none] (100) at (8.25, 0.5) {};
		\node [style=none] (101) at (8.25, -0.5) {};
		\node [style=none] (102) at (7.25, -0.5) {};
		\node [style=none] (103) at (2.75, 1) {};
		\node [style=none] (104) at (6.25, 1) {};
		\node [style=none] (105) at (2.75, 2) {};
		\node [style=none] (106) at (6.25, 2) {};
		\node [style=none] (107) at (2.5, 2.5) {$\Rb^{10 \times 5} \times \Rb^{10}$};
		\node [style=none] (108) at (6.75, 2.5) {$\Rb^{10 \times 5} \times \Rb^{10}$};
		\node [style=none] (109) at (8.25, 0.5) {};
		\node [style=none] (110) at (8.25, -0.5) {};
		\node [style=none] (111) at (4.5, 0) {\textsf{dense}};
		\node [style=none] (112) at (9, -0.5) {$\Rb^{10}$};
		\node [style=none] (113) at (-16.25, 1) {};
		\node [style=none] (114) at (-16.25, -1) {};
		\node [style=none] (115) at (-10.75, 1) {};
		\node [style=none] (116) at (-10.75, -1) {};
		\node [style=none] (117) at (-18.25, 0.5) {$\Rb^{28 \times 28}$};
		\node [style=none] (118) at (-18.25, -0.5) {$\Rb^{28 \times 28}$};
		\node [style=none] (119) at (-17.25, 0.5) {};
		\node [style=none] (120) at (-16.25, 0.5) {};
		\node [style=none] (121) at (-16.25, -0.5) {};
		\node [style=none] (122) at (-17.25, -0.5) {};
		\node [style=none] (123) at (-10.75, 0.5) {};
		\node [style=none] (124) at (-10.75, -0.5) {};
		\node [style=none] (125) at (-15.25, 1) {};
		\node [style=none] (126) at (-11.75, 1) {};
		\node [style=none] (127) at (-15.25, 2) {};
		\node [style=none] (128) at (-11.75, 2) {};
		\node [style=none] (129) at (-15.5, 2.5) {$\R^{3\times3}$};
		\node [style=none] (130) at (-11.25, 2.5) {$\R^{3\times3}$};
		\node [style=none] (131) at (-13.5, 0) {\textsf{CPR}};
	\end{pgfonlayer}
	\begin{pgfonlayer}{edgelayer}
		\draw (63.center) to (65.center);
		\draw (65.center) to (66.center);
		\draw (66.center) to (64.center);
		\draw (64.center) to (63.center);
		\draw (70.center) to (71.center);
		\draw [style=pointy] (72.center) to (73.center);
		\draw [style=pointy] (80.center) to (78.center);
		\draw [style=pointy] (79.center) to (81.center);
		\draw (88.center) to (90.center);
		\draw (90.center) to (91.center);
		\draw (91.center) to (89.center);
		\draw (89.center) to (88.center);
		\draw [style=pointy] (101.center) to (102.center);
		\draw [style=pointy] (105.center) to (103.center);
		\draw [style=pointy] (104.center) to (106.center);
		\draw [style=Black arrow] (97.center) to (85.center);
		\draw (85.center) to (77.center);
		\draw [style=Black arrow] (74.center) to (84.center);
		\draw (84.center) to (96.center);
		\draw [style=Black arrow] (99.center) to (109.center);
		\draw (113.center) to (115.center);
		\draw (115.center) to (116.center);
		\draw (116.center) to (114.center);
		\draw (114.center) to (113.center);
		\draw [style=pointy] (119.center) to (120.center);
		\draw [style=pointy] (121.center) to (122.center);
		\draw [style=pointy] (127.center) to (125.center);
		\draw [style=pointy] (126.center) to (128.center);
		\draw [style=Black arrow] (123.center) to (70.center);
		\draw (73.center) to (124.center);
	\end{pgfonlayer}
\end{tikzpicture}} \]
\end{exa}

Lastly, we mention the architecture of Generative Adversarial Networks which e define and thoroughly discuss in Section \ref{subsec:unsupervised-learning}.

\subsubsection{Weight Tying and Batching}

A number of general techniques are employed in designing deep learning models.
We now describe two examples of these techniques in terms of their categorical interpretations.
The first is \emph{weight-tying}, which is required in Subsection
\ref{subsec:unsupervised-learning} to define a more complex architecture for
unsupervised learning: the GAN.

\begin{exa}[Weight Tying]
  Weight tying is the sharing of parameters between two different components.
  Categorically speaking, this means using the copy map on parameters as below.
  \[ \scalebox{0.8}{\begin{tikzpicture}
	\begin{pgfonlayer}{nodelayer}
		\node [style=none] (0) at (-1, 1) {};
		\node [style=none] (1) at (-1, -1) {};
		\node [style=none] (2) at (1, 1) {};
		\node [style=none] (3) at (1, -1) {};
		\node [style=none] (4) at (-2.5, 0.5) {$\Rb^a$};
		\node [style=none] (5) at (-2.5, -0.5) {$\Rb^a$};
		\node [style=none] (7) at (-2, 0.5) {};
		\node [style=none] (8) at (-1, 0.5) {};
		\node [style=none] (9) at (-1, -0.5) {};
		\node [style=none] (10) at (-2, -0.5) {};
		\node [style=none] (11) at (1, 0.5) {};
		\node [style=none] (12) at (2, 0.5) {};
		\node [style=none] (13) at (2, -0.5) {};
		\node [style=none] (14) at (1, -0.5) {};
		\node [style=none] (15) at (-0.5, 1) {};
		\node [style=none] (16) at (0.5, 1) {};
		\node [style=none] (17) at (-0.5, 2) {};
		\node [style=none] (18) at (0.5, 2) {};
		\node [style=none] (19) at (1, 5.75) {$\Rb^{b \times a}$};
		\node [style=none] (20) at (2.75, 5.75) {$\Rb^{b \times a}$};
		\node [style=none] (21) at (2, 0.5) {};
		\node [style=none] (22) at (6.5, 0.5) {};
		\node [style=none] (23) at (6.5, -0.5) {};
		\node [style=none] (24) at (2, -0.5) {};
		\node [style=none] (25) at (0, 0) {$f$};
		\node [style=none] (27) at (2.5, -2) {};
		\node [style=none] (28) at (2.5, -4) {};
		\node [style=none] (29) at (4.5, -2) {};
		\node [style=none] (30) at (4.5, -4) {};
		\node [style=none] (34) at (-2, -2.5) {};
		\node [style=none] (35) at (2.5, -2.5) {};
		\node [style=none] (36) at (2.5, -3.5) {};
		\node [style=none] (37) at (-2, -3.5) {};
		\node [style=none] (38) at (4.5, -2.5) {};
		\node [style=none] (39) at (5.5, -2.5) {};
		\node [style=none] (40) at (5.5, -3.5) {};
		\node [style=none] (41) at (4.5, -3.5) {};
		\node [style=none] (42) at (3, -2) {};
		\node [style=none] (43) at (4, -2) {};
		\node [style=none] (44) at (3, 2) {};
		\node [style=none] (45) at (4, 2) {};
		\node [style=none] (48) at (5.5, -2.5) {};
		\node [style=none] (49) at (6.5, -2.5) {};
		\node [style=none] (50) at (6.5, -3.5) {};
		\node [style=none] (51) at (5.5, -3.5) {};
		\node [style=none] (52) at (3.5, -3) {$g$};
		\node [style=blackcircle] (54) at (1, 4) {};
		\node [style=blackcircle] (55) at (3, 4) {};
		\node [style=none] (56) at (1, 5) {};
		\node [style=none] (57) at (3, 5) {};
		\node [style=none] (58) at (-2.5, -2.5) {$\R^c$};
		\node [style=none] (59) at (-2.5, -3.5) {$\R^c$};
		\node [style=none] (60) at (7.25, -2.5) {$\R^d$};
		\node [style=none] (61) at (7.25, -3.5) {$\R^d$};
		\node [style=none] (62) at (7.25, 0.5) {$\Rb^b$};
		\node [style=none] (63) at (7.25, -0.5) {$\Rb^b$};
	\end{pgfonlayer}
	\begin{pgfonlayer}{edgelayer}
		\draw (0.center) to (2.center);
		\draw (2.center) to (3.center);
		\draw (3.center) to (1.center);
		\draw (1.center) to (0.center);
		\draw [style=pointy] (7.center) to (8.center);
		\draw [style=pointy] (9.center) to (10.center);
		\draw (11.center) to (12.center);
		\draw [style=pointy] (13.center) to (14.center);
		\draw [style=pointy] (17.center) to (15.center);
		\draw [style=pointy] (16.center) to (18.center);
		\draw [style=pointy] (21.center) to (22.center);
		\draw (23.center) to (24.center);
		\draw (27.center) to (29.center);
		\draw (29.center) to (30.center);
		\draw (30.center) to (28.center);
		\draw (28.center) to (27.center);
		\draw [style=pointy] (34.center) to (35.center);
		\draw [style=pointy] (36.center) to (37.center);
		\draw (38.center) to (39.center);
		\draw [style=pointy] (40.center) to (41.center);
		\draw [style=pointy] (44.center) to (42.center);
		\draw [style=pointy] (43.center) to (45.center);
		\draw [style=pointy] (48.center) to (49.center);
		\draw (50.center) to (51.center);
		\draw [style=Black arrow, in=90, out=-120] (54) to (17.center);
		\draw [style=Black arrow, in=90, out=-45] (54) to (44.center);
		\draw [style=Black arrow, in=225, out=90, looseness=0.75] (18.center) to (55);
		\draw [style=Black arrow, in=300, out=90] (45.center) to (55);
		\draw [style=Black arrow] (56.center) to (54);
		\draw [style=Black arrow] (55) to (57.center);
	\end{pgfonlayer}
\end{tikzpicture}} \]
  Note that $f$ and $g$ have the same parameters, but might be applied to
  different parts of the input.
  In this sense, one can think of \emph{convolutional} layers (Example
  \ref{example:convolutional-layer}) as using weight-tying.
\end{exa}

A related technique is \emph{batching}.
So far, we have considered learning as updating a model with a single data
example $(x, y)$ at each timestep.
However, it is common for efficiency purposes to update the model using a
\emph{batch} of examples:
a finite number $n$ of examples $(x_i, y_i)$ for $i \in \{ 0 \ldots n \}$.

\begin{exa}[Batching]
  Suppose we have a model $(P, f : P \times A \to B)$.
  The batched model with \textbf{batch size} $n$ is a parametric map
  $(P, f' : P \times A^n : B^n)$ where $f'$ consists of $n$ copies of $f$
  applied to each input, but with the same parameters.
  For example, when $n = 2$, the batch update is as follows.
  \[ \scalebox{0.8}{\begin{tikzpicture}
	\begin{pgfonlayer}{nodelayer}
		\node [style=none] (0) at (-1, 1) {};
		\node [style=none] (1) at (-1, -1) {};
		\node [style=none] (2) at (1, 1) {};
		\node [style=none] (3) at (1, -1) {};
		\node [style=none] (4) at (-2.5, 0.5) {$A$};
		\node [style=none] (5) at (-2.5, -0.5) {$A'$};
		\node [style=none] (7) at (-2, 0.5) {};
		\node [style=none] (8) at (-1, 0.5) {};
		\node [style=none] (9) at (-1, -0.5) {};
		\node [style=none] (10) at (-2, -0.5) {};
		\node [style=none] (11) at (1, 0.5) {};
		\node [style=none] (12) at (2, 0.5) {};
		\node [style=none] (13) at (2, -0.5) {};
		\node [style=none] (14) at (1, -0.5) {};
		\node [style=none] (15) at (-0.5, 1) {};
		\node [style=none] (16) at (0.5, 1) {};
		\node [style=none] (17) at (-0.5, 2) {};
		\node [style=none] (18) at (0.5, 2) {};
		\node [style=none] (19) at (1, 5.75) {$P$};
		\node [style=none] (20) at (3, 5.75) {$P'$};
		\node [style=none] (21) at (2, 0.5) {};
		\node [style=none] (22) at (6.5, 0.5) {};
		\node [style=none] (23) at (6.5, -0.5) {};
		\node [style=none] (24) at (2, -0.5) {};
		\node [style=none] (25) at (0, 0) {$f$};
		\node [style=none] (27) at (2.5, -2) {};
		\node [style=none] (28) at (2.5, -4) {};
		\node [style=none] (29) at (4.5, -2) {};
		\node [style=none] (30) at (4.5, -4) {};
		\node [style=none] (34) at (-2, -2.5) {};
		\node [style=none] (35) at (2.5, -2.5) {};
		\node [style=none] (36) at (2.5, -3.5) {};
		\node [style=none] (37) at (-2, -3.5) {};
		\node [style=none] (38) at (4.5, -2.5) {};
		\node [style=none] (39) at (5.5, -2.5) {};
		\node [style=none] (40) at (5.5, -3.5) {};
		\node [style=none] (41) at (4.5, -3.5) {};
		\node [style=none] (42) at (3, -2) {};
		\node [style=none] (43) at (4, -2) {};
		\node [style=none] (44) at (3, 2) {};
		\node [style=none] (45) at (4, 2) {};
		\node [style=none] (48) at (5.5, -2.5) {};
		\node [style=none] (49) at (6.5, -2.5) {};
		\node [style=none] (50) at (6.5, -3.5) {};
		\node [style=none] (51) at (5.5, -3.5) {};
		\node [style=none] (52) at (3.5, -3) {$f$};
		\node [style=blackcircle] (54) at (1, 4) {};
		\node [style=blackcircle] (55) at (3, 4) {};
		\node [style=none] (56) at (1, 5) {};
		\node [style=none] (57) at (3, 5) {};
		\node [style=none] (62) at (7.25, 0.5) {$B$};
		\node [style=none] (63) at (7.25, -0.5) {$B'$};
		\node [style=none] (64) at (-2.5, -2.5) {$A$};
		\node [style=none] (65) at (-2.5, -3.5) {$A'$};
		\node [style=none] (66) at (7.25, -2.5) {$B$};
		\node [style=none] (67) at (7.25, -3.5) {$B'$};
	\end{pgfonlayer}
	\begin{pgfonlayer}{edgelayer}
		\draw (0.center) to (2.center);
		\draw (2.center) to (3.center);
		\draw (3.center) to (1.center);
		\draw (1.center) to (0.center);
		\draw [style=pointy] (7.center) to (8.center);
		\draw [style=pointy] (9.center) to (10.center);
		\draw (11.center) to (12.center);
		\draw [style=pointy] (13.center) to (14.center);
		\draw [style=pointy] (17.center) to (15.center);
		\draw [style=pointy] (16.center) to (18.center);
		\draw [style=pointy] (21.center) to (22.center);
		\draw (23.center) to (24.center);
		\draw (27.center) to (29.center);
		\draw (29.center) to (30.center);
		\draw (30.center) to (28.center);
		\draw (28.center) to (27.center);
		\draw [style=pointy] (34.center) to (35.center);
		\draw [style=pointy] (36.center) to (37.center);
		\draw (38.center) to (39.center);
		\draw [style=pointy] (40.center) to (41.center);
		\draw [style=pointy] (44.center) to (42.center);
		\draw [style=pointy] (43.center) to (45.center);
		\draw [style=pointy] (48.center) to (49.center);
		\draw (50.center) to (51.center);
		\draw [style=Black arrow, in=90, out=-120] (54) to (17.center);
		\draw [style=Black arrow, in=90, out=-45] (54) to (44.center);
		\draw [style=Black arrow, in=225, out=90, looseness=0.75] (18.center) to (55);
		\draw [style=Black arrow, in=300, out=90] (45.center) to (55);
		\draw [style=Black arrow] (56.center) to (54);
		\draw [style=Black arrow] (55) to (57.center);
	\end{pgfonlayer}
\end{tikzpicture}} \]
\end{exa}

The above diagrams highlight the relationship between weight-tying and batching.
However, note that these techniques serve different purposes: while weight-tying
can be used to reduce the number of weights in a model, batching is used for
efficiency reasons to update a model with multiple examples in parallel.

\subsection{Loss Maps as Parametric Lenses}\label{subsec:lossmaps}

Another key component of any learning algorithm is the choice of loss map.  This
gives a measurement of how far the current output of the model is from the
desired output.  In standard learning in $\Smooth$, this loss map is viewed as a
map of type $B \times B \to \Rb$.  However, in our setup, this is naturally
viewed as a parametric map from $B$ to $\Rb$ with parameter space $B$.\footnote{Here the loss map has its parameter space equal to its input space. However, putting loss maps on the same footing as models lends itself to further generalizations where the parameter space is different, and where the loss map can itself be learned. See Generative Adversarial Networks, \cite[Figure 7.]{TowardsCatCyber}.} We
also generalize the codomain to an arbitrary object $L$.

\begin{defi}
A \textbf{loss map on $B$} consists of a parametric map $(B, \mbox{loss}) :
\Para(\Ca)(B, L)$ for some object $L$.
\end{defi}

Note that we can precompose a loss map $(B, \mbox{loss}) \colon B \to L$ with a neural network $(P, f): A \to B$ (below left), and apply the functor in \eqref{eq:para_rdc} (with $\Ca = \Smooth$) to obtain the parametric lens below right.
\begin{equation}\label{eq:para-nn-loss}
  \begin{gathered}
   \scalebox{0.8}{\begin{tikzpicture}
	\begin{pgfonlayer}{nodelayer}
		\node [style=none] (0) at (-2.75, -1.25) {};
		\node [style=none] (1) at (-2.75, -0.25) {};
		\node [style=none] (2) at (-1.25, -0.25) {};
		\node [style=none] (3) at (-1.25, -1.25) {};
		\node [style=none] (4) at (-2, -0.75) {$f$};
		\node [style=none] (5) at (-3.5, -0.75) {};
		\node [style=none] (6) at (-2.75, -0.75) {};
		\node [style=none] (7) at (-0.5, -0.75) {};
		\node [style=none] (8) at (-1.25, -0.75) {};
		\node [style=none] (9) at (-2, -0.25) {};
		\node [style=none] (10) at (-2, 0.75) {};
		\node [style=none] (11) at (-2, 1.25) {$P$};
		\node [style=none] (12) at (-4, -0.75) {$A$};
		\node [style=none] (13) at (1.25, -1.25) {};
		\node [style=none] (14) at (1.25, -0.25) {};
		\node [style=none] (15) at (2.75, -0.25) {};
		\node [style=none] (16) at (2.75, -1.25) {};
		\node [style=none] (17) at (2, -0.75) {$\mathsf{loss}$};
		\node [style=none] (18) at (0.5, -0.75) {};
		\node [style=none] (19) at (1.25, -0.75) {};
		\node [style=none] (20) at (3.5, -0.75) {};
		\node [style=none] (21) at (2.75, -0.75) {};
		\node [style=none] (22) at (2, -0.25) {};
		\node [style=none] (23) at (2, 0.75) {};
		\node [style=none] (24) at (2, 1.25) {$B$};
		\node [style=none] (25) at (4, -0.75) {$L$};
		\node [style=none] (26) at (0, 0) {$B$};
	\end{pgfonlayer}
	\begin{pgfonlayer}{edgelayer}
		\draw (1.center) to (2.center);
		\draw (2.center) to (3.center);
		\draw (3.center) to (0.center);
		\draw (0.center) to (1.center);
		\draw (5.center) to (6.center);
		\draw (8.center) to (7.center);
		\draw (10.center) to (9.center);
		\draw (14.center) to (15.center);
		\draw (15.center) to (16.center);
		\draw (16.center) to (13.center);
		\draw (13.center) to (14.center);
		\draw (18.center) to (19.center);
		\draw (21.center) to (20.center);
		\draw (23.center) to (22.center);
		\draw (7.center) to (18.center);
	\end{pgfonlayer}
\end{tikzpicture}} 
   \quad
   \mapsto 
   \quad
   \scalebox{0.8}{\begin{tikzpicture}
	\begin{pgfonlayer}{nodelayer}
		\node [style=none] (20) at (-3.25, 1) {};
		\node [style=none] (21) at (-3.25, -1) {};
		\node [style=none] (22) at (-0.75, 1) {};
		\node [style=none] (23) at (-0.75, -1) {};
		\node [style=none] (26) at (-4.75, 0.5) {$A$};
		\node [style=none] (27) at (-4.75, -0.5) {$A'$};
		\node [style=none] (28) at (0, 1) {$B$};
		\node [style=none] (29) at (0, -1) {$B'$};
		\node [style=none] (39) at (-4.25, 0.5) {};
		\node [style=none] (40) at (-3.25, 0.5) {};
		\node [style=none] (43) at (-3.25, -0.5) {};
		\node [style=none] (44) at (-4.25, -0.5) {};
		\node [style=none] (45) at (-0.75, 0.5) {};
		\node [style=none] (46) at (0, 0.5) {};
		\node [style=none] (57) at (0, -0.5) {};
		\node [style=none] (58) at (-0.75, -0.5) {};
		\node [style=none] (60) at (-2.5, 1) {};
		\node [style=none] (61) at (-1.5, 1) {};
		\node [style=none] (62) at (-2.5, 2) {};
		\node [style=none] (63) at (-1.5, 2) {};
		\node [style=none] (64) at (-2.5, 2.5) {$P$};
		\node [style=none] (65) at (-1.5, 2.5) {$P'$};
		\node [style=none] (66) at (0.75, 1) {};
		\node [style=none] (67) at (0.75, -1) {};
		\node [style=none] (68) at (3.25, 1) {};
		\node [style=none] (69) at (3.25, -1) {};
		\node [style=none] (72) at (4.75, 0.5) {$L$};
		\node [style=none] (73) at (4.75, -0.5) {$L'$};
		\node [style=none] (74) at (0, 0.5) {};
		\node [style=none] (75) at (0.75, 0.5) {};
		\node [style=none] (76) at (0.75, -0.5) {};
		\node [style=none] (77) at (0, -0.5) {};
		\node [style=none] (78) at (3.25, 0.5) {};
		\node [style=none] (79) at (4.25, 0.5) {};
		\node [style=none] (80) at (4.25, -0.5) {};
		\node [style=none] (81) at (3.25, -0.5) {};
		\node [style=none] (82) at (1.5, 1) {};
		\node [style=none] (83) at (2.5, 1) {};
		\node [style=none] (84) at (1.5, 2) {};
		\node [style=none] (85) at (2.5, 2) {};
		\node [style=none] (86) at (1.5, 2.5) {$B$};
		\node [style=none] (87) at (2.5, 2.5) {$B'$};
		\node [style=none] (210) at (-2, 0.5) {$f$};
		\node [style=none] (211) at (-2, -0.5) {$R[f]$};
		\node [style=none] (212) at (2, 0.5) {$\mathsf{loss}$};
		\node [style=none] (213) at (2, -0.5) {$R[\mathsf{loss}]$};
	\end{pgfonlayer}
	\begin{pgfonlayer}{edgelayer}
		\draw (20.center) to (22.center);
		\draw (22.center) to (23.center);
		\draw (23.center) to (21.center);
		\draw (21.center) to (20.center);
		\draw [style=pointy] (39.center) to (40.center);
		\draw [style=pointy] (43.center) to (44.center);
		\draw (45.center) to (46.center);
		\draw [style=pointy] (57.center) to (58.center);
		\draw [style=pointy] (62.center) to (60.center);
		\draw [style=pointy] (61.center) to (63.center);
		\draw (66.center) to (68.center);
		\draw (68.center) to (69.center);
		\draw (69.center) to (67.center);
		\draw (67.center) to (66.center);
		\draw [style=pointy] (74.center) to (75.center);
		\draw (76.center) to (77.center);
		\draw [style=pointy] (78.center) to (79.center);
		\draw [style=pointy] (80.center) to (81.center);
		\draw [style=pointy] (84.center) to (82.center);
		\draw [style=pointy] (83.center) to (85.center);
	\end{pgfonlayer}
\end{tikzpicture}}
  \end{gathered}
\end{equation}

This is getting closer to the parametric lens we want: it can now receive inputs of type
$B$.  However, this is at the cost of now needing an input to $L'$; we consider
how to handle this in the next section.

\begin{exa}[Quadratic error]\label{ex:l2_loss}
In $\Smooth$, the standard loss function on $\Rb^b$ is quadratic error: it uses $L = \Rb$ and has parametric map $e: \Rb^b  \times \Rb^b \to
\Rb$ given by 
\[
e(b_t,b_p) = \frac{1}{2} \sum_{i=1}^b ((b_p)_i - (b_t)_i)^2
  \]
where we think of $b_t$ as the ``true'' value and $b_p$ the predicted value.
This has reverse derivative $R[e]: \Rb^b \times \Rb^b \times \Rb \to \Rb^b
\times \Rb^b$ given by
$R[e](b_t,b_p,\alpha) = \alpha \cdot (b_p-b_t, b_t-b_p)$ --- note $\alpha$ suggests the idea of \textit{learning rate}, which we will explore in Section~\ref{subsec:learningrate}.
\end{exa}




\begin{exa}[Boolean error]\label{ex:boolean_error}
In $\PolyZ$, the loss function on $\Z^b$ which is implicitly used in \cite{rda} is a bit different: it uses $L = \Z^b$ and has parametric map $e: \Z^b \times \Z^b \to \Z^b$ given by
	\[ e(b_t,b_p) = b_t + b_p. \]
(Note that this is $+$ in $Z_2$; equivalently this is given by XOR.)   Its reverse derivative is of type $R[e]: \Z^b \times \Z^b \times \Z^b \to \Z^b \times \Z^b$ given by $R[e](b_t,b_p,\alpha) = (\alpha,\alpha)$.
\end{exa}

\begin{exa}[Softmax cross entropy]\label{ex:softmax_ce}
The Softmax cross entropy loss is a $\Rb^b$-parametric map $\Rb^b \to \Rb$
defined by 
\[
e(b_t, b_p) = \sum_{i=1}^{b}(b_t)_i((b_p)_i- \log(\Softmax(b_p)_i))
\]
where $\Softmax(b_p) = \frac{\exp((b_p)_i)}{\sum_{j=1}^{b}\exp((b_p)_j)}$ is defined componentwise for each class $i$.
\end{exa}

We note that, although $b_t$ needs to be a probability distribution, at the moment
there is no need to ponder the question of interaction of probability
distributions with the reverse derivative framework: one can simply consider $b_t$
as the image of some logits under the $\Softmax$ function.

\begin{exa}[Dot product]\label{ex:dot_product}
In Deep Dreaming (Section \ref{subsec:deep_dreaming}) we often want to focus
only on a particular element of the network output $\Rb^b$. This is done by
supplying a one-hot vector $b_t$ as the ground truth to the loss function $e(b_t, b_p) = b_t \cdot b_p$ which
computes the dot product of two vectors. If the ground truth vector $y$ is a one-hot vector (active at the $i$-th
element), then the dot product performs masking of all inputs except
the $i$-th one. Note the reverse derivative $R[e] \colon \Rb^b \times \Rb^b
\times \Rb \to \Rb^b \times \Rb^b$ of the dot product is defined as $R[e](b_t, b_p, \alpha) = (\alpha \cdot b_p, \alpha \cdot b_t)$.
\end{exa}

\subsection{Learning Rates as Parametric Lenses}\label{subsec:learningrate}


After models and loss maps, another ingredient of the learning process are \emph{learning rates}, which we formalise as follows. 

\begin{defi}
A \textbf{learning rate} $\alpha$ on $L$ consists of a lens from $\diset{L}{L'}$ to $\diset{1}{1}$ where $1$ is a terminal object in $\Ca$.
\end{defi}
Note that the get component of the learning rate lens must be the unique map to $1$, while the put component is a map $L \times 1 \to L'$; that is, simply a map $\alpha^*: L \to L'$.   Thus we can view $\alpha$ as a parametric lens from $\diset{L}{L'} \to \diset{1}{1}$ (with trivial parameter space) and compose it in $\Para(\Lens(\Ca))$ with a model and a loss map (\emph{cf.} \eqref{eq:para-nn-loss}) to get
\begin{equation}
   \scalebox{0.8}{\begin{tikzpicture}
	\begin{pgfonlayer}{nodelayer}
		\node [style=none] (20) at (-3.5, 1) {};
		\node [style=none] (21) at (-3.5, -1) {};
		\node [style=none] (22) at (-1, 1) {};
		\node [style=none] (23) at (-1, -1) {};
		\node [style=none] (26) at (-5, 0.5) {$A$};
		\node [style=none] (27) at (-5, -0.5) {$A'$};
		\node [style=none] (28) at (0, 1) {$B$};
		\node [style=none] (29) at (0, -1) {$B'$};
		\node [style=none] (39) at (-4.5, 0.5) {};
		\node [style=none] (40) at (-3.5, 0.5) {};
		\node [style=none] (43) at (-3.5, -0.5) {};
		\node [style=none] (44) at (-4.5, -0.5) {};
		\node [style=none] (45) at (-1, 0.5) {};
		\node [style=none] (46) at (0, 0.5) {};
		\node [style=none] (57) at (0, -0.5) {};
		\node [style=none] (58) at (-1, -0.5) {};
		\node [style=none] (60) at (-2.75, 1) {};
		\node [style=none] (61) at (-1.75, 1) {};
		\node [style=none] (62) at (-2.75, 2) {};
		\node [style=none] (63) at (-1.75, 2) {};
		\node [style=none] (64) at (-2.75, 2.5) {$P$};
		\node [style=none] (65) at (-1.75, 2.5) {$P'$};
		\node [style=none] (66) at (1, 1) {};
		\node [style=none] (67) at (1, -1) {};
		\node [style=none] (68) at (3.5, 1) {};
		\node [style=none] (69) at (3.5, -1) {};
		\node [style=none] (72) at (4.5, 1) {$L$};
		\node [style=none] (73) at (4.5, -1) {$L'$};
		\node [style=none] (74) at (0, 0.5) {};
		\node [style=none] (75) at (1, 0.5) {};
		\node [style=none] (76) at (1, -0.5) {};
		\node [style=none] (77) at (0, -0.5) {};
		\node [style=none] (78) at (3.5, 0.5) {};
		\node [style=none] (79) at (4.5, 0.5) {};
		\node [style=none] (80) at (4.5, -0.5) {};
		\node [style=none] (81) at (3.5, -0.5) {};
		\node [style=none] (82) at (1.75, 1) {};
		\node [style=none] (83) at (2.75, 1) {};
		\node [style=none] (84) at (1.75, 2) {};
		\node [style=none] (85) at (2.75, 2) {};
		\node [style=none] (86) at (1.75, 2.5) {$B$};
		\node [style=none] (87) at (2.75, 2.5) {$B'$};
		\node [style=none] (88) at (-2.25, 0.5) {$f$};
		\node [style=none] (89) at (2.25, 0.5) {$\mathsf{loss}$};
		\node [style=none] (90) at (-2.25, -0.5) {$R[f]$};
		\node [style=none] (91) at (2.25, -0.5) {$R[\mathsf{loss}]$};
		\node [style=none] (92) at (5.5, 1) {};
		\node [style=none] (93) at (5.5, -1) {};
		\node [style=none] (94) at (7.5, -1) {};
		\node [style=none] (95) at (8.5, 0) {};
		\node [style=none] (96) at (7.5, 1) {};
		\node [style=none] (97) at (5.5, -0.5) {};
		\node [style=none] (98) at (5.5, 0.5) {};
		\node [style=none] (99) at (6.75, 0) {$\alpha$};
	\end{pgfonlayer}
	\begin{pgfonlayer}{edgelayer}
		\draw (20.center) to (22.center);
		\draw (22.center) to (23.center);
		\draw (23.center) to (21.center);
		\draw (21.center) to (20.center);
		\draw [style=pointy] (39.center) to (40.center);
		\draw [style=pointy] (43.center) to (44.center);
		\draw (45.center) to (46.center);
		\draw [style=pointy] (57.center) to (58.center);
		\draw [style=pointy] (62.center) to (60.center);
		\draw [style=pointy] (61.center) to (63.center);
		\draw (66.center) to (68.center);
		\draw (68.center) to (69.center);
		\draw (69.center) to (67.center);
		\draw (67.center) to (66.center);
		\draw [style=pointy] (74.center) to (75.center);
		\draw (76.center) to (77.center);
		\draw (78.center) to (79.center);
		\draw [style=pointy] (80.center) to (81.center);
		\draw [style=pointy] (84.center) to (82.center);
		\draw [style=pointy] (83.center) to (85.center);
		\draw (92.center) to (96.center);
		\draw (96.center) to (95.center);
		\draw (95.center) to (94.center);
		\draw (94.center) to (93.center);
		\draw (93.center) to (92.center);
		\draw [style=pointy] (79.center) to (98.center);
		\draw [style=pointy] (97.center) to (80.center);
	\end{pgfonlayer}
\end{tikzpicture}}
  \label{fig:paralens-model-loss-cap}
\end{equation}
\begin{exa}\label{ex:smooth_learning_rate}
In standard supervised learning in $\Smooth$, one fixes some $\epsilon > 0$ as a learning rate, and this is used to define $\alpha$: $\alpha$ is simply constantly $-\epsilon$, ie., $\alpha(l) = -\epsilon$ for any $l \in L$.
\end{exa}

\begin{exa}\label{ex:boolean_learning_rate}
In supervised learning in $\PolyZ$, the standard learning rate is quite different: for a given $L$ it is defined as the identity function, $\alpha(l) = l$.
\end{exa}

Other learning rate morphisms are possible as well: for example, one could fix some $\epsilon > 0$ and define a learning rate in $\Smooth$ by $\alpha(l) = -\epsilon \cdot l$.
Such a choice would take into account how far away the network is from
its desired goal and adjust the learning rate accordingly.

\subsection{Optimisers as Reparameterisations}\label{subsec:optimisers}

In this section we consider how to implement gradient descent, ascent, and other gradient updates into our framework. To this aim, note that the parametric lens $\diset{f}{R[f]}$ representing our model (see \eqref{eq:para_rdc}) outputs a $P'$, which represents a \emph{change} in the parameter space.  Now, we would like to receive not just the requested change in the parameter, but the new parameter itself. This is precisely what gradient update accomplishes, when formalised as a lens.  
We start by describing gradient ascent and gradient descent.


  \begin{defi}[Gradient ascent]
    \label{def:gradient_ascent}
    Let $\Ca$ be a CRDC.
    Gradient ascent on $P : \Ca$ is a lens
    \[
      \diset{\id_P}{+_P} : \diset{P}{P} \to \diset{P}{P'}
    \] 
    where $+_P : P \times P' \to P$ is the monoid structure of $P$.\footnote{Note that as in the discussion in Section \ref{sec:crdc}, we are implicitly assuming that $P = P'$; we have merely notated them differently to emphasize the different ``roles'' they play (the first $P$ can be thought of as ``points'', the second as ``vectors'').}
  \end{defi}

\begin{figure}[h]
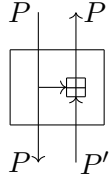

  \scaletikzfig[1]{gradient_ascent}
  \caption{Gradient Ascent}
\end{figure}

Intuitively, such a lens allows one to receive the requested change in parameter and implement that change by adding that value to the current parameter. By its type, we can now ``plug'' the gradient descent lens $G \colon \diset{P}{P} \to \diset{P}{P'}$ above the model $\diset{f}{R[f]}$ in \eqref{eq:para_rdc} --- formally, this is accomplished as a \emph{reparameterisation} of the parametric morphism $\diset{f}{R[f]}$, \emph{cf.} Section~\ref{sec:para}. This gives us Figure \ref{fig:model_and_optimiser} (left).
 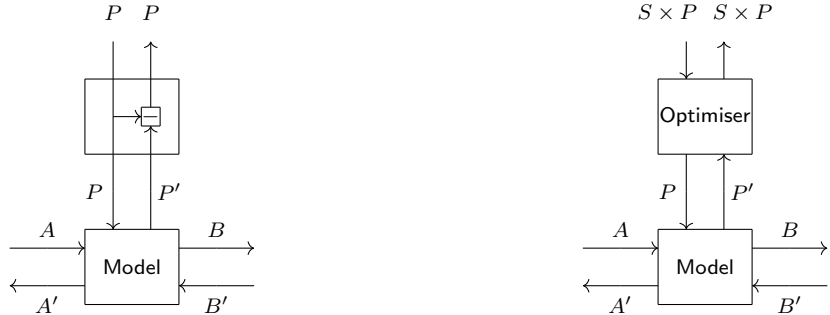
\begin{figure}[h]
\centering
\begin{subfigure}{.5\textwidth}
   \centering
  {\scalefont{0.75}
    \begin{tikzpicture}
	\begin{pgfonlayer}{nodelayer}
		\node [style=none] (20) at (-3.5, 1) {};
		\node [style=none] (21) at (-3.5, -1) {};
		\node [style=none] (22) at (-1, 1) {};
		\node [style=none] (23) at (-1, -1) {};
		\node [style=none] (26) at (-4.5, 1) {$A$};
		\node [style=none] (27) at (-4.5, -1) {$A'$};
		\node [style=none] (28) at (0, 1) {$B$};
		\node [style=none] (29) at (0, -1) {$B'$};
		\node [style=none] (39) at (-4.5, 0.5) {};
		\node [style=none] (40) at (-3.5, 0.5) {};
		\node [style=none] (43) at (-3.5, -0.5) {};
		\node [style=none] (44) at (-4.5, -0.5) {};
		\node [style=none] (45) at (-1, 0.5) {};
		\node [style=none] (46) at (0, 0.5) {};
		\node [style=none] (57) at (0, -0.5) {};
		\node [style=none] (58) at (-1, -0.5) {};
		\node [style=none] (60) at (-2.75, 1) {};
		\node [style=none] (61) at (-1.75, 1) {};
		\node [style=none] (62) at (-2.75, 2) {};
		\node [style=none] (63) at (-1.75, 2) {};
		\node [style=none] (64) at (-3.25, 2) {$P$};
		\node [style=none] (65) at (-1.25, 2) {$P'$};
		\node [style=none] (66) at (1, 1) {};
		\node [style=none] (67) at (1, -1) {};
		\node [style=none] (74) at (0, 0.5) {};
		\node [style=none] (75) at (1, 0.5) {};
		\node [style=none] (76) at (1, -0.5) {};
		\node [style=none] (77) at (0, -0.5) {};
		\node [style=none] (100) at (-3.5, 3) {};
		\node [style=none] (101) at (-1, 3) {};
		\node [style=none] (102) at (-1, 5) {};
		\node [style=none] (103) at (-3.5, 5) {};
		\node [style=none] (105) at (-2.75, 3) {};
		\node [style=none] (106) at (-1.75, 3) {};
		\node [style=none] (107) at (-2.75, 4) {};
		\node [style=none] (108) at (-2.75, 5) {};
		\node [style=none] (110) at (-1.75, 6) {};
		\node [style=none] (111) at (-2.75, 6) {};
		\node [style=none] (112) at (-2, 4.25) {};
		\node [style=none] (113) at (-2, 3.75) {};
		\node [style=none] (114) at (-1.75, 3.75) {};
		\node [style=none] (115) at (-1.5, 3.75) {};
		\node [style=none] (116) at (-1.5, 4.25) {};
		\node [style=none] (117) at (-1.75, 4.25) {};
		\node [style=none] (118) at (-2, 4) {};
		\node [style=none] (119) at (-1.75, 4) {$-$};
		\node [style=none] (122) at (-5.5, 1) {};
		\node [style=none] (123) at (-5.5, -1) {};
		\node [style=none] (126) at (-5.5, 0.5) {};
		\node [style=none] (127) at (-5.5, -0.5) {};
		\node [style=none] (130) at (-2.25, 0) {\textsf{Model}};
		\node [style=none] (136) at (-2.75, 6.75) {$P$};
		\node [style=none] (137) at (-1.75, 6.75) {$P$};
	\end{pgfonlayer}
	\begin{pgfonlayer}{edgelayer}
		\draw (20.center) to (22.center);
		\draw (22.center) to (23.center);
		\draw (23.center) to (21.center);
		\draw (21.center) to (20.center);
		\draw [style=pointy] (39.center) to (40.center);
		\draw (43.center) to (44.center);
		\draw (45.center) to (46.center);
		\draw [style=pointy] (57.center) to (58.center);
		\draw [style=pointy] (62.center) to (60.center);
		\draw (61.center) to (63.center);
		\draw [style=pointy] (74.center) to (75.center);
		\draw (76.center) to (77.center);
		\draw (103.center) to (102.center);
		\draw (102.center) to (101.center);
		\draw (101.center) to (100.center);
		\draw (100.center) to (103.center);
		\draw (62.center) to (105.center);
		\draw (63.center) to (106.center);
		\draw (111.center) to (108.center);
		\draw (108.center) to (107.center);
		\draw (107.center) to (105.center);
		\draw [style=pointy] (107.center) to (118.center);
		\draw [style=pointy] (117.center) to (110.center);
		\draw [style=pointy] (106.center) to (114.center);
		\draw (113.center) to (112.center);
		\draw (112.center) to (116.center);
		\draw (116.center) to (115.center);
		\draw (115.center) to (113.center);
		\draw (126.center) to (39.center);
		\draw [style=pointy] (44.center) to (127.center);
	\end{pgfonlayer}
\end{tikzpicture}
  }
\end{subfigure}%
\begin{subfigure}{.5\textwidth}
  \centering
  {\scalefont{0.75}
    \begin{tikzpicture}
	\begin{pgfonlayer}{nodelayer}
		\node [style=none] (20) at (-3.5, 1) {};
		\node [style=none] (21) at (-3.5, -1) {};
		\node [style=none] (22) at (-1, 1) {};
		\node [style=none] (23) at (-1, -1) {};
		\node [style=none] (26) at (-4.5, 1) {$A$};
		\node [style=none] (27) at (-4.5, -1) {$A'$};
		\node [style=none] (28) at (0, 1) {$B$};
		\node [style=none] (29) at (0, -1) {$B'$};
		\node [style=none] (39) at (-4.5, 0.5) {};
		\node [style=none] (40) at (-3.5, 0.5) {};
		\node [style=none] (43) at (-3.5, -0.5) {};
		\node [style=none] (44) at (-4.5, -0.5) {};
		\node [style=none] (45) at (-1, 0.5) {};
		\node [style=none] (46) at (0, 0.5) {};
		\node [style=none] (57) at (0, -0.5) {};
		\node [style=none] (58) at (-1, -0.5) {};
		\node [style=none] (60) at (-2.75, 1) {};
		\node [style=none] (61) at (-1.75, 1) {};
		\node [style=none] (62) at (-2.75, 2) {};
		\node [style=none] (63) at (-1.75, 2) {};
		\node [style=none] (64) at (-3.25, 2) {$P$};
		\node [style=none] (65) at (-1.25, 2) {$P'$};
		\node [style=none] (66) at (1, 1) {};
		\node [style=none] (67) at (1, -1) {};
		\node [style=none] (74) at (0, 0.5) {};
		\node [style=none] (75) at (1, 0.5) {};
		\node [style=none] (76) at (1, -0.5) {};
		\node [style=none] (77) at (0, -0.5) {};
		\node [style=none] (100) at (-3.5, 3) {};
		\node [style=none] (101) at (-1, 3) {};
		\node [style=none] (102) at (-1, 5) {};
		\node [style=none] (103) at (-3.5, 5) {};
		\node [style=none] (105) at (-2.75, 3) {};
		\node [style=none] (106) at (-1.75, 3) {};
		\node [style=none] (108) at (-2.75, 5) {};
		\node [style=none] (110) at (-1.75, 6) {};
		\node [style=none] (111) at (-2.75, 6) {};
		\node [style=none] (117) at (-1.75, 5) {};
		\node [style=none] (122) at (-5.5, 1) {};
		\node [style=none] (123) at (-5.5, -1) {};
		\node [style=none] (126) at (-5.5, 0.5) {};
		\node [style=none] (127) at (-5.5, -0.5) {};
		\node [style=none] (130) at (-2.25, 0) {\textsf{Model}};
		\node [style=none] (136) at (-3.25, 6.75) {$S \times P$};
		\node [style=none] (137) at (-1.25, 6.75) {$S \times P$};
		\node [style=none] (139) at (-2.25, 4) {\textsf{Optimiser}};
	\end{pgfonlayer}
	\begin{pgfonlayer}{edgelayer}
		\draw (20.center) to (22.center);
		\draw (22.center) to (23.center);
		\draw (23.center) to (21.center);
		\draw (21.center) to (20.center);
		\draw [style=pointy] (39.center) to (40.center);
		\draw (43.center) to (44.center);
		\draw (45.center) to (46.center);
		\draw [style=pointy] (57.center) to (58.center);
		\draw [style=pointy] (62.center) to (60.center);
		\draw (61.center) to (63.center);
		\draw [style=pointy] (74.center) to (75.center);
		\draw (76.center) to (77.center);
		\draw (103.center) to (102.center);
		\draw (102.center) to (101.center);
		\draw (101.center) to (100.center);
		\draw (100.center) to (103.center);
		\draw (62.center) to (105.center);
		\draw [style=pointy] (63.center) to (106.center);
		\draw [style=pointy] (111.center) to (108.center);
		\draw [style=pointy] (117.center) to (110.center);
		\draw (126.center) to (39.center);
		\draw [style=pointy] (44.center) to (127.center);
	\end{pgfonlayer}
\end{tikzpicture}
  }

\end{subfigure}
\caption{Model reparameterised by basic gradient descent (left) and a
  generic stateful optimiser (right).}
\label{fig:model_and_optimiser}
\end{figure}

\begin{exa}[Gradient ascent in Smooth]\label{ex:grad_desc_lens}
In $\Smooth$, the gradient ascent reparameterisation will take the output from $P'$ and add it to the current value of $P$ to get a new value of~$P$.
\end{exa}

\begin{exa}[Gradient ascent in Boolean circuits]\label{ex:boolean_circuit_descent}
In the CRDC $\PolyZ$, the gradient ascent reparameterisation will again take the output from $P'$ and add it to the current value of $P$ to get a new value of $P$; however, since $+$ in $\mathbb{Z}_2$ is the same as XOR, this can be also be seen as taking the XOR of the current parameter and the requested change; this is exactly how this algorithm is implemented in \cite{rda}.
\end{exa}


  \begin{defi}[Gradient descent]
    \label{def:gradient_descent}
    Let $\Ca$ be a CRDC where every monoid is additionally a commutative group.\footnote{Since a homomorphism between groups needs to satisfy \emph{less} equations than a monoid homomorphism, this means that every monoid homomorphism is also a group homomorphism. This in turn means there are no extra conditions we need to impose on such a CRDC equipped with group objects.}
    Gradient descent on $P$ is a lens
    \[
      \diset{\id_P}{-_P} : \diset{P}{P} \to \diset{P}{P'}
    \]
    where $-_P:(p, p') = p - p'$.
  \end{defi}

\begin{figure}[h]
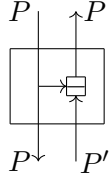

  \scaletikzfig[1]{gradient_descent}
  \caption{Gradient Descent}
\end{figure}

In $\Smooth$ this instantiates to usual gradient descent.
Gradient ascent in $\PolyZ$ is equal to gradient descent because $\XOR$ is its own inverse.
Intuitively in $\PolyZ$ there is always only one direction we can move (other than staying still): it's flipping the bit.
Gradient descent and ascent are not usually seen as a lens --- but they fit precisely into this picture that we are creating.

Other variants of gradient descent also fit naturally into this
framework by allowing for additional input/output data with $P$.    In
particular, many of them keep track of the
history of previous updates and use that to inform the next one.  This is easy to model in our setup: instead of asking for a lens $\diset{P}{P} \to \diset{P}{P'}$, we ask instead for a lens $\diset{S \times P}{S \times P} \to \diset{P}{P'}$ where $S$ is some ``state'' object.

\begin{defi}
A \textbf{stateful parameter update} consists of a choice of object $S$ (the \textbf{state} object) and a lens $U: \diset{S \times P}{S \times P} \to \diset{P}{P'}$.
\end{defi}

Again, we view this optimiser as a reparameterisation which may be ``plugged in'' a model as in Figure
~\ref{fig:model_and_optimiser} (right). Let us now consider how several well-known optimisers can be implemented in this way.

\begin{exa}[Momentum]
In the momentum variant of gradient descent, one keeps track of the previous
change and uses this to inform how the current parameter should be changed.
Thus, in this case, we set $S = P$, fix some $\gamma > 0$, and define the
\textbf{momentum} lens $\diset{U}{U^*}: \diset{P \times P}{P \times P} \to \diset{P}{P'}$.
by $U(s,p) = p$ and $ U^*(s,p,p') = (s', p + s')$, where $ s' = -\gamma s + p'$. Note momentum recovers
  gradient descent when $\gamma = 0$.
\end{exa}

In both standard gradient descent and momentum, our lens representation has trivial get part.  However, as soon as we move to more complicated variants, this is not anymore the case, as for instance in Nesterov momentum below.

\begin{exa}[Nesterov momentum]\label{ex:nesterov_momentum}
  In Nesterov momentum, one uses the momentum from previous updates to tweak the input parameter supplied to the network. We can precisely capture this by using a small variation of the lens in the previous example.  Again, we set $S = P$, fix some $\gamma > 0$, and define the \textbf{Nesterov momentum} lens $\diset{U}{U^*}: \diset{P \times P}{P \times P} \to \diset{P}{P'}$
by $U(s,p) = p + \gamma s$ and $U^*$ as in the previous example.
\end{exa}

\begin{exa}[Adagrad]\label{ex:adagrad}
Given any fixed $\epsilon > 0$ and $\delta \sim 10^{-7}$, Adagrad \cite{Adagrad}
is given by $S = P$, with the lens whose \gett{} part is $(g, p)
\mapsto p$. The \putt{} is $(g, p, p') \mapsto (g', p + \frac{\epsilon}{\delta + \sqrt{g'}} \odot p')$
where $g' = g + p' \odot p'$ and $\odot$ is the elementwise (Hadamard) product. Unlike with other optimization algorithms where the learning rate is the same
for all parameters, Adagrad divides the learning rate of each individual
parameter with the square root of the past accumulated gradients.
\end{exa}

\begin{exa}[Adam]\label{ex:adam}
Adaptive Moment Estimation (Adam) \cite{Adam}  is another method that computes adaptive
learning rates for each parameter by storing exponentially decaying average of past
gradients ($m$) and past squared gradients ($v$). For fixed $\beta_1, \beta_2 \in [0, 1)$, $\epsilon > 0$, and $\delta \sim 10^{-8}$, Adam is given by $S = P \times P$, with the lens whose \gett{} part is $(m, v, p) \mapsto p$ and whose \putt{} part is
$\putt{}(m, v, p, p') = (\widehat{m}', \widehat{v}', p + \frac{\epsilon}{\delta + \sqrt{\widehat{v}'}} \odot \widehat{m}')$ 
where $m' = \beta_1m + (1 - \beta_1)p'$, $v' = \beta_2v + (1 - \beta_2)p'^2$, and
	$\widehat{m}' = \frac{m'}{1 - \beta_1^t}, \widehat{v}' = \frac{v'}{1 - \beta_2^t}$.

\end{exa}

Note that, so far, optimsers/reparameterisations have been added to the $P/P'$ wires. In order to change the model's parameters (Fig.~\ref{fig:model_and_optimiser}).  In Section \ref{subsec:deep_dreaming} we will study them on the $A/A'$ wires instead, giving \emph{deep dreaming}.

\subsubsection{Can we compose optimisers?}
\label{subsubsec:can_we_compose_optimisers}

Even though not explicitly acknowledged in the literature, optimisers can be composed, and this composition plays an important role in settings where deep learning intersects with multivariable optimisation.
In such settings we're interested in their \emph{parallel} composition, therefore giving a positive answer to the above question.\footnote{One might wonder whether optimisers can be composed in sequence as well. The apparent sequential composability of optimisers is unfortunately an artefact of our limited view without dependent types.}
Parallel composition of optimisers arises out of the fact that optimisers are lenses, and lenses are a monoidal category (Remark~\ref{remark:lens_monoidal}).
In such settings we might have an optimiser of two variables which descends in one direction, and ascends in the other one, for instance. 

\begin{defi}[Gradient descent-ascent (GDA)]
  \label{ex:gradient_descent_ascent}
  Given objects $P$ and $Q$, gradient descent-ascent on $P \times Q$ is a lens
  \[
    \diset{P \times Q}{\gda} : \diset{P \times Q}{P \times Q} \to \diset{P \times Q}{P' \times Q'}
  \]
  where $\gda(p, q, p', q') = (p - p', q + q')$.
\end{defi}

In $\Smooth$ this gives an optimiser which descends on $P$ and ascents on $Q$.
In $\PolyZ$ this map ends up computing the same update function on both parameter spaces: the one that just flips the underlying bit.
This is something that ends up preventing us from modelling GANs in this setting (compare with Ex. \ref{ex:gan_gda_dotproduct} where both positive and negative polarity of the optimiser map is needed).

When it comes to optimisers of two parameters, gradient descent-ascent is a particular type of an optimiser that is a product of two optimisers.
But not all optimisers can be factored in such a way, much like a general monoidal product doesn't necessarily have to be cartesian.
A good example of this is an optimiser on two parameters called \emph{competitive gradient descent} (\cite{schafer_competitive_2020}).
We don't explicitly define or use it in this paper, instead inviting the reader to the aforementioned reference for more information.



\section{Learning with Parametric Lenses}
\label{section:learning-with-lenses}

In the previous section we have seen how all the components of learning can be
modeled as parametric lenses. We now study how all these components can be put together to form learning systems. We cover the most common 
examples of supervised learning, also discussing different kinds of layers, architectures, and techniques such as weight tying and batching. We also consider unsupervised learning, in the form of Generative Adversarial Networks. Finally, in addition to systems that learn \textit{parameters}, we study systems that learn their \textit{inputs}. This is a
technique commonly known as deep dreaming, and we present it as a natural
counterpart of supervised learning of parameters.

Before we describe these systems, it will be convenient to represent all the
inputs and outputs of our parametric lenses as parameters. In \eqref{fig:paralens-model-loss-cap}, we see the $P/P'$ and $B/B'$ inputs and outputs as parameters; however, the
$A/A'$ wires are not. To view the $A/A'$ inputs as parameters, we compose that
system with the parametric lens $\eta$ we now define. The parametric lens
$\eta$ has the type $\diset{1}{1} \to \diset{A}{A'}$ with parameter space $\diset{A}{A'}$ defined by
$(\gett{}_{\eta} = 1_A, \putt{}_{\eta} = \pi_1)$ and can be depicted graphically as $\scalebox{0.7}{ \begin{tikzpicture}
	\begin{pgfonlayer}{nodelayer}
		\node [style=none] (26) at (-4.5, 1) {$A$};
		\node [style=none] (27) at (-4.5, -1) {$A'$};
		\node [style=none] (39) at (-4.5, 0.5) {};
		\node [style=none] (44) at (-4.5, -0.5) {};
		\node [style=none] (120) at (-8, 1) {};
		\node [style=none] (121) at (-8, -1) {};
		\node [style=none] (122) at (-5.5, 1) {};
		\node [style=none] (123) at (-5.5, -1) {};
		\node [style=none] (124) at (-8, 0.5) {};
		\node [style=none] (125) at (-8, -0.5) {};
		\node [style=none] (126) at (-5.5, 0.5) {};
		\node [style=none] (127) at (-5.5, -0.5) {};
		\node [style=none] (128) at (-7.25, 1) {};
		\node [style=none] (129) at (-6.25, 1) {};
		\node [style=none] (131) at (-6.75, 0.5) {};
		\node [style=none] (132) at (-6.25, 0) {};
		\node [style=none] (133) at (-5.75, -0.5) {};
		\node [style=none] (134) at (-7.25, 2) {};
		\node [style=none] (135) at (-6.25, 2) {};
		\node [style=none] (138) at (-7.75, 2.25) {$A$};
	\end{pgfonlayer}
	\begin{pgfonlayer}{edgelayer}
		\draw [style=lightnone] (120.center) to (122.center);
		\draw [style=lightnone] (122.center) to (123.center);
		\draw [style=lightnone] (123.center) to (121.center);
		\draw [style=lightnone] (121.center) to (120.center);
		\draw [style=pointy] (126.center) to (39.center);
		\draw (44.center) to (127.center);
		\draw (126.center) to (131.center);
		\draw [in=-90, out=-180, looseness=1.25] (131.center) to (128.center);
		\draw [bend right=315, looseness=1.25] (133.center) to (132.center);
		\draw (127.center) to (133.center);
		\draw (132.center) to (129.center);
		\draw (134.center) to (128.center);
		\draw [style=pointy] (129.center) to (135.center);
	\end{pgfonlayer}
\end{tikzpicture}}$.
Composing $\eta$ with the rest of the learning system in 
~\eqref{fig:paralens-model-loss-cap} gives us the closed parametric lens

\begin{equation}
    \scalebox{0.8}{ \begin{tikzpicture}
	\begin{pgfonlayer}{nodelayer}
		\node [style=none] (20) at (-3.5, 1) {};
		\node [style=none] (21) at (-3.5, -1) {};
		\node [style=none] (22) at (-1, 1) {};
		\node [style=none] (23) at (-1, -1) {};
		\node [style=none] (26) at (-4.5, 1) {$A$};
		\node [style=none] (27) at (-4.5, -1) {$A'$};
		\node [style=none] (28) at (0, 1) {$B$};
		\node [style=none] (29) at (0, -1) {$B'$};
		\node [style=none] (39) at (-4.5, 0.5) {};
		\node [style=none] (40) at (-3.5, 0.5) {};
		\node [style=none] (43) at (-3.5, -0.5) {};
		\node [style=none] (44) at (-4.5, -0.5) {};
		\node [style=none] (45) at (-1, 0.5) {};
		\node [style=none] (46) at (0, 0.5) {};
		\node [style=none] (57) at (0, -0.5) {};
		\node [style=none] (58) at (-1, -0.5) {};
		\node [style=none] (60) at (-2.75, 1) {};
		\node [style=none] (61) at (-1.75, 1) {};
		\node [style=none] (62) at (-2.75, 2) {};
		\node [style=none] (63) at (-1.75, 2) {};
		\node [style=none] (64) at (-3.25, 2) {$P$};
		\node [style=none] (65) at (-1.25, 2) {$P'$};
		\node [style=none] (66) at (1, 1) {};
		\node [style=none] (67) at (1, -1) {};
		\node [style=none] (68) at (3.5, 1) {};
		\node [style=none] (69) at (3.5, -1) {};
		\node [style=none] (72) at (4.5, 1) {$L$};
		\node [style=none] (73) at (4.5, -1) {$L'$};
		\node [style=none] (74) at (0, 0.5) {};
		\node [style=none] (75) at (1, 0.5) {};
		\node [style=none] (76) at (1, -0.5) {};
		\node [style=none] (77) at (0, -0.5) {};
		\node [style=none] (78) at (3.5, 0.5) {};
		\node [style=none] (79) at (4.5, 0.5) {};
		\node [style=none] (80) at (4.5, -0.5) {};
		\node [style=none] (81) at (3.5, -0.5) {};
		\node [style=none] (82) at (1.75, 1) {};
		\node [style=none] (83) at (2.75, 1) {};
		\node [style=none] (84) at (1.75, 2) {};
		\node [style=none] (85) at (2.75, 2) {};
		\node [style=none] (86) at (1, 2) {$B$};
		\node [style=none] (87) at (3.5, 2) {$B'$};
		\node [style=none] (89) at (2.25, 0) {\textsf{Loss}};
		\node [style=none] (92) at (5.5, 1) {};
		\node [style=none] (93) at (5.5, -1) {};
		\node [style=none] (94) at (7.5, -1) {};
		\node [style=none] (95) at (8.5, 0) {};
		\node [style=none] (96) at (7.5, 1) {};
		\node [style=none] (97) at (5.5, -0.5) {};
		\node [style=none] (98) at (5.5, 0.5) {};
		\node [style=none] (99) at (6.75, 0) {$\alpha$};
		\node [style=none] (120) at (-8, 1) {};
		\node [style=none] (121) at (-8, -1) {};
		\node [style=none] (122) at (-5.5, 1) {};
		\node [style=none] (123) at (-5.5, -1) {};
		\node [style=none] (124) at (-8, 0.5) {};
		\node [style=none] (125) at (-8, -0.5) {};
		\node [style=none] (126) at (-5.5, 0.5) {};
		\node [style=none] (127) at (-5.5, -0.5) {};
		\node [style=none] (128) at (-7.25, 1) {};
		\node [style=none] (129) at (-6.25, 1) {};
		\node [style=none] (130) at (-2.25, 0) {\textsf{Model}};
		\node [style=none] (131) at (-6.75, 0.5) {};
		\node [style=none] (132) at (-6.25, 0) {};
		\node [style=none] (133) at (-5.75, -0.5) {};
		\node [style=none] (134) at (-7.25, 2) {};
		\node [style=none] (135) at (-6.25, 2) {};
		\node [style=none] (138) at (-7.75, 2.25) {$A$};
		\node [style=none] (139) at (-5.75, 2.25) {$A'$};
	\end{pgfonlayer}
	\begin{pgfonlayer}{edgelayer}
		\draw (20.center) to (22.center);
		\draw (22.center) to (23.center);
		\draw (23.center) to (21.center);
		\draw (21.center) to (20.center);
		\draw [style=pointy] (39.center) to (40.center);
		\draw (43.center) to (44.center);
		\draw (45.center) to (46.center);
		\draw [style=pointy] (57.center) to (58.center);
		\draw [style=pointy] (62.center) to (60.center);
		\draw [style=pointy] (61.center) to (63.center);
		\draw (66.center) to (68.center);
		\draw (68.center) to (69.center);
		\draw (69.center) to (67.center);
		\draw (67.center) to (66.center);
		\draw [style=pointy] (74.center) to (75.center);
		\draw (76.center) to (77.center);
		\draw (78.center) to (79.center);
		\draw [style=pointy] (80.center) to (81.center);
		\draw [style=pointy] (84.center) to (82.center);
		\draw [style=pointy] (83.center) to (85.center);
		\draw (92.center) to (96.center);
		\draw (96.center) to (95.center);
		\draw (95.center) to (94.center);
		\draw (94.center) to (93.center);
		\draw (93.center) to (92.center);
		\draw [style=pointy] (79.center) to (98.center);
		\draw (97.center) to (80.center);
		\draw [style=lightnone] (120.center) to (122.center);
		\draw [style=lightnone] (122.center) to (123.center);
		\draw [style=lightnone] (123.center) to (121.center);
		\draw [style=lightnone] (121.center) to (120.center);
		\draw (126.center) to (39.center);
		\draw (44.center) to (127.center);
		\draw (126.center) to (131.center);
		\draw [in=-90, out=-180, looseness=1.25] (131.center) to (128.center);
		\draw [bend right=315, looseness=1.25] (133.center) to (132.center);
		\draw (127.center) to (133.center);
		\draw (132.center) to (129.center);
		\draw [style=pointy] (134.center) to (128.center);
		\draw [style=pointy] (129.center) to (135.center);
	\end{pgfonlayer}
\end{tikzpicture}}
  \label{fig:supervised}
\end{equation}
This composite is now a map in $\Para(\Lens(\Ca))$ from $\diset{1}{1}$ to $\diset{1}{1}$; all
its inputs and outputs are now vertical wires, ie., parameters. Unpacking it
further, this is a lens of type $\diset{A \times P \times B}{A' \times P' \times B'}
\to \diset{1}{1}$ whose $\gett{}$ map is the terminal map, and whose $\putt{}$ map is
of the type $A \times P \times B \to A' \times P' \times B'$. It can be unpacked
as the composite 
\begin{align*}
  \putt{}(a, p, b_t) = (a', p', b_t') \qquad \texttt{ where } \quad \qquad b_p &= f(p, a) \\
  (b'_t, b'_p) &= R[\loss](b_t, b_p, \alpha(\loss(b_t, b_p))) \\
  (p', a') &= R[f](p, a, b'_p)
\end{align*}
In the next two sections we consider further additions to the image above which
correspond to different types of supervised learning.

\subsection{Supervised Learning of Parameters}\label{subsec:learning-parameters}


The most common type of learning performed on \eqref{fig:supervised} is supervised learning of \textit{parameters}.
This is done by reparameterising (\emph{cf.} Section \ref{sec:para}) the image 
in the following manner. The parameter ports are reparameterised by one of the
(possibly stateful) optimisers described in the previous section, while the backward wires $A'$ of inputs and $B'$ of outputs are discarded.
This finally yields the complete picture of a system which learns the
parameters in a supervised~manner:

\begin{equation*}
  {\scalefont{0.75}
    \begin{tikzpicture}
	\begin{pgfonlayer}{nodelayer}
		\node [style=none] (20) at (-3.5, 1) {};
		\node [style=none] (21) at (-3.5, -1) {};
		\node [style=none] (22) at (-1, 1) {};
		\node [style=none] (23) at (-1, -1) {};
		\node [style=none] (26) at (-4.5, 1) {$A$};
		\node [style=none] (27) at (-4.5, -1) {$A'$};
		\node [style=none] (28) at (0, 1) {$B$};
		\node [style=none] (29) at (0, -1) {$B'$};
		\node [style=none] (39) at (-4.5, 0.5) {};
		\node [style=none] (40) at (-3.5, 0.5) {};
		\node [style=none] (43) at (-3.5, -0.5) {};
		\node [style=none] (44) at (-4.5, -0.5) {};
		\node [style=none] (45) at (-1, 0.5) {};
		\node [style=none] (46) at (0, 0.5) {};
		\node [style=none] (57) at (0, -0.5) {};
		\node [style=none] (58) at (-1, -0.5) {};
		\node [style=none] (60) at (-2.75, 1) {};
		\node [style=none] (61) at (-1.75, 1) {};
		\node [style=none] (62) at (-2.75, 2) {};
		\node [style=none] (63) at (-1.75, 2) {};
		\node [style=none] (64) at (-3.25, 2) {$P$};
		\node [style=none] (65) at (-1.25, 2) {$P'$};
		\node [style=none] (66) at (1, 1) {};
		\node [style=none] (67) at (1, -1) {};
		\node [style=none] (68) at (3.5, 1) {};
		\node [style=none] (69) at (3.5, -1) {};
		\node [style=none] (72) at (4.5, 1) {$L$};
		\node [style=none] (73) at (4.5, -1) {$L'$};
		\node [style=none] (74) at (0, 0.5) {};
		\node [style=none] (75) at (1, 0.5) {};
		\node [style=none] (76) at (1, -0.5) {};
		\node [style=none] (77) at (0, -0.5) {};
		\node [style=none] (78) at (3.5, 0.5) {};
		\node [style=none] (79) at (4.5, 0.5) {};
		\node [style=none] (80) at (4.5, -0.5) {};
		\node [style=none] (81) at (3.5, -0.5) {};
		\node [style=none] (82) at (1.75, 1) {};
		\node [style=none] (83) at (2.75, 1) {};
		\node [style=none] (84) at (1.75, 6) {};
		\node [style=node] (85) at (2.75, 2.5) {};
		\node [style=none] (86) at (1.75, 6.5) {$B$};
		\node [style=none] (87) at (3.25, 1.5) {$B'$};
		\node [style=none] (89) at (2.25, 0) {\textsf{Loss}};
		\node [style=none] (92) at (5.5, 1) {};
		\node [style=none] (93) at (5.5, -1) {};
		\node [style=none] (94) at (7.5, -1) {};
		\node [style=none] (95) at (8.5, 0) {};
		\node [style=none] (96) at (7.5, 1) {};
		\node [style=none] (97) at (5.5, -0.5) {};
		\node [style=none] (98) at (5.5, 0.5) {};
		\node [style=none] (99) at (6.75, 0) {$\alpha$};
		\node [style=none] (100) at (-3.5, 3) {};
		\node [style=none] (101) at (-1, 3) {};
		\node [style=none] (102) at (-1, 5) {};
		\node [style=none] (103) at (-3.5, 5) {};
		\node [style=none] (105) at (-2.75, 3) {};
		\node [style=none] (106) at (-1.75, 3) {};
		\node [style=none] (108) at (-2.75, 5) {};
		\node [style=none] (110) at (-1.75, 6) {};
		\node [style=none] (111) at (-2.75, 6) {};
		\node [style=none] (117) at (-1.75, 5) {};
		\node [style=none] (120) at (-8, 1) {};
		\node [style=none] (121) at (-8, -1) {};
		\node [style=none] (122) at (-5.5, 1) {};
		\node [style=none] (123) at (-5.5, -1) {};
		\node [style=none] (124) at (-8, 0.5) {};
		\node [style=none] (125) at (-8, -0.5) {};
		\node [style=none] (126) at (-5.5, 0.5) {};
		\node [style=none] (127) at (-5.5, -0.5) {};
		\node [style=none] (128) at (-7.25, 1) {};
		\node [style=none] (129) at (-6.25, 1) {};
		\node [style=none] (130) at (-2.25, 0) {\textsf{Model}};
		\node [style=none] (131) at (-6.75, 0.5) {};
		\node [style=none] (132) at (-6.25, 0) {};
		\node [style=none] (133) at (-5.75, -0.5) {};
		\node [style=none] (134) at (-7.25, 6) {};
		\node [style=node] (135) at (-6.25, 2.5) {};
		\node [style=none] (136) at (-3.25, 6.75) {$S \times P$};
		\node [style=none] (137) at (-1.25, 6.75) {$S \times P$};
		\node [style=none] (138) at (-7.25, 6.75) {$A$};
		\node [style=none] (139) at (-2.25, 4) {\textsf{Optimiser}};
	\end{pgfonlayer}
	\begin{pgfonlayer}{edgelayer}
		\draw (20.center) to (22.center);
		\draw (22.center) to (23.center);
		\draw (23.center) to (21.center);
		\draw (21.center) to (20.center);
		\draw [style=pointy] (39.center) to (40.center);
		\draw (43.center) to (44.center);
		\draw (45.center) to (46.center);
		\draw [style=pointy] (57.center) to (58.center);
		\draw [style=pointy] (62.center) to (60.center);
		\draw (61.center) to (63.center);
		\draw (66.center) to (68.center);
		\draw (68.center) to (69.center);
		\draw (69.center) to (67.center);
		\draw (67.center) to (66.center);
		\draw [style=pointy] (74.center) to (75.center);
		\draw (76.center) to (77.center);
		\draw (78.center) to (79.center);
		\draw [style=pointy] (80.center) to (81.center);
		\draw [style=pointy] (84.center) to (82.center);
		\draw (83.center) to (85);
		\draw (92.center) to (96.center);
		\draw (96.center) to (95.center);
		\draw (95.center) to (94.center);
		\draw (94.center) to (93.center);
		\draw (93.center) to (92.center);
		\draw [style=pointy] (79.center) to (98.center);
		\draw [style=pointy] (97.center) to (80.center);
		\draw (103.center) to (102.center);
		\draw (102.center) to (101.center);
		\draw (101.center) to (100.center);
		\draw (100.center) to (103.center);
		\draw (62.center) to (105.center);
		\draw (63.center) to (106.center);
		\draw [style=pointy] (111.center) to (108.center);
		\draw [style=pointy] (117.center) to (110.center);
		\draw [style=lightnone] (120.center) to (122.center);
		\draw [style=lightnone] (122.center) to (123.center);
		\draw [style=lightnone] (123.center) to (121.center);
		\draw [style=lightnone] (121.center) to (120.center);
		\draw (126.center) to (39.center);
		\draw [style=pointy] (44.center) to (127.center);
		\draw (126.center) to (131.center);
		\draw [in=-90, out=-180, looseness=1.25] (131.center) to (128.center);
		\draw [bend right=315, looseness=1.25] (133.center) to (132.center);
		\draw (127.center) to (133.center);
		\draw (132.center) to (129.center);
		\draw [style=pointy] (134.center) to (128.center);
		\draw (129.center) to (135);
	\end{pgfonlayer}
\end{tikzpicture}
  }
\end{equation*}

Fixing a particular optimiser $\diset{U}{U^*} : \diset{S \times P}{S \times P} \to \diset{P}{P'}$
we again unpack the entire construction. This is a map in $\Para(\Lens(\Ca))$
from $\diset{1}{1}$ to $\diset{1}{1}$ whose parameter space is $\diset{A \times S \times P \times
B}{S \times P}$. In other words, this is a lens of type $\diset{A \times S \times P
\times B}{S \times P}  \to \diset{1}{1}$ whose $\gett{}$ component is the terminal map. Its $\putt{}$ map has the type $A \times S \times P \times B \to S \times P$ and unpacks to $\putt{}(a, s, p, b_t) = U^*(s, p, p')$, where
\begin{align*}
  \putt{}(a, s, p, b_t) = U^*(s, p, p') \qquad \texttt{ where } \quad \qquad \overline{p}  &= U(s, p) \\ b_p &= f(\overline{p}, a) \\
  (b'_t, b'_p) &= R[\loss](b_t, b_p, \alpha(\loss(b_t, b_p))) \\
  (p', a') &= R[f](\overline{p}, a, b'_p)
\end{align*}

While this formulation might seem daunting, we note that it just explicitly
specifies the computation performed by a supervised learning system. The
variable $\overline{p}$ represents the parameter supplied to the network by the
stateful gradient update rule (in many cases this is equal to $p$); $b_p$
represents the prediction of the network (contrast this with $b_t$ which
represents the ground truth from the dataset).  Variables with a tick $'$
represent changes: $ b'_p$ and $b'_t$ are the changes on predictions and true
values respectively, while $p'$ and $a'$ are changes on the parameters and
inputs.
Furthermore, this arises automatically out of the rule for lens composition~\eqref{eq:lens-composition}; what we needed to specify is just the lenses themselves.

We justify and illustrate our approach on a series of case studies drawn from the literature. This presentation has the advantage of treating all these instances uniformly in terms of basic constructs, highlighting their similarities and differences. First, we fix some parametric map $(\Rb^p, f) :
\Para(\Smooth)(\Rb^a, \Rb^b)$ in $\Smooth$ and the constant \textit{negative} learning rate $\alpha : \Rb$ (Example ~\ref{ex:smooth_learning_rate}). We then vary the loss
function and the gradient update, seeing how the $\putt{}$ map above reduces to
many of the known cases in the literature.

\begin{exa}[Quadratic error, basic gradient descent]\label{ex:quad_err_grad_desc}
Fix the quadratic error (Example \ref{ex:l2_loss}) as the loss map and basic
gradient update (Example \ref{ex:grad_desc_lens}).
Then the aforementioned $\putt{}$ map simplifies.
Since there is no state, its type reduces to $A \times P \times B
\to P$, and we have $\putt{}(a, p, b_t) = p + p'$, where $(p',
a') = R[f](p, a, \alpha \cdot (f(p, a) - b_t))$.

Note that $\alpha$ here is simply a constant, and due to the linearity of the
reverse derivative (Def ~\ref{defn:CRDC}), we can slide the $\alpha$ from the costate
into the basic gradient update lens. Rewriting this update, and performing this
sliding we obtain a closed form update step 
\[
\putt{}(a, p, b_t) = p + \alpha \cdot (R[f](p, a, f(p, a) - b_t);\pi_0)
\]
where the negative \textit{descent} component of gradient descent is here contained in the choice of the negative constant $\alpha$.
\end{exa}

This example gives us a variety of \textit{regression} algorithms solved
iteratively by gradient descent: it embeds some parametric map $(\R^p, f) \colon \R^a \to \R^b$ into the system which performs regression on input
data - where $a$ denotes the input to the model and $b_t$ denotes the ground truth. If the corresponding $f$ is linear and $b = 1$, we recover simple linear
regression with gradient descent. If the codomain is 
multi-dimensional, i.e. we are predicting multiple scalars, then we recover
multivariate linear regression. Likewise, we can model a
multi-layer perceptron or even more complex neural network architectures
performing supervised learning of parameters simply
by changing the underlying parametric map.

\begin{exa}[Softmax cross entropy, basic gradient descent]
Fix Softmax cross entropy (Example \ref{ex:softmax_ce}) as the loss map and
basic gradient update (Example \ref{ex:grad_desc_lens}). 
Again the $\putt{}$ map simplifies. 
The type reduces to $A \times P \times B \to P$ and we have 
\[
\putt{}(a, p, b_t) = p + p'
\]
where $(p', a') = R[f](\overline{p}, a, \alpha \cdot (\Softmax(f(p, a)) - b_t))$.
The same rewriting performed on the previous example can be done here.
\end{exa}

This example recovers \textit{logistic regression}, e.g. classification.

\begin{exa}[Mean squared error, Nesterov Momentum]
Fix the quadratic error (Example \ref{ex:l2_loss}) as the loss map and Nesterov
momentum (Example \ref{ex:nesterov_momentum}) as the gradient update. This time the
$\putt{}$ map $A \times S \times P
\times B \to S \times P$ does not have a simplified type. The implementation of $\putt{}$ reduces to 
\begin{align*}
 \putt{}(a, s, p, b_t) = (s', p + s') \qquad \texttt{ where } \quad \qquad  \overline{p}  &= p + \gamma s \\
  (p', a') &= R[f](\overline{p}, a, \alpha \cdot (f(\overline{p}, a) - b_t)) \\
  s' &= - \gamma s + p' 
\end{align*}
\end{exa}

This example with Nesterov momentum differs in two key points from all the other
ones: i) the optimiser is stateful, and ii) its $\gett{}$ map is not trivial.
While many other optimisers are stateful, the non-triviality of the $\gett{}$
map here showcases the importance of lenses. They allow us to make precise the notion
of computing a ``lookahead'' value for Nesterov momentum, something that is in
practice usually handled in ad-hoc ways. Here, the algebra of lens composition handles this case naturally by
using the $\gett{}$ map, a seemingly trivial, unused piece of data for previous optimisers.

Our last example, using a different base category $\PolyZ$, shows that our framework captures learning in not just continuous, but discrete settings too.
Again, we fix a parametric map $(\Zb^p, f) : \PolyZ(\Zb^a, \Zb^b)$ but this
time we fix the identity learning rate (Example \ref{ex:boolean_learning_rate}),
instead of a constant one.

\begin{exa}[Basic learning in Boolean circuits]
Fix XOR as the loss map (Example \ref{ex:boolean_error}) and the basic gradient update (Example \ref{ex:boolean_circuit_descent}). The put map again simplifies. The type reduces to $A \times P \times B \to P$ and the implementation to $\putt{}(a, p, b_t) = p + p'$ where $(p', a') = R[f](p, a, f(p, a) + b_t)$.
\end{exa}

\textbf{A sketch of learning iteration.}
Having described a number of examples in supervised learning, we outline how to
model learning iteration in our framework.
Recall the aforementioned  \putt{} map whose type is $A \times P
\times B \to P$ (for simplicity here modelled without state $S$). This map takes an
input-output pair $(a_0, b_0)$, the current parameter $p_i$ and produces an updated
parameter $p_{i + 1}$. At the next time step, it takes a potentially different
input-output pair $(a_1, b_1)$, the updated parameter $p_{i + 1}$ and produces
$p_{i + 2}$. This process is then repeated. We can model this iteration as a composition of the $\putt{}$ map with itself, as a composite $(A \times \putt{} \times B);\putt{}$
whose type is $A \times A \times P \times B \times B \to P$. This map takes two
input-output pairs $A \times B$, a parameter and produces a new parameter by
processing these datapoints in sequence. One can see how
this process can be iterated any number of times, and even represented as a
string diagram.

But we note that with a slight reformulation of the \putt{} map, it is possible to obtain a conceptually much simpler definition. The key insight lies in seeing that the
map $\putt{} : A \times P \times B \to P$ is essentially an endo-map $P \to P$ with some extra inputs $A \times B$; it's a parametric map!

In other words, we can recast the $\putt{}$ map as a parametric map $(A
\times B, \putt{}) : \Para(\Ca)(P, P)$. Being an endo-map, it can be
composed with itself. The resulting composite is an endo-map taking two ``parameters'': input-output pair at the time step $0$ and time step
$1$. This process can then be repeated, with $\Para$ composition automatically taking care of the algebra of iteration.

\begin{equation*}
  \scalebox{0.8}{\begin{tikzpicture}
	\begin{pgfonlayer}{nodelayer}
		\node [style=none] (27) at (-6, 0.5) {};
		\node [style=none] (28) at (-4, -0.5) {};
		\node [style=none] (29) at (-4, 0.5) {};
		\node [style=none] (30) at (-6, -0.5) {};
		\node [style=none] (31) at (-6, 0) {};
		\node [style=none] (33) at (-5, 0.5) {};
		\node [style=none] (34) at (-4, 0) {};
		\node [style=none] (35) at (-2, 0) {};
		\node [style=none] (37) at (-7, 0) {};
		\node [style=none] (38) at (-7.5, 0) {$P$};
		\node [style=none] (39) at (-5, 0) {\textsf{put}};
		\node [style=none] (41) at (-5, 1.75) {};
		\node [style=none] (43) at (-5, 2.25) {$A \times B$};
		\node [style=none] (44) at (-2, 0.5) {};
		\node [style=none] (45) at (0, -0.5) {};
		\node [style=none] (46) at (0, 0.5) {};
		\node [style=none] (47) at (-2, -0.5) {};
		\node [style=none] (48) at (-2, 0) {};
		\node [style=none] (50) at (-1, 0.5) {};
		\node [style=none] (51) at (0, 0) {};
		\node [style=none] (52) at (1, 0) {};
		\node [style=none] (55) at (-3, 0.5) {$P$};
		\node [style=none] (56) at (-1, 0) {\textsf{put}};
		\node [style=none] (58) at (-1, 1.75) {};
		\node [style=none] (60) at (-1, 2.25) {$A \times B$};
		\node [style=none] (61) at (3.75, 0.5) {};
		\node [style=none] (62) at (5.75, -0.5) {};
		\node [style=none] (63) at (5.75, 0.5) {};
		\node [style=none] (64) at (3.75, -0.5) {};
		\node [style=none] (65) at (3.75, 0) {};
		\node [style=none] (67) at (4.75, 0.5) {};
		\node [style=none] (68) at (5.75, 0) {};
		\node [style=none] (69) at (6.75, 0) {};
		\node [style=none] (70) at (7.25, 0) {$P$};
		\node [style=none] (71) at (2.75, 0) {};
		\node [style=none] (72) at (4.75, 0) {\textsf{put}};
		\node [style=none] (74) at (4.75, 1.75) {};
		\node [style=none] (76) at (4.75, 2.25) {$A \times B$};
		\node [style=none] (77) at (1.75, 0) {$\overset{n}{\ldots}$};
	\end{pgfonlayer}
	\begin{pgfonlayer}{edgelayer}
		\draw (27.center) to (31.center);
		\draw (31.center) to (30.center);
		\draw (28.center) to (30.center);
		\draw (28.center) to (34.center);
		\draw (34.center) to (29.center);
		\draw (29.center) to (33.center);
		\draw [style=pointy] (37.center) to (31.center);
		\draw [style=pointy] (34.center) to (35.center);
		\draw [style=pointy] (41.center) to (33.center);
		\draw (44.center) to (48.center);
		\draw (48.center) to (47.center);
		\draw (45.center) to (47.center);
		\draw (45.center) to (51.center);
		\draw (51.center) to (46.center);
		\draw (46.center) to (50.center);
		\draw [style=pointy] (51.center) to (52.center);
		\draw [style=pointy] (58.center) to (50.center);
		\draw (61.center) to (65.center);
		\draw (65.center) to (64.center);
		\draw (62.center) to (64.center);
		\draw (62.center) to (68.center);
		\draw (68.center) to (63.center);
		\draw (63.center) to (67.center);
		\draw [style=pointy] (71.center) to (65.center);
		\draw [style=pointy] (68.center) to (69.center);
		\draw [style=pointy] (74.center) to (67.center);
		\draw (27.center) to (33.center);
		\draw (44.center) to (50.center);
		\draw (61.center) to (67.center);
	\end{pgfonlayer}
\end{tikzpicture}}
\end{equation*}

This reformulation captures the essence of parameter iteration: one can think of it as a
trajectory $p_i, p_{i + 1}, p_{i + 2}, ...$ through the parameter space; but
it is a \textit{trajectory parameterised by the dataset}. With different datasets the
algorithm will take a different path through this space and learn different things.

\subsection{Unsupervised Learning}
\label{subsec:unsupervised-learning}

Many kinds of systems that are traditionally considered unsupervised can be recast to their supervised form.
One example is a Generative Adversarial Network (\cite{GAN, arjovsky_wasserstein_2017}).
This is a neural network architecture that lies in the centre of the intersection of deep learning and game theory.
It is a system of two neural networks trained with ``competing'' optimisers.
One neural network is called \emph{the generator} whose optimiser is, as usual, tasked with moving in the direction of the negative gradient of the loss.
However, the other network --- called \emph{the discriminator} --- has an optimiser which is tasked with moving in the \emph{positive}, i.e.\ ascending direction of the gradient of the total loss --- maximising the loss.
The actual networks are wired in such a way (Fig.\ \ref{fig:gan_box}) where the discriminator effectively serves as a loss function to the generator, i.e.\ being the generator's only source of information on how to update.
Dually, taking the vantage point of the discriminator, the generator serves as an ever changing source of training data.

\begin{defi}[GAN]
  \label{def:gan}
  Fix three objects $Z, X$ and $L$ in $\Ca$ (respectively called ``the latent space'', ``the data space'' and ``the payoff space'').
  Then given two parametric morphisms
  \[
    (P, g) : \Para(\Ca)(Z, X) \quad \text{and} \quad (Q, d) : \Para(\Ca)(X, L)
  \]
  a \textbf{generative adversarial network} is a morphism $(P \times Q, \GAN_{g, d}) : \Para(\Ca)(Z \times X, L \times L)$ where $\GAN_{g, d}$ is defined as the composite
  \[
    Z \times X \times P \times Q \cong Z \times P \times X \times Q \xrightarrow{g \times X \times \Delta_Q} X \times X \times Q \times Q \cong X \times Q \times X \times Q \xrightarrow{d \times d} L \times L
  \]
\end{defi}

\begin{figure}[h]
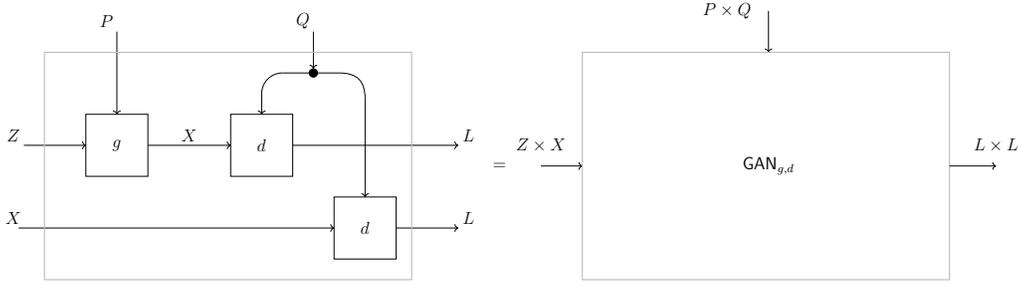

  \scaletikzfig[0.55][1.1]{gan_box}
  \caption{A generative adversarial network as a parametric morphism.}
  \label{fig:gan_box}
\end{figure}

Its string diagram representation is shown in Fig. \ref{fig:gan_box} where we see that a GAN consists of two parallel tracks.
We will see in Ex. \ref{ex:gan_gda_dotproduct} how the first one will be used to process latent vectors, and the second one to process samples from a chosen dataset.
Despite the fact that there are two boxes labeled $d$, they are weight tied, making them behave like a singular unit.

We can easily state what the reverse derivative of $\GAN_{g, d}$ is in terms of its components:
\begin{align}
  \label{ex:gan_reverse_derivative}
  \begin{split}
    R[\GAN_{g, d}](z, x_r, p, q, \alpha_g, \alpha_r) = (z',x_r',p',q_g' + q_r') \quad \text{where} \quad 
      (x_g', q_g') &= R[d](g(z, p), q, \alpha_g)\\
      (x_r', q_r') &= R[d](x_r, q, \alpha_r)\\
      (z', p') &= R[g](z, p, x_g')
    \end{split}
\end{align}

The pair $(\GAN_{g, d}, R[\GAN_{g, d}])$ yields a parametric lens of type $(Z \times X)(Z' \times X') \to (L \times L)(L' \times L')$ (Fig.\ \ref{fig:gan_differentiated}), which we interpret as follows.

It consumes two pieces of data, ``a latent vector'' $z : Z$, a ``real'' sample from the dataset $x_r : X$, in addition to the parameter $p : P$ for the generator and a parameter $q : Q$ for the discriminator.
What happens then are two independent evaluations done by the discriminator.
The first one uses the generator's attempt of producing a sample from the dataset (the latent vector which was fed into it, producing $g(z, p) : X$) as input to the discriminator, producing a payoff $d((g, z, p), q) : L$ for this particular sample.
The second one uses the actual sample from the dataset $x_r$, producing the payoff $d(x_r, q) : L$.

By choosing $\GAN_{g, d}$ as the parametric map representing our supervised learning model, we can differentiate it as in (Fig. \ref{fig:gan_differentiated}), and, with the appropriate choice of a loss function, produce the learning system in the literature called \emph{Wasserstein} GAN (\cite{arjovsky_wasserstein_2017}).

\begin{figure}[h]
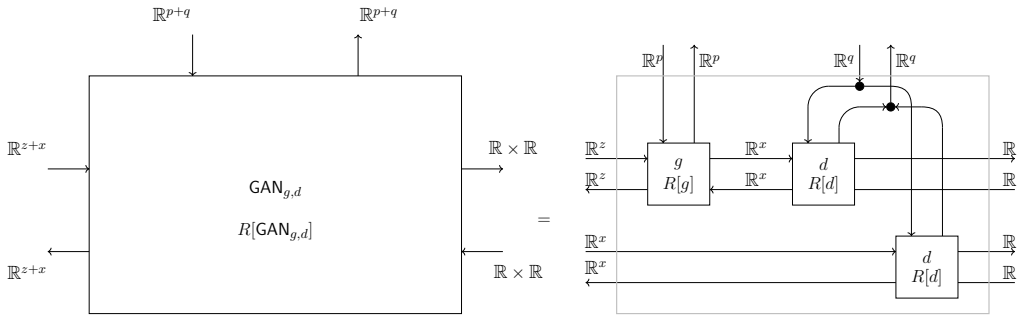

  \scaletikzfig[0.55][1.2]{gan_differentiated}
  \caption{A generative adversarial network under the image of $\Para(\RC)$.}
  \label{fig:gan_differentiated}
\end{figure}

\begin{exa}[GANs, Dot product, GDA]
  \label{ex:gan_gda_dotproduct}
  Fix $\GAN_{g, d}$ as the parametric map (Def. \ref{def:gan}), gradient descent-ascent (Ex. \ref{ex:gradient_descent_ascent}) as the optimiser and dot product (Ex. \ref{ex:dot_product}) as the loss function.
  Then $\putt{}$ becomes a map of type $Z \times X \times P \times Q \times L \times L \to P \times Q$ and its implementation reduces to
  \begin{align*}
    \putt{}(z, x, p, q, \groundtruth) = (p - p', q + q') \qquad \text{where} \quad \qquad (z', x', p', q') &= R[\GAN_{g, d}](z, x, p, q, \alpha \cdot \groundtruth)
  \end{align*}
  We can further unpack the label 
  \begin{align*}
    \putt{}(z, x_r, p, q, {\groundtruth}_g, {\groundtruth}_r) = (p - p', q + q_g' + q_r') \qquad \text{where} \quad \qquad (x_g', q_g') &= R[d](g(z, p), q, \alpha \cdot {\groundtruth}_g)\\
       (z', p') &= R[g](z, p, x_g')\\
       (x_r', q_r') &= R[d](x_r, q, \alpha \cdot {\groundtruth}_r)
  \end{align*}
  This brings us to the last step, where by linearity of the backward pass we can extract $\alpha$ and components of $\groundtruth$ out:
  \begin{align*}
    \putt{}(z, x_r, p, q, {\groundtruth}_g, {\groundtruth}_r) = (p - \alpha {\groundtruth}_g p', q + \alpha({\groundtruth}_g q_g' + {\groundtruth}_r q_r')) \qquad \text{where} \quad (x_g', q_g') &= R[d](g(z, p), q, 1)\\
    (z', p') &= R[g](z, p, x_g')\\
    (x_r', q_r') &= R[d](x_r, q, 1)
  \end{align*}
\end{exa}

The ultimate representation is a form in which it makes it possible to see how the update recovers that of Wasserstein GANs.
The last missing piece is to note that the supervision labels $y_t$ here are effectively ``masks''.
Just like in standard supervised learning an input-output pair $(x_i, y_i)$ consisted of an input value and a corresponding label which guided the direction in which the output $f(x_i, p)$ should've been improved, here the situation is the same.
Given any latent vector $z$ its corresponding ``label'' is the learning signal ${\groundtruth}_g = 1$ which does not change anything in the update, effectively signaling to the generator's and discriminator's optimisers that they should descend (minimizing the assigned loss, making the image more realistic next time), and respectively ascend (maximizing the loss, becoming better at detecting when its input  is a sample generated by the generator).
On the other hand, given any real sample $x_r$ its corresponding ``label'' is the learning signal ${\groundtruth}_r = -1$ which signals to the discriminator's optimiser that it should do the opposite of what it usually does; it should descend, causing it to assign a lower loss value actual samples from the dataset.
In other words, the input-output pairs here are always of the form $((z, x)_i, (1, -1)_i)$, making this GAN in many ways a \emph{constantly} supervised model.
Nonetheless, these different ``forces'' that pull the discriminator in different directions depending on the source of the input, coupled with the ever-changing generated inputs make GANs have intrinsically complex dynamics that are still being studied.

  The fact that we were able to encode Wasserstein GAN in this form in our framework is a consequence of its simple formulation of its loss function, which is effectively given by subtraction \cite[Theorem 3]{arjovsky_wasserstein_2017}.

\subsection{Deep Dreaming: Supervised Learning of Inputs}\label{subsec:deep_dreaming}

We have seen that reparameterising the parameter port with gradient descent allows us to capture supervised parameter learning.
In this section we describe how reparameterising the \textit{input
port} provides us with a way to enhance an input image to elicit a particular
interpretation. This is the idea behind the technique called Deep Dreaming,
appearing in the literature in many forms \cite{DeepDreaming1, DeepDreaming2,
  DeepDreaming3, DeepDreaming4}.

\begin{equation}
  {\scalefont{0.75}
    \begin{tikzpicture}
	\begin{pgfonlayer}{nodelayer}
		\node [style=none] (20) at (-3.5, 1) {};
		\node [style=none] (21) at (-3.5, -1) {};
		\node [style=none] (22) at (-1, 1) {};
		\node [style=none] (23) at (-1, -1) {};
		\node [style=none] (26) at (-4.5, 1) {$A$};
		\node [style=none] (27) at (-4.5, -1) {$A'$};
		\node [style=none] (28) at (0, 1) {$B$};
		\node [style=none] (29) at (0, -1) {$B'$};
		\node [style=none] (39) at (-4.5, 0.5) {};
		\node [style=none] (40) at (-3.5, 0.5) {};
		\node [style=none] (43) at (-3.5, -0.5) {};
		\node [style=none] (44) at (-4.5, -0.5) {};
		\node [style=none] (45) at (-1, 0.5) {};
		\node [style=none] (46) at (0, 0.5) {};
		\node [style=none] (57) at (0, -0.5) {};
		\node [style=none] (58) at (-1, -0.5) {};
		\node [style=none] (60) at (-7.25, 1) {};
		\node [style=none] (61) at (-6.25, 1) {};
		\node [style=none] (62) at (-7.25, 2) {};
		\node [style=none] (63) at (-6.25, 2) {};
		\node [style=none] (64) at (-7.75, 2) {$A$};
		\node [style=none] (65) at (-5.75, 2) {$A'$};
		\node [style=none] (66) at (1, 1) {};
		\node [style=none] (67) at (1, -1) {};
		\node [style=none] (68) at (3.5, 1) {};
		\node [style=none] (69) at (3.5, -1) {};
		\node [style=none] (72) at (4.5, 1) {$L$};
		\node [style=none] (73) at (4.5, -1) {$L'$};
		\node [style=none] (74) at (0, 0.5) {};
		\node [style=none] (75) at (1, 0.5) {};
		\node [style=none] (76) at (1, -0.5) {};
		\node [style=none] (77) at (0, -0.5) {};
		\node [style=none] (78) at (3.5, 0.5) {};
		\node [style=none] (79) at (4.5, 0.5) {};
		\node [style=none] (80) at (4.5, -0.5) {};
		\node [style=none] (81) at (3.5, -0.5) {};
		\node [style=none] (82) at (1.75, 1) {};
		\node [style=none] (83) at (2.75, 1) {};
		\node [style=none] (84) at (1.75, 6) {};
		\node [style=node] (85) at (2.75, 2.5) {};
		\node [style=none] (86) at (1.75, 6.75) {$B$};
		\node [style=none] (87) at (3.25, 1.5) {$B'$};
		\node [style=none] (89) at (2.25, 0) {\textsf{Loss}};
		\node [style=none] (92) at (5.5, 1) {};
		\node [style=none] (93) at (5.5, -1) {};
		\node [style=none] (94) at (7.5, -1) {};
		\node [style=none] (95) at (8.5, 0) {};
		\node [style=none] (96) at (7.5, 1) {};
		\node [style=none] (97) at (5.5, -0.5) {};
		\node [style=none] (98) at (5.5, 0.5) {};
		\node [style=none] (99) at (6.75, 0) {$\alpha$};
		\node [style=none] (100) at (-8, 3) {};
		\node [style=none] (101) at (-5.5, 3) {};
		\node [style=none] (102) at (-5.5, 5) {};
		\node [style=none] (103) at (-8, 5) {};
		\node [style=none] (105) at (-7.25, 3) {};
		\node [style=none] (106) at (-6.25, 3) {};
		\node [style=none] (108) at (-7.25, 5) {};
		\node [style=none] (110) at (-6.25, 6) {};
		\node [style=none] (111) at (-7.25, 6) {};
		\node [style=none] (120) at (-8, 1) {};
		\node [style=none] (121) at (-8, -1) {};
		\node [style=none] (122) at (-5.5, 1) {};
		\node [style=none] (123) at (-5.5, -1) {};
		\node [style=none] (124) at (-8, 0.5) {};
		\node [style=none] (125) at (-8, -0.5) {};
		\node [style=none] (126) at (-5.5, 0.5) {};
		\node [style=none] (127) at (-5.5, -0.5) {};
		\node [style=none] (128) at (-7.25, 1) {};
		\node [style=none] (129) at (-6.25, 1) {};
		\node [style=none] (130) at (-2.25, 0) {\textsf{Model}};
		\node [style=none] (131) at (-6.75, 0.5) {};
		\node [style=none] (132) at (-6.25, 0) {};
		\node [style=none] (133) at (-5.75, -0.5) {};
		\node [style=none] (136) at (-7.75, 6.75) {$S \times A$};
		\node [style=none] (137) at (-5.75, 6.75) {$S \times A$};
		\node [style=none] (138) at (-2.75, 6) {};
		\node [style=none] (139) at (-2.75, 6.75) {$P$};
		\node [style=none] (140) at (-2.75, 1) {};
		\node [style=none] (141) at (-1.75, 1) {};
		\node [style=node] (142) at (-1.75, 2.5) {};
		\node [style=none] (143) at (-6.25, 5) {};
		\node [style=none] (144) at (-6.75, 4) {\textsf{Optimiser}};
	\end{pgfonlayer}
	\begin{pgfonlayer}{edgelayer}
		\draw (20.center) to (22.center);
		\draw (22.center) to (23.center);
		\draw (23.center) to (21.center);
		\draw (21.center) to (20.center);
		\draw [style=pointy] (39.center) to (40.center);
		\draw (43.center) to (44.center);
		\draw (45.center) to (46.center);
		\draw [style=pointy] (57.center) to (58.center);
		\draw [style=pointy] (62.center) to (60.center);
		\draw (61.center) to (63.center);
		\draw (66.center) to (68.center);
		\draw (68.center) to (69.center);
		\draw (69.center) to (67.center);
		\draw (67.center) to (66.center);
		\draw [style=pointy] (74.center) to (75.center);
		\draw (76.center) to (77.center);
		\draw (78.center) to (79.center);
		\draw [style=pointy] (80.center) to (81.center);
		\draw [style=pointy] (84.center) to (82.center);
		\draw (83.center) to (85);
		\draw (92.center) to (96.center);
		\draw (96.center) to (95.center);
		\draw (95.center) to (94.center);
		\draw (94.center) to (93.center);
		\draw (93.center) to (92.center);
		\draw [style=pointy] (79.center) to (98.center);
		\draw (97.center) to (80.center);
		\draw (103.center) to (102.center);
		\draw (102.center) to (101.center);
		\draw (101.center) to (100.center);
		\draw (100.center) to (103.center);
		\draw (62.center) to (105.center);
		\draw [style=pointy] (63.center) to (106.center);
		\draw [style=pointy] (111.center) to (108.center);
		\draw [style=lightnone] (120.center) to (122.center);
		\draw [style=lightnone] (122.center) to (123.center);
		\draw [style=lightnone] (123.center) to (121.center);
		\draw [style=lightnone] (121.center) to (120.center);
		\draw (126.center) to (39.center);
		\draw [style=pointy] (44.center) to (127.center);
		\draw (126.center) to (131.center);
		\draw [in=-90, out=-180, looseness=1.25] (131.center) to (128.center);
		\draw [bend right=315, looseness=1.25] (133.center) to (132.center);
		\draw (127.center) to (133.center);
		\draw (132.center) to (129.center);
		\draw (141.center) to (142);
		\draw [style=pointy] (138.center) to (140.center);
		\draw [style=pointy] (143.center) to (110.center);
	\end{pgfonlayer}
\end{tikzpicture}
  }  \label{eq:deep_dreaming}
\end{equation}

Deep dreaming is a technique which uses the parameters $p$ of some trained
classifier network to iteratively dream up, or amplify some features of a class $b$ on a chosen input $a$.
For example, if we start with an image of a landscape $a_0$, a label $b$ of a ``cat''
and a parameter $p$ of a sufficiently well-trained classifier, we can start performing
``learning'' as usual: computing the predicted class for the landscape $a_0$
for the network with parameters $p$, and then computing the distance between the
prediction and our label of a cat $b$. When performing backpropagation, the
respective changes computed for each layer tell us how the activations of that
layer should have been changed to be more ``cat'' like. This includes the first
(input) layer of the landscape $a_0$. Usually, we discard this changes and
apply gradient update to the parameters.
In deep dreaming we \textit{discard the parameters} and \textit{apply gradient update to the input} (see~\eqref{eq:deep_dreaming}).
Gradient update here takes these changes and computes a new image $a_1$ which
is the same image of the landscape, but changed slightly so to look more like
whatever the network thinks a cat looks like.
This is the essence of deep dreaming, where iteration of this process allows
networks to dream up features and shapes on a particular chosen image
\cite{GoogleInceptionism}.

Just like in the previous subsection, we can write this deep dreaming system as
a map in $\Para(\Lens(\Ca))$ from $\diset{1}{1}$ to $\diset{1}{1}$ whose parameter space is
$\diset{S \times A \times P \times B}{S \times A}$. In other words, this is a lens of
type $\diset{S \times A \times P \times B}{S \times A} \to \diset{1}{1}$ whose $\gett{}$ map is trivial. Its $\putt{}$ map has the type $S \times A \times P \times B \to S \times A$ and unpacks to 
\begin{align*}
  \putt{}(s, a, p, b_t) = U^*(s, a, a') \qquad \texttt{ where} \quad \qquad \overline{a}  &= U(s, a) \\
  b_p &= f(p, \overline{a}) \\
  (b'_t, b'_p) &= R[\loss](b_t, b_p, \alpha(\loss(b_t, b_p))) \\
  (p', a') &= R[f](p, \overline{a}, b'_p)
\end{align*}

We note that deep dreaming is usually presented without any loss function as a
maximisation of a particular activation in the last layer of the network output
\cite[Section 2.]{DeepDreaming4}. This maximisation is done with gradient
ascent, as opposed to gradient descent.
However, this is just a special case of our framework where the loss function is
the dot product (Example \ref{ex:dot_product}). The choice of the particular
activation is encoded as a one-hot vector, and the loss function in that case essentially masks the network output, leaving active only the particular chosen activation. The final component is the gradient \textit{ascent}: this is simply recovered by choosing
a positive, instead of a negative learning rate \cite{DeepDreaming4}. We
explicitly unpack this in the following example.

\begin{exa}[Deep dreaming, dot product loss, basic gradient update] Fix $\Smooth$ as base category, a parametric map $(\Rb^p, f) :
\Para(\Smooth)(\Rb^a, \Rb^b)$, the dot product loss (Example \ref{ex:dot_product}), basic gradient
update (Example \ref{ex:grad_desc_lens}), and a \textit{positive} learning rate
$\alpha : \Rb$. Then the above put map simplifies.
Since there is no state, its type reduces to $A \times P \times B \to A$ and its
implementation to 
\begin{align*}
  \putt{}(a, p, b_t) = a + a' &\qquad \qquad \qquad \qquad \texttt{  where }
  (p', a') = R[f](p, a, \alpha \cdot b_t).
\end{align*}
Like in Example \ref{ex:quad_err_grad_desc}, this update can be rewritten as
\[
\putt{}(a, p, b_t) = a + \alpha \cdot (R[f](p, a, b_t);\pi_1)
\]
making a few things apparent. This update does not depend on the prediction
$f(p, a)$: no matter what the network has predicted, the goal is always to
maximise particular activations. Which activations? The ones chosen by $b_t$.
When $b_t$ is a one-hot vector, this picks out the activation of just one class
to maximise, which is often done in practice.
\end{exa}

While we present only the most basic image, there is plenty of room left for
exploration. The work of \cite[Section 2.]{DeepDreaming4} adds an extra
regularisation term to the image. In general, the neural network $f$ is sometimes changed to copy a number of internal activations which are then exposed on the output layer. Maximising all these activations often produces more visually appealing results.
In the literature we did not find an example which uses the Softmax-cross
entropy (Example \ref{ex:softmax_ce}) as a loss function in deep dreaming, which
seems like the more natural choice in this setting. Furthermore, while deep
dreaming commonly uses basic gradient descent, there is nothing
preventing the use of any of the optimiser lenses
discussed in the previous section, or even doing deep dreaming in the context of
Boolean circuits. Lastly, learning iteration which was described in at the end
of previous subsection can be modelled here in an analogous way.

\section{Implementation}
\label{section:implementation}

We provide a proof-of-concept implementation
as a Python library --- full usage examples, source code, and experiments can be  found at \cite{python-library}.
We demonstrate the correctness of our library empirically using a number of
experiments implemented both in our library and in Keras \cite{Keras}, a popular
framework for deep learning.
For example, one experiment is a model for the MNIST image classification
problem \cite{mnist}: we implement the same model in both frameworks and achieve
comparable accuracy.
Note that despite similarities between the user interfaces of our library and of
Keras, a model in our framework is constructed
as a composition of parametric lenses. 
This is fundamentally different to the approach taken by Keras and other
existing libraries,
and highlights how our proposed algebraic structures naturally 
guide programming practice

In summary, our implementation demonstrates the advantages of our approach.
Firstly, computing the gradients of the network is greatly simplified
through the use of lens composition.
Secondly, model architectures can be expressed in a principled, mathematical
language; as morphisms of a monoidal category.
Finally, the modularity of our approach makes it easy to see how various aspects
of training can be modified: for example, one can define a new optimization
algorithm simply by defining an appropriate lens. We now give a brief sketch of our implementation.

\subsection{Constructing a Model with Lens and Para}

We model a lens $\diset{f}{f^*}$ in our library with the \texttt{Lens} class, which
consists of a pair of maps \texttt{fwd} and \texttt{rev} corresponding to $f$
and $f^*$, respectively.
For example, we write the identity lens $\diset{1_A}{\pi_2}$ as follows:

\begin{lstlisting}[
  %caption=Identity Lens,
  label=lst:identity-lens,
]
identity = Lens(lambda x: x, lambda x_dy: x_dy[1])
\end{lstlisting}

The composition (in diagrammatic order) of \texttt{Lens} values \texttt{f} and
\texttt{g} is written \texttt{f >> g}, and monoidal composition as \texttt{f @
g}.
Similarly, the type of $\Para$ maps is modeled by the \texttt{Para} class, with
composition and monoidal product written the same way.
Our library provides several primitive \texttt{Lens} and \texttt{Para} values.

Let us now see how to construct a single layer neural network from the
composition of such primitives.
Diagrammatically, we will construct a model consisting of a single dense layer,
as in Example \ref{example:dense-layer} and below.

\[ \scalebox{0.8}{\begin{tikzpicture}
	\begin{pgfonlayer}{nodelayer}
		\node [style=none] (20) at (-3, 1) {};
		\node [style=none] (21) at (-3, -1) {};
		\node [style=none] (22) at (-1, 1) {};
		\node [style=none] (23) at (-1, -1) {};
		\node [style=none] (26) at (-4.5, 0.5) {$\Rb^a$};
		\node [style=none] (27) at (-4.5, -0.5) {$\Rb^a$};
		\node [style=none] (28) at (0, 1) {$\Rb^b$};
		\node [style=none] (39) at (-4, 0.5) {};
		\node [style=none] (40) at (-3, 0.5) {};
		\node [style=none] (43) at (-3, -0.5) {};
		\node [style=none] (44) at (-4, -0.5) {};
		\node [style=none] (45) at (-1, 0.5) {};
		\node [style=none] (46) at (0, 0.5) {};
		\node [style=none] (57) at (0, -0.5) {};
		\node [style=none] (58) at (-1, -0.5) {};
		\node [style=none] (60) at (-2.5, 1) {};
		\node [style=none] (61) at (-1.5, 1) {};
		\node [style=none] (62) at (-2.5, 2) {};
		\node [style=none] (63) at (-1.5, 2) {};
		\node [style=none] (64) at (-2.75, 2.5) {$\Rb^{b \times a}$};
		\node [style=none] (65) at (-1, 2.5) {$\Rb^{b \times a}$};
		\node [style=none] (66) at (1, 1) {};
		\node [style=none] (67) at (1, -1) {};
		\node [style=none] (68) at (3, 1) {};
		\node [style=none] (69) at (3, -1) {};
		\node [style=none] (74) at (0, 0.5) {};
		\node [style=none] (75) at (1, 0.5) {};
		\node [style=none] (76) at (1, -0.5) {};
		\node [style=none] (77) at (0, -0.5) {};
		\node [style=none] (78) at (3, 0.5) {};
		\node [style=none] (79) at (4, 0.5) {};
		\node [style=none] (80) at (4, -0.5) {};
		\node [style=none] (81) at (3, -0.5) {};
		\node [style=none] (82) at (1.5, 1) {};
		\node [style=none] (83) at (2.5, 1) {};
		\node [style=none] (84) at (1.5, 2) {};
		\node [style=none] (85) at (2.5, 2) {};
		\node [style=none] (86) at (1.5, 2.5) {$\Rb^b$};
		\node [style=none] (87) at (2.75, 2.5) {$\Rb^b$};
		\node [style=none] (88) at (-2, 0) {\textsf{linear}};
		\node [style=none] (89) at (2, 0) {\textsf{bias}};
		\node [style=none] (92) at (4, 0.5) {};
		\node [style=none] (93) at (4, -0.5) {};
		\node [style=none] (94) at (5, 1) {};
		\node [style=none] (95) at (5, -1) {};
		\node [style=none] (96) at (8, 1) {};
		\node [style=none] (97) at (8, -1) {};
		\node [style=none] (100) at (4, 0.5) {};
		\node [style=none] (101) at (5, 0.5) {};
		\node [style=none] (102) at (5, -0.5) {};
		\node [style=none] (103) at (4, -0.5) {};
		\node [style=none] (104) at (8, 0.5) {};
		\node [style=none] (105) at (9, 0.5) {};
		\node [style=none] (106) at (9, -0.5) {};
		\node [style=none] (107) at (8, -0.5) {};
		\node [style=none] (108) at (5.5, 1) {};
		\node [style=none] (109) at (7.5, 1) {};
		\node [style=none] (114) at (6.5, 0) {\textsf{activation}};
		\node [style=none] (115) at (0, -1) {$\Rb^b$};
		\node [style=none] (116) at (4, 1) {$\Rb^b$};
		\node [style=none] (117) at (4, -1) {$\Rb^b$};
		\node [style=none] (118) at (9.75, 0.5) {$\Rb^b$};
		\node [style=none] (119) at (9.75, -0.5) {$\Rb^b$};
	\end{pgfonlayer}
	\begin{pgfonlayer}{edgelayer}
		\draw (20.center) to (22.center);
		\draw (22.center) to (23.center);
		\draw (23.center) to (21.center);
		\draw (21.center) to (20.center);
		\draw [style=pointy] (39.center) to (40.center);
		\draw [style=pointy] (43.center) to (44.center);
		\draw (45.center) to (46.center);
		\draw [style=pointy] (57.center) to (58.center);
		\draw [style=pointy] (62.center) to (60.center);
		\draw [style=pointy] (61.center) to (63.center);
		\draw (66.center) to (68.center);
		\draw (68.center) to (69.center);
		\draw (69.center) to (67.center);
		\draw (67.center) to (66.center);
		\draw [style=pointy] (74.center) to (75.center);
		\draw (76.center) to (77.center);
		\draw [style=pointy] (78.center) to (79.center);
		\draw [style=pointy] (80.center) to (81.center);
		\draw [style=pointy] (84.center) to (82.center);
		\draw [style=pointy] (83.center) to (85.center);
		\draw (94.center) to (96.center);
		\draw (96.center) to (97.center);
		\draw (97.center) to (95.center);
		\draw (95.center) to (94.center);
		\draw [style=pointy] (100.center) to (101.center);
		\draw (102.center) to (103.center);
		\draw [style=pointy] (104.center) to (105.center);
		\draw [style=pointy] (106.center) to (107.center);
	\end{pgfonlayer}
\end{tikzpicture}} \]

Recall that the parameters of \texttt{linear} are the coefficients of a $b \times a$ matrix,
and the underlying lens has as its forward map the function $(M, x) \to M \cdot x$, where $M$ is the $b \times a$ matrix whose coefficients are the
$\Rb^{b \times a}$ parameters, and $x \in \Rb^a$ is the input vector.  The
\texttt{bias} map is even simpler: the forward map of the underlying lens is
simply pointwise addition of inputs and parameters: $(b, x) \to b + x$.
Finally, the \texttt{activation} map simply applies a nonlinear function (e.g., $\textsf{sigmoid}$) to the input,
and thus has the trivial (unit) parameter space.
The representation of this composition in code is straightforward: we can simply
compose the three primitive \texttt{Para} maps as in \eqref{eq:paralens-dense-layer}:

\begin{lstlisting}[
  %caption={The \texttt{dense} lens},
  label=lst:dense-lens,
]
def dense(a, b, activation):
  return linear(a, b) >> bias(b) >> activation
\end{lstlisting}

Note that by constructing model architectures in this way, the computation of
reverse derivatives is greatly simplified: we obtain the reverse
derivative `for free' as the $\putt$ map of the model.
Furthermore, adding new primitives is also simplified: the user need simply
provide a function and its reverse derivative in the form of a \texttt{Para}
map.
Finally, notice also that our approach is truly compositional: we can define a
hidden layer neural network with $n$ hidden units simply by composing two dense
layers, as follows:
\begin{lstlisting}[
  %caption={The \texttt{dense} lens},
  label=lst:dense-lenss,
]
dense(a, n, activation) >> dense(n, b, activation)
\end{lstlisting}

\subsection{Learning}

Now that we have constructed a model, we also need to use it to \emph{learn} from data.
Concretely, we will construct a full parametric lens as in
Figure ~\ref{fig:roadmap}
then extract its $\putt$ map to iterate over the dataset.

By way of example, let us see how to construct the following parametric lens,
representing basic gradient descent over a single layer neural network with a
fixed learning rate:

\begin{equation}\label{eq:paralens-implementation-full-learner}
   \scalebox{0.8}{\begin{tikzpicture}
	\begin{pgfonlayer}{nodelayer}
		\node [style=none] (20) at (-3.5, 1) {};
		\node [style=none] (21) at (-3.5, -1) {};
		\node [style=none] (22) at (-1, 1) {};
		\node [style=none] (23) at (-1, -1) {};
		\node [style=none] (26) at (-4.5, 1) {$A$};
		\node [style=none] (27) at (-4.5, -1) {$A'$};
		\node [style=none] (28) at (0, 1) {$B$};
		\node [style=none] (29) at (0, -1) {$B'$};
		\node [style=none] (39) at (-4.5, 0.5) {};
		\node [style=none] (40) at (-3.5, 0.5) {};
		\node [style=none] (43) at (-3.5, -0.5) {};
		\node [style=none] (44) at (-4.5, -0.5) {};
		\node [style=none] (45) at (-1, 0.5) {};
		\node [style=none] (46) at (0, 0.5) {};
		\node [style=none] (57) at (0, -0.5) {};
		\node [style=none] (58) at (-1, -0.5) {};
		\node [style=none] (60) at (-2.75, 1) {};
		\node [style=none] (61) at (-1.75, 1) {};
		\node [style=none] (62) at (-2.75, 2) {};
		\node [style=none] (63) at (-1.75, 2) {};
		\node [style=none] (64) at (-3.25, 2) {$P$};
		\node [style=none] (65) at (-1.25, 2) {$P'$};
		\node [style=none] (66) at (1, 1) {};
		\node [style=none] (67) at (1, -1) {};
		\node [style=none] (68) at (3.5, 1) {};
		\node [style=none] (69) at (3.5, -1) {};
		\node [style=none] (72) at (4.5, 1) {$L$};
		\node [style=none] (73) at (4.5, -1) {$L'$};
		\node [style=none] (74) at (0, 0.5) {};
		\node [style=none] (75) at (1, 0.5) {};
		\node [style=none] (76) at (1, -0.5) {};
		\node [style=none] (77) at (0, -0.5) {};
		\node [style=none] (78) at (3.5, 0.5) {};
		\node [style=none] (79) at (4.5, 0.5) {};
		\node [style=none] (80) at (4.5, -0.5) {};
		\node [style=none] (81) at (3.5, -0.5) {};
		\node [style=none] (82) at (1.75, 1) {};
		\node [style=none] (83) at (2.75, 1) {};
		\node [style=none] (84) at (1.75, 6) {};
		\node [style=node] (85) at (2.75, 2.5) {};
		\node [style=none] (86) at (1.75, 6.5) {$B$};
		\node [style=none] (87) at (3.25, 1.5) {$B'$};
		\node [style=none] (89) at (2.25, 0) {$\mathsf{loss}$};
		\node [style=none] (92) at (5.5, 1) {};
		\node [style=none] (93) at (5.5, -1) {};
		\node [style=none] (94) at (7.5, -1) {};
		\node [style=none] (95) at (8.5, 0) {};
		\node [style=none] (96) at (7.5, 1) {};
		\node [style=none] (97) at (5.5, -0.5) {};
		\node [style=none] (98) at (5.5, 0.5) {};
		\node [style=none] (99) at (6.25, -0.5) {$\epsilon$};
		\node [style=none] (100) at (-3.5, 3) {};
		\node [style=none] (101) at (-1, 3) {};
		\node [style=none] (102) at (-1, 5) {};
		\node [style=none] (103) at (-3.5, 5) {};
		\node [style=none] (105) at (-2.75, 3) {};
		\node [style=none] (106) at (-1.75, 3) {};
		\node [style=none] (107) at (-2.75, 4) {};
		\node [style=none] (108) at (-2.75, 5) {};
		\node [style=none] (110) at (-1.75, 6) {};
		\node [style=none] (111) at (-2.75, 6) {};
		\node [style=none] (112) at (-2, 4.25) {};
		\node [style=none] (113) at (-2, 3.75) {};
		\node [style=none] (114) at (-1.75, 3.75) {};
		\node [style=none] (115) at (-1.5, 3.75) {};
		\node [style=none] (116) at (-1.5, 4.25) {};
		\node [style=none] (117) at (-1.75, 4.25) {};
		\node [style=none] (118) at (-2, 4) {};
		\node [style=none] (119) at (-1.75, 4) {$+$};
		\node [style=none] (120) at (-8, 1) {};
		\node [style=none] (121) at (-8, -1) {};
		\node [style=none] (122) at (-5.5, 1) {};
		\node [style=none] (123) at (-5.5, -1) {};
		\node [style=none] (124) at (-8, 0.5) {};
		\node [style=none] (125) at (-8, -0.5) {};
		\node [style=none] (126) at (-5.5, 0.5) {};
		\node [style=none] (127) at (-5.5, -0.5) {};
		\node [style=none] (128) at (-7.25, 1) {};
		\node [style=none] (129) at (-6.25, 1) {};
		\node [style=none] (130) at (-2.25, 0) {$\mathsf{dense}$};
		\node [style=none] (131) at (-6.75, 0.5) {};
		\node [style=none] (132) at (-6.25, 0) {};
		\node [style=none] (133) at (-5.75, -0.5) {};
		\node [style=none] (134) at (-7.25, 6) {};
		\node [style=node] (135) at (-6.25, 2.5) {};
		\node [style=none] (136) at (-2.75, 6.75) {$P$};
		\node [style=none] (137) at (-1.75, 6.75) {$P$};
		\node [style=none] (138) at (-7.25, 6.75) {$A$};
		\node [style=node] (139) at (6.5, 0.5) {};
		\node [style=none] (140) at (6, -0.25) {};
		\node [style=none] (141) at (6, -0.75) {};
		\node [style=none] (142) at (6.5, -0.75) {};
		\node [style=none] (143) at (6.5, -0.25) {};
		\node [style=none] (144) at (6.75, -0.5) {};
		\node [style=none] (145) at (6, -0.5) {};
	\end{pgfonlayer}
	\begin{pgfonlayer}{edgelayer}
		\draw (20.center) to (22.center);
		\draw (22.center) to (23.center);
		\draw (23.center) to (21.center);
		\draw (21.center) to (20.center);
		\draw [style=pointy] (39.center) to (40.center);
		\draw (43.center) to (44.center);
		\draw (45.center) to (46.center);
		\draw [style=pointy] (57.center) to (58.center);
		\draw [style=pointy] (62.center) to (60.center);
		\draw (61.center) to (63.center);
		\draw (66.center) to (68.center);
		\draw (68.center) to (69.center);
		\draw (69.center) to (67.center);
		\draw (67.center) to (66.center);
		\draw [style=pointy] (74.center) to (75.center);
		\draw (76.center) to (77.center);
		\draw (78.center) to (79.center);
		\draw [style=pointy] (80.center) to (81.center);
		\draw [style=pointy] (84.center) to (82.center);
		\draw (83.center) to (85);
		\draw (92.center) to (96.center);
		\draw (96.center) to (95.center);
		\draw (95.center) to (94.center);
		\draw (94.center) to (93.center);
		\draw (93.center) to (92.center);
		\draw [style=pointy] (79.center) to (98.center);
		\draw [style=pointy] (97.center) to (80.center);
		\draw (103.center) to (102.center);
		\draw (102.center) to (101.center);
		\draw (101.center) to (100.center);
		\draw (100.center) to (103.center);
		\draw (62.center) to (105.center);
		\draw (63.center) to (106.center);
		\draw (111.center) to (108.center);
		\draw (108.center) to (107.center);
		\draw (107.center) to (105.center);
		\draw [style=pointy] (107.center) to (118.center);
		\draw [style=pointy] (117.center) to (110.center);
		\draw [style=pointy] (106.center) to (114.center);
		\draw (113.center) to (112.center);
		\draw (112.center) to (116.center);
		\draw (116.center) to (115.center);
		\draw (115.center) to (113.center);
		\draw [style=lightnone] (120.center) to (122.center);
		\draw [style=lightnone] (122.center) to (123.center);
		\draw [style=lightnone] (123.center) to (121.center);
		\draw [style=lightnone] (121.center) to (120.center);
		\draw (126.center) to (39.center);
		\draw [style=pointy] (44.center) to (127.center);
		\draw (126.center) to (131.center);
		\draw [in=-90, out=-180, looseness=1.25] (131.center) to (128.center);
		\draw [bend right=315, looseness=1.25] (133.center) to (132.center);
		\draw (127.center) to (133.center);
		\draw (132.center) to (129.center);
		\draw [style=pointy] (134.center) to (128.center);
		\draw (129.center) to (135);
		\draw (98.center) to (139);
		\draw (143.center) to (144.center);
		\draw (144.center) to (142.center);
		\draw (142.center) to (141.center);
		\draw (141.center) to (140.center);
		\draw (140.center) to (143.center);
		\draw (97.center) to (145.center);
	\end{pgfonlayer}
\end{tikzpicture}}
\end{equation}

This morphism is constructed essentially as below, where
\texttt{apply\_update($\alpha$, $f$)} represents the `vertical stacking' of
$\alpha$ atop $f$:

\begin{lstlisting}[mathescape]
apply_update(basic_update, dense) >> loss >> learning_rate($\epsilon$)
\end{lstlisting}

Now, given the parametric lens of \eqref{eq:paralens-implementation-full-learner},
one can construct a morphism $\textsf{step} : B \times P \times A \to P$ which is simply the put map of the lens.
Training the model then consists of iterating the $\textsf{step}$ function over
dataset examples $(x, y) \in A \times B$ to optimise some initial choice of
parameters $\theta_0 \in P$, by letting
$\theta_{i+1} = \mathsf{step}(y_i, \theta_i, x_i)$.

Note that our library also provides a utility function to construct $\mathsf{step}$ from its various pieces:
\begin{lstlisting}[mathescape]
step = supervised_step(model, update, loss, learning_rate)
\end{lstlisting}

For an end-to-end example of model training and iteration, we refer the
interested reader to the experiments accompanying the code \cite{python-library}.


\section{Related Work}
\label{section:related-work}

The work \cite{BackpropAsFunctor} is closely related to ours, in that it
provides an abstract categorical model of backpropagation. However, it differs in
a number of key aspects.
We give a complete lens-theoretic explanation of {\em what} is
back-propagated via (i) the use of CRDCs to model gradients; and (ii) the $\Para$ construction to model parametric functions and
parameter update.
We thus can go well beyond \cite{BackpropAsFunctor} in terms of examples - their example of
smooth functions and basic gradient descent is covered in our subsection
\ref{subsec:learning-parameters}.

We also explain some of the constructions of \cite{BackpropAsFunctor} in a more
structured way.  For example, rather than considering the category $\Learn$ of
\cite{BackpropAsFunctor} as primitive, here we construct it as a composite of two
more basic constructions (the $\Para$ and $\Lens$ constructions).
The flexibility could be used, for example, to compositionally replace $\Para$
with a variant allowing parameters to come from a different category, or lenses
with the category of optics \cite{Optics} enabling us to model things such as control flow using prisms. 

One more relevant aspect is functoriality. We use a functor to augment a
parametric map with its backward pass, just like \cite{BackpropAsFunctor}.
However, they additionally augmented this map with a loss map and gradient
descent using a functor as well. This added extra conditions on the partial
derivatives of the loss function: it needed to be invertible in the 2nd
variable. This constraint was not justified in \cite{BackpropAsFunctor}, nor is
it a constraint that appears in machine learning
practice. This led us to reexamine their constructions, coming up with our
reformulation that does not require it.
While loss maps and optimisers are mentioned in \cite{BackpropAsFunctor} as parts of the aforementioned functor, here they are extracted out and play a key role: loss maps are parametric lenses and optimisers are reparameterisations.
Thus, in this paper we instead use $\Para$-composition
to add the loss map to the model, and $\Para$ 2-cells to add optimisers.
The mentioned inverse of the partial derivative of the loss map in the 2$^{nd}$ variable was also hypothesised to be relevant to deep dreaming. We have investigated this possibility thoroughly in our paper, showing it is gradient update which is used to dream
up pictures. We also correct a small issue in Theorem III.2 of \cite{BackpropAsFunctor}.
There, the morphisms of $\Learn$ were defined up to an equivalence (pg. 4 of
\cite{BackpropAsFunctor}) but, unfortunately, the functor defined in Theorem
III.2 does not respect this equivalence relation. Our approach instead uses
2-cells which comes from the universal property of $\Para$ ---  a 2-cell from
$(P,f): A \to B$ to $(Q,g): A \to B$ is a lens, and hence has two components: a
map $\alpha: Q \to P$ and $\alpha^*: Q \times P \to Q$.  By comparison, we can
see the equivalence relation of \cite{BackpropAsFunctor} as being induced by
map $\alpha: Q \to P$, and not a lens.  Our approach highlights the importance
of the 2-categorical structure of learners. In addition, it does not treat the
functor $\Para(\Ca) \to \Learn$ as a primitive. In our case, this functor has
the type $\Para(\Ca) \to \Para(\Lens(\Ca))$ and arises from applying $\Para$ to a
canonical functor $\Ca \to \Lens(\Ca)$ existing for \textit{any} reverse derivative
category, not just $\Smooth$.
Lastly, in our paper we took advantage of the graphical calculus for
$\Para$, redrawing many diagrams appearing in \cite{BackpropAsFunctor} in a structured way.

Other than \cite{BackpropAsFunctor}, there are a few more relevant papers.
The work of \cite{Dioptics} contains a sketch of some of the ideas this paper
evolved from. They are based on the interplay of optics with parameterisation,
albeit framed in the setting of diffeological spaces, and requiring cartesian and
local cartesian closed structure on the base category.
Lenses and Learners are studied in the eponymous work of
\cite{LensesAndLearners} which observes that learners are parametric lenses. They do not explore any of the
relevant $\Para$ or CRDC structure, but make the distinction between \textit{symmetric}
and \textit{asymmetric lenses}, studying how they are related to learners
defined in \cite{BackpropAsFunctor}. A lens-like implementation of automatic differentiation is the focus of
\cite{SimpleAD}, but learning algorithms aren't studied.
A relationship between category-theoretic perspective on probabilistic modeling
and gradient-based optimisation is studied in \cite{CatStochLik} which also
studies a variant of the $\Para$ construction.
Usage of Cartesian
differential categories to study learning is found in \cite{DelayedTrace}. They
extend the differential operator to work on stateful maps, but do not study
lenses, parameterisation nor update maps.
The work of \cite{CompDL} studies deep learning in the context of
Cycle-consistent Generative Adversarial Networks \cite{CycleGAN} and formalises
it via free and quotient categories, making parallels to the categorical
formulations of database theory \cite{FunctorialDataMigration}. They do use the
$\Para$ construction, but do not relate it to lenses nor reverse derivative categories. A general survey of category theoretic approaches to machine learning, covering
many of the above papers, can be found in \cite{CategoryTheoryMachineLearning}. Lastly, the concept of parametric lenses has started appearing in recent formulations
of categorical game theory and cybernetics \cite{TowardsCatCyber, ExtensiveFormGamesAgency}. The work of \cite{TowardsCatCyber} generalises the study of parametric lenses
into parametric optics and connects it to game thereotic concepts such as
Nash equilibria.


\section{Conclusions and Future Directions}
\label{section:conclusions}


We have given a categorical foundation of gradient-based learning algorithms which achieves a number of important goals.  The foundation is principled and mathematically clean, based on the fundamental idea of a \emph{parametric lens}.  The foundation covers a wide variety of examples: different optimisers and loss maps in gradient-based learning, different architectures and layer structures, different settings where gradient-based learning happens (smooth functions vs. boolean circuits), adversarial unsupervised learning, and  both learning of parameters and learning of inputs (deep dreaming).  Finally, the foundation is more than a mere abstraction: we have also shown how it can be used to give a practical implementation of learning, as discussed in Section \ref{section:implementation}.

There are a number of important directions which are possible to explore because of this work. 
One of the most exciting ones is a more comprehensive study of neural network architectures through the category-theoretic perspective.
Neural network architectures have begun to be studied using category theory adjacent machinery in the context of \emph{Geometric Deep Learning} (\cite{bronstein_geometric_2021}) and \emph{Topological Deep Learning} (\cite{papillon_architectures_2023}).
Recurrent neural networks, in particular, have been been studied in \cite{DelayedTrace}, in the context of differential categories and the concept of \emph{delayed trace} introduced in the same paper.
Despite this, a comprehensive categorical study of architectures is still missing in the literature.
As first noticed in \cite{DeepLearningAdHoc}, many architectures such as recurrent and recursive neural network have close parallels to concepts in functional programming such as folds, unfolds and accumulating maps, for instance.
As these functional concepts have clear categorical semantics, it is natural to ask whether these categorical semantics can be used to study neural network architectures.
We believe the categorical framework presented in this paper can serve as a natural starting point for such a study.
Future work includes modelling some classical systems as well, such as the Support Vector Machines \cite{SVM}, which should be possible with the usage of loss maps such as Hinge loss.

In all our settings we have fixed an optimiser beforehand. The work of
\cite{LTL} describes a \textit{meta-learning} approach which sees the optimiser
as a neural network whose parameters and gradient update rule can be learned.
This is an exciting prospect since one can model optimisers as parametric
lenses; and our framework covers learning with parametric lenses. 

Future work also includes using the full power of CRDC axioms. In particular, axioms RD.6 or RD.7, which deal with the behaviour of
higher-order derivatives, were not exploited in our work,   but they should play a role in modelling some supervised learning algorithms using higher-order derivatives (for example, the Hessian) for additional
optimisations. Taking this idea in a different direction, one
can see that much of our work can be applied to any functor of the form $F: \Ca
\to \Lens(\Ca)$ - $F$ does not necessarily have to be of the form $f \mapsto
\diset{f}{R[f]}$ for a CRDC $R$.  Moreover, by working with more generalised forms of
the lens category (such as dependent lenses), we may be able to capture ideas
related to supervised learning on manifolds. And, of course, we can vary
the parameter space to endow it with different structure from the functions we
wish to learn. In this vein, we wish to use fibrations/dependent types to model
the use of tangent bundles: this would foster the extension of the {\em correct
by construction} paradigm to machine learning, and thereby addressing the widely
acknowledged problem of trusted machine learning. The possibilities are made
much easier by the compositional nature of our framework.
Another key topic for future work is to link gradient-based learning with game theory. At a
high level, the former takes little incremental steps to achieve an equilibrium
while the later aims to do so in one fell swoop. Formalising this intuition is
possible with our lens-based framework and the lens-based framework for game
theory~\cite{CompositionalGameTheory}.  Finally, because our framework is quite
general, in future work we plan to consider further modifications and additions
to encompass probabilistic, non-gradient based, and other forms of non-supervised learning. This
includes genetic algorithms and reinforcement learning.

\bigskip

\noindent\textbf{Acknowledgements} Fabio Zanasi acknowledges support from \textsc{epsrc} EP/V002376/1.  Geoff Cruttwell acknowledges support from NSERC.  

\bibliographystyle{plain}

\appendix

 \section{More details on Parametric Categories}\label{app:para}

As mentioned in the main text, coherence rules in combining the two operations in \eqref{eq:para-composition-doublecat} just work as expected, in the sense that these diagrams can be ultimately `compiled' down to string diagrams for monoidal categories. For example, given maps $(P,f): A \to B$, $(Q,g): B \to C$ with reparametrisations $\alpha: P' \to P$, $\beta: Q' \to Q$, one could either first reparametrise $f$ and $g$ separately and then compose the results (below left), or compose first then reparametrise jointly (below right):
\begin{equation}
   \raisebox{-11pt}{\scalebox{0.8}{\begin{tikzpicture}
	\begin{pgfonlayer}{nodelayer}
		\node [style={medium_box}] (0) at (-2, 0) {$f$};
		\node [style=none] (1) at (-2, 2.75) {};
		\node [style=none] (2) at (-4, 0) {};
		\node [style=none] (3) at (-2, 0.25) {};
		\node [style=none] (4) at (-2.25, 0) {};
		\node [style=none] (5) at (-1.75, 0) {};
		\node [style=none] (6) at (0, 0) {};
		\node [style=none] (7) at (-2.5, 1.5) {$P$};
		\node [style=none] (8) at (-5, 0) {$A$};
		\node [style={medium_box}] (10) at (-2, 3) {$\alpha$};
		\node [style=none] (11) at (-2, 3.25) {};
		\node [style=none] (12) at (-2, 4.5) {};
		\node [style=none] (13) at (-2, 5) {$P'$};
		\node [style={medium_box}] (14) at (6, 0) {$g$};
		\node [style=none] (15) at (6, 2.75) {};
		\node [style=none] (16) at (4, 0) {};
		\node [style=none] (17) at (6, 0.25) {};
		\node [style=none] (18) at (5.75, 0) {};
		\node [style=none] (19) at (6.25, 0) {};
		\node [style=none] (20) at (8, 0) {};
		\node [style=none] (21) at (5.5, 1.5) {$Q$};
		\node [style=none] (23) at (9, 0) {$C$};
		\node [style={medium_box}] (24) at (6, 3) {$\beta$};
		\node [style=none] (25) at (6, 3.25) {};
		\node [style=none] (26) at (6, 4.5) {};
		\node [style=none] (27) at (6, 5) {$Q'$};
		\node [style=none] (28) at (3, 0) {$B$};
		\node [style=none] (29) at (1, 0) {$B$};
		\node [style=none] (30) at (2, 2) {};
		\node [style=none] (31) at (2, 2) {$\cp$};
	\end{pgfonlayer}
	\begin{pgfonlayer}{edgelayer}
		\draw [in=-180, out=0, looseness=0.75] (2.center) to (4.center);
		\draw (5.center) to (6.center);
		\draw (1.center) to (3.center);
		\draw (12.center) to (11.center);
		\draw [in=-180, out=0, looseness=0.75] (16.center) to (18.center);
		\draw (19.center) to (20.center);
		\draw (15.center) to (17.center);
		\draw (26.center) to (25.center);
	\end{pgfonlayer}
\end{tikzpicture}}} \qquad \qquad \qquad   \scalebox{0.8}{ \begin{tikzpicture}
	\begin{pgfonlayer}{nodelayer}
		\node [style={medium_box}] (0) at (-2, -2) {$f$};
		\node [style=none] (1) at (-2, 2) {};
		\node [style=none] (2) at (-4, -2) {};
		\node [style=none] (3) at (-2, -1.75) {};
		\node [style=none] (4) at (-2.25, -2) {};
		\node [style=none] (5) at (-1.75, -2) {};
		\node [style=none] (6) at (0, -2) {};
		\node [style=none] (7) at (-2, 0) {$P$};
		\node [style=none] (8) at (-5, -2) {$A$};
		\node [style={medium_box}] (10) at (-2, 2.25) {$\alpha$};
		\node [style=none] (11) at (-2, 2.5) {};
		\node [style=none] (12) at (-2, 3.75) {};
		\node [style=none] (13) at (-2, 4.25) {$P'$};
		\node [style={medium_box}] (14) at (2, -2) {$g$};
		\node [style=none] (15) at (2, 2) {};
		\node [style=none] (16) at (0, -2) {};
		\node [style=none] (17) at (2, -1.75) {};
		\node [style=none] (18) at (1.75, -2) {};
		\node [style=none] (19) at (2.25, -2) {};
		\node [style=none] (20) at (4, -2) {};
		\node [style=none] (21) at (2, 0) {$Q$};
		\node [style=none] (23) at (5, -2) {$C$};
		\node [style={medium_box}] (24) at (2, 2.25) {$\beta$};
		\node [style=none] (25) at (2, 2.5) {};
		\node [style=none] (26) at (2, 3.75) {};
		\node [style=none] (27) at (2, 4.25) {$Q'$};
		\node [style=none] (32) at (-2, -0.75) {};
		\node [style=none] (33) at (2, -0.75) {};
		\node [style=none] (34) at (-2, 0.75) {};
		\node [style=none] (35) at (2, 0.75) {};
	\end{pgfonlayer}
	\begin{pgfonlayer}{edgelayer}
		\draw [in=-180, out=0, looseness=0.75] (2.center) to (4.center);
		\draw (5.center) to (6.center);
		\draw (12.center) to (11.center);
		\draw [in=-180, out=0, looseness=0.75] (16.center) to (18.center);
		\draw (19.center) to (20.center);
		\draw (26.center) to (25.center);
		\draw (3.center) to (32.center);
		\draw (17.center) to (33.center);
		\draw (15.center) to (35.center);
		\draw (1.center) to (34.center);
	\end{pgfonlayer}
\end{tikzpicture}}
\end{equation}
As expected, translating these two operations into string diagrams for monoidal categories yield equivalent representations of the same morphism. 
\begin{equation}\label{eq:para-language-rules-3}
  \scalebox{0.8}{ \begin{tikzpicture}
	\begin{pgfonlayer}{nodelayer}
		\node [style=none] (50) at (-4, -1) {};
		\node [style=none] (52) at (-5, 2) {};
		\node [style=none] (53) at (-5, 0.5) {};
		\node [style=none] (54) at (-5, -1) {};
		\node [style=none] (55) at (-4, 0.5) {};
		\node [style=none] (56) at (-4, 1) {};
		\node [style=none] (57) at (-3, 1) {};
		\node [style=none] (58) at (-3, 0) {};
		\node [style=none] (59) at (-4, 0) {};
		\node [style=none] (60) at (-3.5, 0.5) {$\alpha$};
		\node [style=none] (61) at (-3, 0.5) {};
		\node [style=none] (62) at (-2, 0.5) {};
		\node [style=none] (63) at (-1, 0) {};
		\node [style=none] (64) at (-1, 0.25) {};
		\node [style=none] (65) at (0, 0.25) {};
		\node [style=none] (66) at (0, -0.75) {};
		\node [style=none] (67) at (-1, -0.75) {};
		\node [style=none] (68) at (-0.5, -0.25) {$f$};
		\node [style=none] (69) at (0, -0.25) {};
		\node [style=none] (70) at (-1, -0.5) {};
		\node [style=none] (71) at (-2, -1) {};
		\node [style=none] (72) at (0, -0.25) {};
		\node [style=none] (73) at (0, 2) {};
		\node [style=none] (74) at (0, -0.25) {};
		\node [style=none] (75) at (0, 2) {};
		\node [style=none] (76) at (0, -0.25) {};
		\node [style=none] (77) at (0, 2) {};
		\node [style=none] (78) at (0, 2.5) {};
		\node [style=none] (79) at (1, 2.5) {};
		\node [style=none] (80) at (1, 1.5) {};
		\node [style=none] (81) at (0, 1.5) {};
		\node [style=none] (82) at (0.5, 2) {$\beta$};
		\node [style=none] (83) at (1, 2) {};
		\node [style=none] (84) at (2, 2) {};
		\node [style=none] (85) at (3, 1.25) {};
		\node [style=none] (86) at (3, 1.5) {};
		\node [style=none] (87) at (4, 1.5) {};
		\node [style=none] (88) at (4, 0.5) {};
		\node [style=none] (89) at (3, 0.5) {};
		\node [style=none] (90) at (3.5, 1) {$g$};
		\node [style=none] (91) at (4, 1) {};
		\node [style=none] (92) at (3, 0.75) {};
		\node [style=none] (93) at (2, -0.25) {};
		\node [style=none] (94) at (5, 1) {};
		\node [style=none] (140) at (-2, 1) {$P$};
		\node [style=none] (141) at (-6, 0.5) {$P'$};
		\node [style=none] (142) at (-6, 2) {$Q'$};
		\node [style=none] (143) at (2.5, 2.5) {$Q$};
		\node [style=none] (144) at (-6, -1) {};
		\node [style=none] (145) at (-6, -1) {$A$};
		\node [style=none] (146) at (1, -0.75) {};
		\node [style=none] (147) at (1, -0.75) {$B$};
		\node [style=none] (148) at (6, 1) {$C$};
	\end{pgfonlayer}
	\begin{pgfonlayer}{edgelayer}
		\draw (54.center) to (50.center);
		\draw (56.center) to (57.center);
		\draw (57.center) to (58.center);
		\draw (58.center) to (59.center);
		\draw (59.center) to (56.center);
		\draw (53.center) to (55.center);
		\draw (61.center) to (62.center);
		\draw (64.center) to (65.center);
		\draw (65.center) to (66.center);
		\draw (66.center) to (67.center);
		\draw (67.center) to (64.center);
		\draw (50.center) to (71.center);
		\draw [in=-180, out=0] (62.center) to (63.center);
		\draw [in=180, out=0] (71.center) to (70.center);
		\draw (69.center) to (72.center);
		\draw (76.center) to (74.center);
		\draw (78.center) to (79.center);
		\draw (79.center) to (80.center);
		\draw (80.center) to (81.center);
		\draw (81.center) to (78.center);
		\draw (75.center) to (77.center);
		\draw (83.center) to (84.center);
		\draw (86.center) to (87.center);
		\draw (87.center) to (88.center);
		\draw (88.center) to (89.center);
		\draw (89.center) to (86.center);
		\draw (74.center) to (93.center);
		\draw [in=-180, out=0] (84.center) to (85.center);
		\draw [in=180, out=0] (93.center) to (92.center);
		\draw (91.center) to (94.center);
		\draw (52.center) to (77.center);
	\end{pgfonlayer}
\end{tikzpicture}} \qquad = \qquad  \scalebox{0.8}{\begin{tikzpicture}
	\begin{pgfonlayer}{nodelayer}
		\node [style=none] (50) at (-4, -1) {};
		\node [style=none] (52) at (-5, 2) {};
		\node [style=none] (53) at (-5, 0.5) {};
		\node [style=none] (54) at (-5, -1) {};
		\node [style=none] (55) at (-4, 0.5) {};
		\node [style=none] (56) at (-4, 1) {};
		\node [style=none] (57) at (-3, 1) {};
		\node [style=none] (58) at (-3, 0) {};
		\node [style=none] (59) at (-4, 0) {};
		\node [style=none] (60) at (-3.5, 0.5) {$\alpha$};
		\node [style=none] (61) at (-3, 0.5) {};
		\node [style=none] (62) at (-2, 0.5) {};
		\node [style=none] (63) at (-1, 0) {};
		\node [style=none] (64) at (-1, 0.25) {};
		\node [style=none] (65) at (0, 0.25) {};
		\node [style=none] (66) at (0, -0.75) {};
		\node [style=none] (67) at (-1, -0.75) {};
		\node [style=none] (68) at (-0.5, -0.25) {$f$};
		\node [style=none] (69) at (0, -0.25) {};
		\node [style=none] (70) at (-1, -0.5) {};
		\node [style=none] (71) at (-2, -1) {};
		\node [style=none] (72) at (0, -0.25) {};
		\node [style=none] (73) at (-4, 2) {};
		\node [style=none] (74) at (0, -0.25) {};
		\node [style=none] (75) at (-4, 2) {};
		\node [style=none] (76) at (0, -0.25) {};
		\node [style=none] (77) at (-4, 2) {};
		\node [style=none] (78) at (-4, 2.5) {};
		\node [style=none] (79) at (-3, 2.5) {};
		\node [style=none] (80) at (-3, 1.5) {};
		\node [style=none] (81) at (-4, 1.5) {};
		\node [style=none] (82) at (-3.5, 2) {$\beta$};
		\node [style=none] (83) at (-3, 2) {};
		\node [style=none] (84) at (1, 2) {};
		\node [style=none] (85) at (2, 1.25) {};
		\node [style=none] (86) at (2, 1.5) {};
		\node [style=none] (87) at (3, 1.5) {};
		\node [style=none] (88) at (3, 0.5) {};
		\node [style=none] (89) at (2, 0.5) {};
		\node [style=none] (90) at (2.5, 1) {$g$};
		\node [style=none] (91) at (3, 1) {};
		\node [style=none] (92) at (2, 0.75) {};
		\node [style=none] (93) at (1, -0.25) {};
		\node [style=none] (94) at (4, 1) {};
		\node [style=none] (140) at (-2, 1) {$P$};
		\node [style=none] (141) at (-6, 0.5) {$P'$};
		\node [style=none] (142) at (-6, 2) {$Q'$};
		\node [style=none] (143) at (-2, 2.5) {$Q$};
		\node [style=none] (144) at (-6, -1) {};
		\node [style=none] (145) at (-6, -1) {$A$};
		\node [style=none] (146) at (1, -0.75) {};
		\node [style=none] (147) at (1, -0.75) {$B$};
		\node [style=none] (148) at (5, 1) {$C$};
	\end{pgfonlayer}
	\begin{pgfonlayer}{edgelayer}
		\draw (54.center) to (50.center);
		\draw (56.center) to (57.center);
		\draw (57.center) to (58.center);
		\draw (58.center) to (59.center);
		\draw (59.center) to (56.center);
		\draw (53.center) to (55.center);
		\draw (61.center) to (62.center);
		\draw (64.center) to (65.center);
		\draw (65.center) to (66.center);
		\draw (66.center) to (67.center);
		\draw (67.center) to (64.center);
		\draw (50.center) to (71.center);
		\draw [in=-180, out=0] (62.center) to (63.center);
		\draw [in=180, out=0] (71.center) to (70.center);
		\draw (69.center) to (72.center);
		\draw (76.center) to (74.center);
		\draw (78.center) to (79.center);
		\draw (79.center) to (80.center);
		\draw (80.center) to (81.center);
		\draw (81.center) to (78.center);
		\draw (75.center) to (77.center);
		\draw (83.center) to (84.center);
		\draw (86.center) to (87.center);
		\draw (87.center) to (88.center);
		\draw (88.center) to (89.center);
		\draw (89.center) to (86.center);
		\draw (74.center) to (93.center);
		\draw [in=-180, out=0] (84.center) to (85.center);
		\draw [in=180, out=0] (93.center) to (92.center);
		\draw (91.center) to (94.center);
		\draw (52.center) to (77.center);
	\end{pgfonlayer}
\end{tikzpicture}}
\end{equation}

\begin{rem} There is a 2-categorical perspective on $\Para(\Ca)$, which we glossed over in this paper for the sake of simplicity. In particular, the reparametrisations described above can also be seen as equipping
$\Para(\Ca)$ with 2-cells, giving a 2-categorical structure on $\Para(\Ca)$.
This is also coherent with respect to base change: if $\Ca$ and $\Da$ are strict symmetric monoidal categories, and $F \colon \Ca \to \Da$ a lax symmetric monoidal functor, then there is an induced 2-functor $\Para(F) \colon  \Para(\Ca) \to \Para(\Da)$ which agrees with $F$ on objects. This 2-functor is straightforward: for a 1-cell $(P,f): A \to B$, it applies $F$ to $P$ and $f$ and uses the (lax) comparison to get a map of the correct type.
We will see how this base change becomes important when performing
backpropagation on parametric maps (Eq. \ref{eq:para_rdc})

Lastly, we mention that $\Para(\Ca)$ inherits the symmetric monoidal structure
from $\Ca$ and that the induced 2-functor $\Para(F)$ respects that structure.
This will allow us to compose neural networks not only in series, but also in
parallel.  For more detail on alternative viewpoints on the $\Para$ construction, including how it can be viewed as the Grothendieck construction of a certain indexed category, see \cite{TowardsCatCyber}.  
\end{rem}

 \section{Background on Cartesian Reverse Differential Categories} \label{app:CRDC}

Here we briefly review the definitions of Cartesian left additive category (CLAC), Cartesian reverse differential category (CRDC) and additive and linear maps in these categories.  Note that in this appendix we follow the convention of \cite{CRDC} and write composition in diagrammatic order by juxtaposition of terms (rather than a semicolon) to shorten the form of many of the expressions.  

\begin{defi} A category $\Ca$ is said to be {\bf Cartesian} when there are chosen binary products $\times$, with projection maps $\pi_i$ and pairing operation $\langle - , - \rangle$, and a chosen terminal object $T$, with unique maps $!$ to the terminal object. 
\end{defi}

\begin{defi} A {\bf left additive category} \cite[Definition  1.1.1]{journal:BCS:CDC} (CLAC) is a category $\Ca$ such that each hom-set has  commutative monoid structure, with addition operation $+$ and zero maps 0, such that composition on the left preserves the additive structure: for any appropriate $f, g, h$, $f(g+h) = fg + fh$ and $f0 = 0$.
\end{defi}  

\begin{defi}
A map $h: X \to Y$ in a CLAC is \textbf{additive} if it has the property that it preserves additive structure by composition on the right: for any maps $x,y: Z \to X$, $(x+y);h = x;h + y;h$, and $0;h = 0$. 
\end{defi}

\begin{defi}[{Additive in second component, (compare \cite[Lemma 1.2.3]{journal:BCS:CDC})}]
  \label{def:additive_second_variable}
  A morphism $f : X \times A \to B$  is additive in the variable $A$ if it is an additive morphism of type $A \to B$ in the cartesian left-additive category $\CoKl(X \times -)$, where $\CoKl(X \times -)$ is the coKleisli category of the coreader comonad\footnote{There are a few other terms for this. One of them is ``the writer comonad'', though this is often confused with the writer \emph{monad} which additionally necessitates a monoid structure on $X$. It's also called reader comonad, because of duality to reader monad, and also ``product comonad'' or ``environment comonad''.}.
\end{defi}

\begin{defi}
A {\bf Cartesian left additive category} \cite[Definition 1.2.1]{journal:BCS:CDC} is a left additive category $\Ca$ which is Cartesian and such that all projection maps $\pi_i$ are additive. 
\end{defi}

\begin{defi}
  \label{def:clacat}
  We call $\CLACat$ the category whose objects are cartesian left-additive categories and whose morphisms are cartesian left-additive functors (functors which preserve products and commutative monoid structure on objects (\cite[Definition 1.3.1]{journal:BCS:CDC})).
\end{defi}

\begin{lem}
  \label{lemma:cla_preserves_additive}
  Let $\Ca$ and $\Da$ be cartesian left-additive categories, and $F : \Ca \to \Da$ a a cartesian left-additive functor.
  Let $f : A \to B$ be an additive morphism in $\Ca$. Then $F(f) : F(A) \to F(B)$ is also additive.
\end{lem}

The central definition of \cite{CRDC} is the following:

\begin{defi}\label{def:crdc}
A \textbf{Cartesian reverse differential category} (CRDC) is a Cartesian left additive category $\Ca$ which has, for each map $f: A \to B$ in $\Ca$, a map
	\[ R[f]: A \times B \to A \]
satisfying seven axioms: \\
\noindent 
 {\bf [RD.1]}   $R[f+g] = R[f] + R[g]$ and $R[0]=0$. \\
\noindent  
{\bf [RD.2]} $\<a,b+c\>R[f] = \<a,b\>R[f] + \<a,c\>R[f]$ and  $\<a,0\>R[f] = 0$. \\
{\bf [RD.3]} 
    $R[1] = \pi_1$, $R[\pi_0] = \pi_1 \iota_0$, and $R[\pi_1] = \pi_1 \iota_1$.  
\noindent {\bf [RD.4]} 
   For a tupling of maps $f$ and $g$, the following equality holds:      
$$R[\<f,g\>] = (1 \times \pi_0);R[f]+ (1 \times \pi_1);R[g]$$
And if $!_A: A \to T$ is the unique map to the terminal object, $R[!_A] = 0$.
        
\noindent {\bf [RD.5]} 
For composable maps $f$ and $g$,
	\[ R[fg] = \<\pi_0,\pi_0f,\pi_1\>\>(1 \times R[g])R[f] \]
\noindent {\bf [RD.6]} 
\[ \<1\times \pi_0,0\times \pi_1\>(\iota_0 \times 1)R[R[R[f]]]\pi_1 = (1\times \pi_1)R[f] .  \] \\ 
\noindent {\bf [RD.7]} 
\[ (\iota_0 \times 1);R[R[(\iota_0 \times 1)R[R[f]]\pi_1]];\pi_1  = \mathsf{ex};(\iota_0 \times 1)R[R[(\iota_0 \times 1)R[R[f]]\pi_1]]\pi_1 \] (where $\mathsf{ex}$ is the map that exchanges the middle two variables).  
\end{defi}

As discussed in \cite{CRDC}, these axioms correspond to familiar properties of the reverse derivative:
\begin{itemize}
	\item  {\bf [RD.1]} says that differentiation preserves addition of maps, while  {\bf [RD.2]} says that differentiation is additive in its vector variable.
	\item {\bf [RD.3]} and {\bf [RD.4]} handle the derivatives of identities, projections, and tuples.  
	\item {\bf [RD.5]} is the (reverse) chain rule.  
	\item {\bf [RD.6]} says that the reverse derivative is linear in its vector variable.
	\item {\bf [RD.7]} expresses the independence of order of mixed partial derivatives.  
\end{itemize}

We proceed to prove the following theorem in three steps.
\LensFunctorCLAC*

The first step is formally defining the category $\Lens_A(\Ca)$.
\begin{defi}
  \label{def:lensa}
  Let $\Ca$  be a cartesian left-additive category.
  Then $\Lens_A(\Ca)$ is a wide subcategory of $\Lens(\Ca)$ where
  \[
    \Lens_A(\Ca)\LensHom{A}{A'}{B}{B'} \coloneqq \Ca(A, B) \times \CoKl(A \times -)_A(B', A')
  \]
  Compare this with the defintion of $\Lens(\Ca)$ via the Grothendieck construction (\cite[Prop. 3.10]{spivak_generalized_2022}) where
  \[
    \Lens(\Ca)\LensHom{A}{A'}{B}{B'} \coloneqq \Ca(A, B) \times \CoKl(A \times -)(B', A')
  \]
\end{defi}

The second step is showing this category is cartesian left-additive.
\begin{prop}
  \label{prop:lens_cart_left_additive}
  The category $\Lens_A(\Ca)$ is cartesian left-additive.
\end{prop}

\begin{proof}
  We need to equip $\Lens_A(\Ca)$ with a commutative monoid on every object in a way that's compatible with the cartesian structure.\footnote{We don't need to show that this monoid is unique, just that it exists and can be canonically defined.}
  That is, for every object $\diset{A}{A'}$ we need to provide two morphisms:
\begin{itemize}
\item \textbf{Unit } $0_{\diset{A}{A'}} : \diset{1}{1} \to \diset{A}{A'}$. This is a lens whose forward map we set as the zero map $0_A$ and the backward map as the delete $!_{1 \times A'}$.
\item \textbf{Multiplication} $+_{\diset{A}{A'}} : \diset{A \times A}{A' \times A'} \to \diset{A}{A'}$. This is a lens whose forward map we set as sum $+_A$ and the backward map as copy, i.e. $(A \times A) \times A' \xrightarrow{\pi_2} A' \xrightarrow{\Delta_{A'}} A' \times A'$.
\end{itemize}
Additionally, these morphisms need to obey the monoid laws.
This can be verified by routine.
\end{proof}

This defines the action of $\Lens_A$ on objects of $\CLACat$.
Action on morphisms is defined below.

\begin{prop}
  \label{prop:lens_functor_cart_left_add}
  Let $F : \Ca \to \Da$ be a cartesian left-additive functor.
  This induces a cartesian left-additive functor $\Lens_A(F)$ between the corresponding categories of lenses: 
  
  \begin{equation}
    \label{eq:lensa_action_on_morphisms}
    \begin{tikzcd}
        {\Lens_A(\Ca)} &&& {\Lens_A(\Da)} \\
        {\diset{A}{A'}} &&& {\diset{F(A)}{F(A')}} \\
        \\
        {\diset{B}{B'}} &&& {\diset{F(B)}{F(B')}}
        \arrow["{ \Lens_A(F)}", from=1-1, to=1-4]
        \arrow["{\diset{f}{f^*}}"', from=2-1, to=4-1]
        \arrow["{\diset{F(f)}{\overline{f^*}}}", from=2-4, to=4-4]
        \arrow[maps to, from=2-1, to=2-4]
        \arrow[maps to, from=4-1, to=4-4]
      \end{tikzcd}
  \end{equation}
  where $\overline{f^*} \coloneqq F(A') \times F(B') \cong F(A' \times B') \xrightarrow{F(f')} F(A')$. 
\end{prop}

\begin{proof}
  We need to prove that $\Lens_A(F)$ is a cartesian left-additive functor.
  To prove it is a functor, we need to:
  \begin{itemize}
  \item Define its action on objects and morphisms. We have done this in Prop.\ \ref{prop:lens_functor_cart_left_add} itself;
  \item Prove additivity of $\overline{f^*}$. This follows from Lemma.\ \ref{lemma:cla_preserves_additive};
  \item Prove identities are preserved.
    The identity $\diset{\id_A}{\pi_2} : \diset{A}{A'} \to \diset{A}{A'}$ in the domain gets mapped to $\diset{F(\id_A)}{F(\pi_2)}$.
    By preservation of identities and products of $F$ this is equal to the identity map on $\diset{F(A)}{F(A')}$.
  \item Prove composition is preserved. This can be be by routine, albeit tedious calculation.
  \end{itemize}
  To prove that it is additionally cartesian, we need to show that the image of every comonoid $(\diset{A}{A'}, !_{\diset{A}{A'}}, \Delta_{\diset{A}{A'}})$ is also a comonoid, and that all maps preserve comonoids.
     We can understand the first part in terms of actions on the counit and comultiplication of the comonoid.
    \begin{itemize}
    \item \textbf{Counit.} The action on the counit unpacks to the pair $\diset{F(!_{\diset{A}{A'}})}{F(!_{A \times 1} ; 0_A)}$.
      By preservation of terminal and additive maps of $F$ this morphism is equal to the counit of $\diset{F(A)}{F(A')}$.
    \item \textbf{Comultiplication.} The action on the comultiplication unpacks to $\diset{(F(\Delta_A)}{F(\pi_{2, 3} ; +_A)}$.
      By $F$'s preservation of products and additive morphisms this morphism is equal to the comultiplication of $\diset{F(A)}{F(A')}$.
    \end{itemize}
    It is routine to show it obey the corresponding laws and form a comonoid.
    
    For the second part we need to show that the image of every lens $\diset{f}{f^\sharp} : \diset{A}{A'} \to \diset{B}{B'}$ preserves these comonoids.
    For the forward part this is true because $F$ preserves products.
    For the backwards part this is true because $F$ is left-additive.
    
    Lastly, we need to prove that this functor is additionally left-additive.
    This means that it preserves the monoid $(\diset{A}{A'}, 0_{\diset{A}{A'}}, +_{\diset{A}{A'}})$ of every object.
    We unpack the action of $\Lens_A(F)$ on the unit $0_{\diset{A}{A'}}$ and multiplication $+_{\diset{A}{A'}}$ below.
    \begin{itemize}
    \item \textbf{Unit.} The action on the unit unpacks to the pair $\diset{F(0_A)}{F(!_{1 \times A})}$.
      By preservation of additive and terminal maps of $F$ this morphism is equal to the unit of $\diset{F(A)}{F(A)}$;
    \item \textbf{Multiplication.} The action on the multiplication unpacks to the pair $\diset{F(+_A)}{F(\pi_3 ; \Delta_A)}$.
      By preservation of coadditive maps and products of $F$ this morphism is equal to  the multiplication of $\diset{F(A)}{F(A')}$.
    \end{itemize}
    Seeing as these monoids in the codomain are of the same form as those in the domain, it is routine to show that they obey the monoid laws.
  This concludes the proof that $\Lens_A(F)$ is a cartesian left-additive functor.
\end{proof}

What remains to show is that $\Lens_A$ preserves identities and composition, which follows routinely, concluding the proof of (Thm.\ \ref{thm:lens_functor_clac}).

This functor has additional structure --- it is copointed.\footnote{Despite this the functor $\Lens_A$ does not have the comonad structure, for similar reasons that tangent categories do not.}

\begin{prop}[{Copointed structure of $\Lens_A$}]
  \label{prop:lensa_copointed}
  There is a natural transformation $\epsilon : \Lens_A \Rightarrow \id_{\CLACat}$ which on components assigns to cartesian left-additive category $\Ca$ a cartesian left-additive functor $\epsilon_{\Ca} : \Lens_A(\Ca) \to \Ca$ which forgets the backward part of the lens.
\end{prop}

\FiveOutOfSeven*

\begin{proof}
  We have shown how a putative reverse derivative combinator arises out of the functor $\RC : \Ca \to \Lens_A(\Ca)$.
  What remains to prove is that this combinator satisfies the first five axioms of a CRDC.
\begin{enumerate}
\item \textbf{Additivity of reverse differentiation.} This is recovered by $\RC$ preserving left-additive structure.
\item \textbf{Additivity of reverse derivative in the second variable.} This is recovered by definition of $\Lens_A$ --- the backward maps are additive in the 2nd component.
\item \textbf{Coherence with identities and projections.} Coherence with
  identities is recovered by preservation of identities of the functor $\RC$, where for every $X : \Ca$, $\RC(\id_X) = \id_{\RC(X)} = (\id_X, \pi_2 : X \times X \to X)$. Coherence with projections is recovered by $\RC$ preserving cartesian structure.
\item \textbf{Coherence with pairings.} Recovered by $\RC$ preserving cartesian structure.
\item \textbf{Reverse chain rule.} Recovered by functoriality of $\RC$.
\end{enumerate}
\end{proof}




\end{document}